\documentclass[twoside,11pt]{article}

\usepackage{jmlr2e}
\usepackage{amsmath}
\usepackage{lastpage} 
\usepackage{algorithm}
\usepackage{algpseudocode}
\usepackage{color}

\newtheorem{assumption}{Assumption}
\def \Paren#1{{\left({#1}\right)}}
\newcommand{\probof}[1]{\Pr\Paren{#1}}
\usepackage{subfigure}

\jmlrheading{21}{2022}{1-\pageref{LastPage}}{05/21; Revised 
07/22}{9/20}{19-253}{Di Wang, Lijie Hu, Huanyu Zhang, Marco Gaboardi, and Jinhui Xu}

\ShortHeadings{LDP GLM Estimation}{Wang, Hu, Zhang, Gaboardi and Xu}
\firstpageno{1}

\begin{document}

\title{Generalized Linear Models in Non-interactive Local Differential Privacy with Public Data}

\author{\name Di Wang  \email di.wang@kaust.edu.sa \\
       \addr CEMSE\\
       King Abdullah University of Science and Technology\\
       Thuwal, Saudi Arabia
       \AND
       \name Lijie Hu \email lijie.hu@kaust.edu.sa\\
         \addr CEMSE\\
         King Abdullah University of Science and Technology\\
       Thuwal, Saudi Arabia 
       \AND
       Huanyu Zhang \email hz388@cornell.edu\\
         \addr Meta \\
         New York, NY, USA 
       \AND
       \name Marco Gaboardi
       \email gaboardi@bu.edu \\
      \addr Department of Computer Science \\
       Boston University\\
       Boston, MA 02215, USA
       \AND
              \name Jinhui Xu \email jinhui@buffalo.edu \\
       \addr Department of Computer Science and Engineering\\
       University at Buffalo, SUNY\\
       Buffalo, NY 14260, USA}

\editor{} 
\maketitle
\begin{abstract}%
In this paper, we study the problem of estimating smooth Generalized Linear Models (GLMs) in the Non-interactive Local Differential Privacy (NLDP) model. 
Different from its classical setting, our model allows the server to access 
some additional public but unlabeled data. In the first part of the paper we focus on GLMs. Specifically,
we first consider the case where each data record is i.i.d. sampled from a zero-mean multivariate Gaussian distribution. Motivated by the Stein's lemma, we present an $(\epsilon, \delta)$-NLDP algorithm for 
 GLMs. Moreover, the sample complexity of 
public and private data for the algorithm to achieve an $\ell_2$-norm estimation error of  $\alpha$ (with high probability) is ${O}(p \alpha^{-2})$ and $\tilde{O}(p^3\alpha^{-2}\epsilon^{-2})$ respectively, where $p$ is the dimension of the feature vector. This is a significant improvement over the previously known exponential or quasi-polynomial in $\alpha^{-1}$, or exponential in $p$ sample complexities  of  GLMs with no public data. Then we consider a more general setting where each data record is i.i.d. sampled from some sub-Gaussian distribution with bounded $\ell_1$-norm. Based on a variant of Stein's lemma, we propose an $(\epsilon, \delta)$-NLDP algorithm for 
 GLMs whose sample complexity of 
public and private data to achieve an $\ell_\infty$-norm estimation error of $\alpha$ is ${O}(p^2\alpha^{-2})$ and $\tilde{O}(p^2\alpha^{-2}\epsilon^{-2})$  respectively, under some mild assumptions and if  $\alpha$ is not too small 
({\em i.e.,} $\alpha\geq \Omega(\frac{1}{\sqrt{p}})$). In the second part of the paper, we extend our idea to the problem of estimating non-linear regressions and show 
similar results as in GLMs for both multivariate Gaussian and sub-Gaussian cases.
Finally, we demonstrate the effectiveness of our algorithms through experiments on both synthetic and real-world datasets. 
To our best knowledge, this is the first paper showing the existence of efficient and effective 
algorithms for GLMs and non-linear regressions 
in the NLDP model with public unlabeled data. \footnote{The first three authors contributed equally to this paper.}\footnote{An abstract version of this paper was presented at The 32nd International Conference on Algorithmic Learning Theory (ALT 2021) \citep{wang2021alt}.}

\end{abstract}


\section{Introduction}
Generalized Linear Model (GLM) is one of the most fundamental models in statistics and machine learning. It generalizes the ordinary linear regression by allowing the linear model to be related to the response variable via a link function and by allowing the magnitude of the variance of each measurement to be a function of its predicted value.
GLM was introduced as a way of unifying various statistical models, including linear, logistic and Poisson regressions and 
%
it has a wide range of applications in 
various domains, such as social sciences \citep{warne_2017}, genomics research \citep{takada2017generalized}, finance \citep{mcneil2007bayesian} and medical research \citep{lindsey1998choosing}. The model can be formulated as follows. 
\paragraph{GLM:} Let 
$y\in [0, 1]$ be the response variable that belongs to an exponential family with natural parameter $\psi$. \footnote{For simplicity in this paper we assume $y$ is in $[0 , 1]$. We will leave the case where $y$ could be unbounded as future research.}
 That is, its probability density function can be written as $p(y|\psi)= \exp(\psi y-\Phi(\psi))h(y)$, where $\Phi$ is the \textit{cumulative generating function}. Given observations $y_1, \cdots, y_n$ such that $y_i\sim p(y_i|\psi_i)$ for $\psi=(\psi_1, \cdots, \psi_n)$, the maximum likelihood  function can be written as  $p(y_1,y_2,\cdots|\psi)=\exp(\sum_{i=1}^n y_i\psi_i-\Phi(\psi_i))\Pi_{i=1}^nh(y_i)$. In GLM, we assume that $\psi$ is modeled by linear relations, {\em i.e.,}
$\psi_i=\langle x_i,  w^*\rangle$ for some $w^*\in \mathbb{R}^p$ and feature vector $x_i$. Thus, finding the maximum likelihood estimator (MLE) is equivalent to minimizing $\frac{1}{n}\sum_{i=1}^n[\Phi(\langle x_i, w\rangle)-y_i\langle x_i, w\rangle]$. The goal is to find $w^*$, which is equivalent to minimizing its population version  
\begin{equation}\label{eq:1}
    w^*=\arg\min_{w\in \mathbb{R}^p}\mathbb{E}_{(x,y)}[\Phi(\langle x, w\rangle )- y\langle x, w \rangle ].
\end{equation}

One often encountered challenge for using GLMs in real world applications is how to handle sensitive data, such as those  in social science and  medical research.
As a commonly-accepted technique for preserving privacy, Differential Privacy (DP) \citep{dwork2006calibrating}
provides provable protection against re-identification attacks 
and is resilient to arbitrary auxiliary information that might be available to attackers.  It allows for rich statistical and machine learning analysis, and is becoming a {\em de facto} notion for private data analysis. 

As a popular way of achieving 
DP, 
Local Differential Privacy (LDP) has received considerable attention in recent years and has been adopted in  industry~\citep{ding2017collecting,erlingsson2014rappor,DBLP:journals/corr/abs-1709-02753}. In LDP, each individual manages his/her proper data and discloses them to a server through some DP mechanisms. The server collects the (now private) data of each individual and combines them into a resulting data analysis. Information exchange between the server and individuals could be either only once or multiple times. Correspondingly,  protocols for LDP are called non-interactive LDP (NLDP) or interactive LDP. Due to its ease of    implementation ({\em e.g.} no need to deal with the network latency issue), NLDP is often preferred in practice.   

While there are many results on estimating GLMs in the DP and interactive LDP models such as  \citep{chaudhuri2011differentially,bassily2014private,jain2014near,kasiviswanathan2016efficient}, 
estimating GLMs in NLDP is  still not well-understood due to the  limitation  of  number of interaction round  in  the privacy model. 
Recently \citep{smith2017interaction,wang2018empirical,zheng2017collect} and \citep{wang2019noninteractive} provided comprehensively studies on this problem. However, all of these results 
are on the negative side. More specifically, 
they showed that to achieve an error of $\alpha$, the sample complexity needs to be  quasi-polynomial or exponential in $\alpha^{-1}$ (based on different assumptions) \citep{wang2019noninteractive,zheng2017collect}, or exponential in the dimension $p$ \citep{smith2017interaction,wang2018empirical} (see Related Work section for more details). Recently, \citep{dagan2020interaction} showed that an exponential lower bound (either in $p$ or $\alpha^{-1}$) on the number of samples for solving the standard task of learning a large-margin linear separator in the NLDP model. Due to these negative results, there is no study on the practical performance of these algorithms.

To address this high sample complexity issue of estimating GLMs in NLDP, a possible way is to make use of some recent developments in the central DP model. Quite a few results \citep{bassily2019privately,hamm2016learning,papernot2016semi,papernot2018scalable,bassily2018model,liu2021revisiting} have suggested that by allowing the server to access some public but unlabeled data in addition to the private data, it is possible to further reduce the sample complexity in the central DP model,  under the assumption that these public data samples have the same marginal distribution as the private ones. It has been also shown that such a relaxed setting is likely to enable better practical performance  for various problems such as Empirical Risk Minimization (ERM) and Deep Neural Networks \citep{hamm2016learning,papernot2016semi}. Thus, it would be interesting to know whether the relaxed setting on public unlabeled data can also help to reduce sample complexity of GLMs in the NLDP model.


With this thinking, our main questions now become the following. 
{\bf 
Can we further reduce the sample complexity of GLMs in the NLDP model if the server has additional public but unlabeled data? Moreover, is there any efficient 
algorithm for this problem? } 
\begin{table*}[t]
\begin{center}
\resizebox{\textwidth}{!}{%
\begin{tabular}{|l|l|l|l|l|l|}
\hline
Methods                 & Sample Complexity                                                                                                                                                          & Measure           & Loss Function                       & With public data?                & Data                  \\ [3ex] \hline
\citep{smith2017interaction} &  $O(p\epsilon^{-2}\alpha^{-2})$  & Excess Risk & Linear Regression & No & $\ell_2$-norm Bounded  \\[3ex]
					\hline
					
\citep{smith2017interaction}   & $\tilde{O}(4^p\alpha^{-(p+2)}\epsilon^{-2})$                                                                                                                               & Excess Risk       & Lipschitz                           & No                  & $\ell_2$-norm Bounded \\ [3ex]\hline
\citep{smith2017interaction}    & $\tilde{O}(2^p\alpha^{-(p+1)}\epsilon^{-2})$                                                                                                                               & Excess Risk       & Lipschitz and Convex                & No                  & $\ell_2$-norm Bounded \\ [3ex]\hline
				\citep{wang2018empirical}  &  $\tilde{O}\big( (c_0 p^{\frac{1}{4}})^p\alpha^{-(2+\frac{p}{2})}\epsilon^{-2}\big)$  & Excess Risk& $(8, T)$-smooth & No & $\ell_2$-norm Bounded\\[3ex]
				\hline
                   \citep{wang2018empirical}     & $\tilde{O}(4^{p(p+1)}D^2_p\epsilon^{-2}\alpha^{-4})$                                                                                                                       & Excess Risk       & $(\infty, T)$-smooth                & No                  & $\ell_2$-norm Bounded \\ [3ex]\hline
        \citep{wang2019noninteractive,wangjmlr2020}               & $p\cdot \left(\frac{C}{\alpha^3}\right)^{O( 1/\alpha^3)}/\epsilon^{O(\frac{1}{\alpha^3})}$                                                                                 & Excess Risk       & Lipschitz Convex GLM & No                  & $\ell_2$-norm Bounded \\[3ex] \hline
\citep{zheng2017collect} & \begin{tabular}[c]{@{}l@{}}$p(	\frac{8}{\alpha})^{O(\log\log(\frac{1}{\alpha}))}(\frac{4}{\epsilon})^{O(\log(\frac{1}{\alpha}))}$\end{tabular} & Excess Risk       & Convex $\infty$-Smooth GLM           & No                & $\ell_2$-norm Bounded \\[3ex] \hline
\textbf{This paper}            & $O(p^3\alpha^{-2}\epsilon^{-2})$                                                                                                                                                        & $\ell_2$-norm Error     & \begin{tabular}[c]{@{}l@{}}Smooth GLM\\ (with additional assumptions) \end{tabular}         & Yes & Gaussian              \\[3ex] \hline
\textbf{This paper}             & \begin{tabular}[c]{@{}l@{}} $O(p^2\alpha^{-2}\epsilon^{-2})$\\ for $\alpha\geq \Omega(\frac{1}{\sqrt{p}})$ \end{tabular}                                                                                                                                                        & $\ell_\infty$-norm  Error & \begin{tabular}[c]{@{}l@{}}Smooth GLM\\ (with additional assumptions)\end{tabular}          & Yes &  \begin{tabular}[c]{@{}l@{}}$\ell_1$-norm Bounded \\ and Sub-Gaussian\end{tabular} \\ [3ex]\hline
\end{tabular}}
\caption{Comparisons on the sample complexities (of private data) for achieving error $\alpha$ under different measurements for GLMs in the non-interactive LDP model,   where $c_0, C$ are constants and $D_p$ is a function of dimension $p$. For bounded norm case we assume that $\|x_i\| \leq 1$  for every $i\in [n]$. For multivariate Gaussian case we assume $x_i\sim \mathcal{N}(0, \Sigma)$ with some unknown $\Sigma$.}\label{Table:1}
\end{center}
\end{table*}

In this paper, we provide positive answers to the above two questions, see Table \ref{Table:1} for our results. Specifically, our contributions can be summarized as  follows:
\begin{enumerate}
    \item {Firstly, motivated by the Stein's lemma (Lemma \ref{lemma:1}), we show that when the covariate (feature vector) $x$ follows an (unknown) zero-mean multivariate Gaussian distribution, {\em i.e.,} $x\sim \mathcal{N}(0, \Sigma)$ with some $\Sigma \in \mathbb{R}^{p\times p}$, 
 there exists an $(\epsilon, \delta)$-NLDP algorithm for 
 GLMs. Moreover,  the sample complexity of
public and private data for the algorithm to achieve an  $\ell_2$-norm estimation error of $\alpha$ (with high probability), is ${O}(p\alpha^{-2})$ and $\tilde{O}(p^3\alpha^{-2}\epsilon^{-2})$ (with other terms omitted) respectively. 
    We note that this is the first result that achieves a {\bf fully polynomial} sample complexity for a general class of loss functions in the NLDP model with public unlabeled data.  } 
    \item Then we consider a more general case where the covariate $x$ in GLMs is sub-Gaussian with bounded $\ell_1$-norm. Based on a variant of Stein's lemma we propose an 
   $(\epsilon, \delta)$-NLDP algorithm for GLMs. Moreover, under some mild assumptions, the sample complexity of private and public data to achieve  an $\ell_\infty$-norm error of $\alpha$ is $\tilde{O}(p^2\epsilon^{-2}\alpha^{-2})$ and $\tilde{O}(p^2\alpha^{-2})$ (with other terms omitted)
    respectively, if $\alpha$ is not too small ({\em i.e.,} $\alpha\geq \Omega(\frac{1}{\sqrt{p}})$).
    \item We then extend our idea to the problem of estimating non-linear regressions. 
    By using the Stein's lemma and the zero-bias transformation \citep{goldstein1997stein}, we propose $(\epsilon, \delta)$-NLDP algorithms both cases where $x$ is  multivariate Gaussian and sub-Gaussian with bounded $\ell_1$-norm. Moreover, we show similar estimation errors as in the GLMs case. 
    \item Finally, we provide extensive experimental study of our algorithms on both synthetic and real-world datasets. The experimental results suggest that our methods are efficient and effective, and they are consistent with our theoretical analysis.  Moreover, based on these results we also find some aspects that need further theoretical investigation. 
\end{enumerate}
\section{Related Work}\label{sec:related work}
 Private learning with public unlabeled data has been studied previously in \citep{hamm2016learning,papernot2016semi,papernot2018scalable,bassily2018model,liu2021revisiting}. These results differ from ours in quite a few ways. 
 Firstly, all of them 
 consider either the multiparty setting or the central DP model and cannot be extended to the NLDP model. 
 Consequently, none of them can be used to solve our problems. 
 Specifically, \citet{hamm2016learning} considered the multiparty setting where each party possesses several data records. Their method needs each party to use its data to get a classifier. However, this approach could not be extended to local DP model since in our case each party only has one data sample and it is impossible to get any useful classifier based on one data sample only. \citet{papernot2016semi,papernot2018scalable} considered  training some Deep Neural Networks in the DP model by using the subsample-and-aggregate framework in \citep{nissim2007smooth}. However, there is no provable sample complexity for their methods. \citet{bassily2018model} studied DP-ERM in the central model, which is later extended by \citep{liu2021revisiting}. Their method is based on combining the function of distance to instability and the sparse vector technique. However, both the subsample-and-aggregate framework and the sparse vector technique cannot be used in the local DP model. Secondly, public data samples in those methods are also used quite differently from ours. Specifically, all of the above approaches use private data to get private classifiers. Based on these classifiers, they label the public data and conduct the learning process on the public data (now with pseudo labels), while in this paper
 we use the public data to approximate some crucial constants. Finally, all of the previous methods rely on the known model or the explicit form of loss function, while in our algorithms the loss functions could be unknown to users;  also the server could estimate multiple different GLMs with the same sample complexity. 

The problems considered in this paper can be viewed as restricted cases of the ERM  problem in NLDP model. 
Due to its challenging nature, ERM in NLDP has only been considered in a few papers, such as \citep{smith2017interaction,wang2018empirical,wang2019noninteractive,zheng2017collect,daniely2018learning,wangicml19linear}, see Table \ref{Table:1} for a summary. \citet{smith2017interaction} gave the first result on convex ERM in NLDP and provided an algorithm with a sample complexity of $O(2^p \alpha^{-(p+1)} \epsilon^{-2})$. 
They showed that the exponential dependency on 
the dimension $p$ is unavoidable in the worst case. 
Later, \citet{wang2018empirical} showed that when the loss function is smooth enough, the exponential term of  $\alpha^{-\Omega(p)}$ can be reduced to polynomial. However, there is still another exponential term in their sample complexity. Recently,  \citet{wang2019noninteractive,wangjmlr2020} further showed that 
the sample complexity for any $1$-Lipschitz convex GLM can be reduced to only linear in $p$ and exponential in $\alpha^{-1}$, which extends a results in  \citep{zheng2017collect} whose sample complexity is linear in $p$ and quasi-polynomial in $\alpha^{-1}$ for smooth GLMs. In this paper, we show, for the first time, that the sample complexity of GLMs can be reduced to fully polynomial with the help of some public but unlabeled data under some mild assumptions.
There are also some results for specific loss functions. For example, 
\citep{wangicml19linear} studied the high dimensional sparse linear regression problem and \citep{daniely2018learning} considered the problem of PAC learning halfspaces with polynomial samples. Since these results are only for some special loss functions (instead of a family of functions),  they are incomparable with ours. 

As we mentioned earlier, there is a long list of work studies GLMs in the central DP model and the interactive LDP model. In the central DP model, \cite{jain2014near} provided the first study and showed that to achieve an error $\alpha$ of the excess population risk, there is an $(\epsilon, \delta)$-DP algorithm with sample complexity $\tilde{O}(\epsilon^{-2}\alpha^{-2})$. Recently, \cite{song2021evading} showed a sharper sample complexity bound of $\tilde{O}(\sqrt{\text{rank}}\epsilon^{-1}\alpha^{-1})$, where $\text{rank}$  is the rank of the feature matrix formed by stacking the feature vectors as column, which always holds that  $\text{rank} \leq n$. \cite{bassily2021differentially} provided an algorithm  which runs in (nearly) linear time instead of super-linear in the previous work and its sample complexity is $\tilde{O}(\max\{\sqrt{\text{rank}}\epsilon^{-1}\alpha^{-1}, \alpha^{-2}\})$. They also extended from the $\ell_2$-norm Lipschitz case and the convex setting to the  $\ell_1$-norm Lipschitz case and the weakly-convex setting. \cite{arora2022differentially} studied DP-GLM where the loss is
smooth and non-negative but not necessarily Lipschitz. They showed a near optimal sample complexity, which is $\tilde{O}(\max\{\alpha^{-2}, \alpha^{-\frac{3}{2}}\epsilon^{-1}, \sqrt{p}\epsilon^{-1}\alpha^{-1}\})$ (if $\|w^*\|_2\leq 1$). Besides convex loss functions, \cite{arora2022faster} recently studied non-convex GLMs in the DP model and showed that to achieve an error $\alpha$ of the $\ell_2$-norm of the gradient of the population risk function, there is an $\epsilon, \delta)$-DP algorithm 
whose  sample complexity is $\tilde{O}(\max\{\alpha^{-2}, \sqrt{\text{rank}}\epsilon^{-1}\alpha^{-1}, \alpha^{-\frac{5}{2}}\epsilon^{-1}\})$. \cite{cai2020cost} studied DP-GLM under statistical settings. In the low dimensional case where the covariate $x$ satisfies $\|x\|_2\leq \sqrt{p}$, to achieve an $\alpha$  $\ell_2$-norm estimation error, they provided an algorithm with a near optimal sample complexity of $\tilde{O}(\max\{p\alpha^{-1}, p\epsilon^{-1}\alpha^{-\frac{1}{2}}\})$. Moreover, under the high dimensional sparse setting, in the case where $\|x\|_\infty\leq 1$ and with some additional assumptions, they presented an algorithm with sample complexity $\tilde{O}(\max\{s^*\epsilon^{-1}, s^*\epsilon^{-1}\alpha^{-\frac{1}{2}}\})$, where $s^*$ is the underlying sparsity of $w^*$.  \citep{hu2022high}  recently generalized these results to the case where the covariates are heavy-tailed. In the interactive LDP model, \cite{duchi2013local} provided the first study on ERM in the sequentially interactive LDP model and showed the (nearly) optimal minimax rate of sample complexity should be $\tilde{O}(p\epsilon^{-2}\alpha^{-2})$ to achieve an error $\alpha$ of the excess population risk when the loss function is $\ell_2$-norm Lipschitz. 

\section{Preliminaries}
Since in this paper we mainly focus on multivariate Gaussian and sub-Gaussian covariates, we first recall some definitions. More details can be found in \citep{vershynin2018high}. 
\begin{definition}[Sub-Gaussian]\label{def:a5.1}
For a given constant $\kappa$, a random variable $x\in\mathbb{R}$ is said to be sub-Gaussian if it satisfies $\sup_{m\geq 1}\frac{1}{\sqrt{m}}\mathbb{E}[|x|^m]^\frac{1}{m}\leq \kappa$. The smallest such $\kappa$ is the {\bf sub-Gaussian norm} of $x$ and it is denoted by $\|x\|_{\psi_2}$. A random vector $x\in \mathbb{R}^p$ is called a sub-Gaussian vector if there exists a constant $\kappa$ such that for any unit vector $v$,  we have $\|\langle x, v\rangle \|_{\psi_2}\leq \kappa$. 
\end{definition}
For sub-Gaussian data, we need the following assumptions on its distribution throughout the paper. 
\begin{assumption}\label{ass:2}
For a random vector $x$ that is sub-Gaussian with zero mean and covariance matrix $\Sigma$, 
we assume the following conditions hold 	  
\begin{itemize}
	\item Its distribution is supported on a $\ell_1$-norm ball of radius $r$.
	\item  For the matrix $\Sigma$,  its corresponding $\Sigma^{\frac{1}{2}}$ is diagonally dominant, where $\Sigma^{\frac{1}{2}}$ is the square root of matrix $\Sigma$.  \footnote{A square matrix is said to be diagonally dominant if, for every row of the matrix, the magnitude of the diagonal entry in a row is larger than or equal to the sum of the magnitudes of all the other (non-diagonal) entries in that row. For a semi-definite positive  matrix $M\in\mathbb{R}^{p\times p}$, let its SVD composition be  $M=U^T\Sigma U$, where $\Sigma=\text{diag}(\lambda_1, \cdots, \lambda_p)$, then $M^{\frac{1}{2}}$ is defined as $M^{\frac{1}{2}}=U^T\Sigma^{\frac{1}{2}}U$, where $\Sigma^{\frac{1}{2}}=\text{diag}(\sqrt{\lambda_1}, \cdots, \sqrt{\lambda_p})$.}
	\item Let $v=\Sigma^{-\frac{1}{2}}x$ be the whitened random vector of $x$, each  $v_i$ has constant first and second conditional moments, {\em i.e.,} $\forall j\in[p]$ and $\tilde{w}=\Sigma^{\frac{1}{2}}w^*$,  $\mathbb{E}[v_{ij}|\sum_{k\neq j}\tilde{w}v_{ik}]=O(1)$ and $\mathbb{E}[v^2_{ij}|\sum_{k\neq j}\tilde{w}v_{ik}]=O(1)$.  
\end{itemize}
\end{assumption}
In Assumption \ref{ass:2} there are three terms. The first one is natural as it has also been used in the previous studies on DP-GLM.  For the other two terms, we note that they are crucial for Lemma \ref{lemma:2} and Theorem \ref{thm:5}, which are only used in utility analysis. Thus,  even these two assumptions do not hold, we still have the privacy guarantees. Moreover, it is straightforward to observe that  when the whitened covariates $v$ have independent, but not necessarily identical entries, these two terms hold. We leave it as an open problem to further relax these assumptions. 
\paragraph{Differential Privacy (DP):} In DP, we have data universe $\mathcal{X}\subseteq \mathbb{R}^p$ and $\mathcal{Y}\subseteq \mathbb{R}$, and a dataset $D \in (\mathcal{X} \times \mathcal{Y})^n$ whose size is $n$ and the dataset is stored in some trusted curator. Each data record $(x, y)\in \mathcal{D}$ sampled from some distribution $\mathcal{P}$, where $x\in \mathbb{R}^p$ is the feature vector and $y\in \mathbb{R}$ is the label of response. We say that two datasets $D,D'\subseteq \mathcal{X}$ are neighbors if they differ by only one data record, which is denoted as $D \sim D'$.
\begin{definition}[Differential Privacy \citep{dwork2006calibrating}]\label{def:3.1}
	We call a randomized algorithm $Q$ is $(\epsilon,\delta)$-differentially private (DP) if for all neighboring datasets $D,D'$ and for all events $E$ in the output space of $Q$, the following holds
	\[\mathbb{P}(Q(D)\in E)\leq e^{\epsilon} \mathbb{P}(Q(D')\in E)+\delta.\]
	When $\delta = 0$,  $\mathcal{A}$ is $\epsilon$-DP.
\end{definition}

\paragraph{Local Differential Privacy (LDP):} Instead of the trusted curator, in LDP model \citep{kasiviswanathan2011can}, each player (data provider) perturb his/her private data record locally via some DP algorithms before sending it to the curator. Specifically,  there are $n$ players with each holding a private data record $(x, y) \in \mathcal{X}\times \mathcal{Y}$ sampled  from some distribution $\mathcal{P}$, and a server that is in charge of coordinating the protocol. An LDP protocol proceeds in $T$ rounds. In each round, the server sends a message, which is often called a query, to a subset of the players, requesting them to run a particular algorithm. Based on the query, each player $i$ in the subset selects an algorithm $Q_i$, runs it on her/his own data, and sends the output back to the server.
\begin{definition}[Local Differential Privacy \citep{kasiviswanathan2011can}]\label{def:1}
A randomized algorithm $Q$ is $(\epsilon, \delta)$-locally differentially private (LDP) if for all pairs $x,x'\in \mathcal{D}$, and for all events $E$ in the output space of $Q$, we have $$\mathbb{P}(Q(x)\in E)\leq e^{\epsilon}\mathbb{P}(Q(x')\in E)+\delta.$$ When $\delta = 0$,  $\mathcal{A}$ is $\epsilon$-LDP. A multi-player protocol is $(\epsilon,\delta)/\epsilon$-LDP if for all possible inputs and runs of the protocol, the transcript of player i's interaction with the server is $(\epsilon,\delta)/\epsilon$-LDP. If $T=1$, we say that the protocol is $(\epsilon,\delta)/\epsilon$ \textbf{non-interactive LDP (NLDP)}.
\end{definition}
In this paper, we will mainly focus on $(\epsilon, \delta)$-NLDP and we will mainly use the Gaussian mechanism \citep{dwork2006calibrating} to guarantee $(\epsilon, \delta)$-LDP.
\begin{lemma}[Gaussian Mechanism \citep{dwork2006calibrating}]\label{lemma:gaussian}
Given any function $q : (\mathcal{X}\times \mathcal{Y})^n\rightarrow \mathbb{R}^d$, the Gaussian mechanism is defined as $\mathcal{M}_G(D,q,\epsilon)=q(D)+ Y,$
	where Y is drawn from Gaussian Distribution $\mathcal{N}(0,\sigma^2I_d)$ with  $\sigma\geq \frac{\sqrt{2\ln(1.25/\delta)}\Delta_2(q)}{\epsilon}$. Here $\Delta_2(q)$ is the $\ell_2$-sensitivity of the function $q$, i.e.,\ $\Delta_2(q)=\sup_{D\sim D'}||q(D)-q(D')||_2.$
	Gaussian mechanism preserves $(\epsilon,\delta)$-differential privacy.
\end{lemma}
\paragraph{Our Model:} Different from the above classical NLDP  model  where only one private dataset $D=\{(x_i, y_i)\}_{i=1}^n$ exists, the NLDP model in our setting allows
the server to have an additional public but unlabeled dataset $D'=\{x_j\}_{j=n+1}^{n+m}\subset \mathcal{X}^m$, where each $x_j$ is sampled from  $\mathcal{P}_x$, which 
is the marginal distribution of $\mathcal{P}$ ({\em i.e.,} it has the same distribution as each $x_i$). 


\section{Privately Estimating Generalized Linear Models}\label{eff_glms}

In this section, we  study GLMs in our privacy  model and we aim to privately estimate $w^*$ in (\ref{eq:1}) by using both  private data $\{(x_i, y_i)\}_{i=1}^n$ and public unlabeled data $\{x_j\}_{j=n+1}^{n+m}$. Our goal is to achieve a fully polynomial  sample complexity for $n$ and $m$, {\em i.e.,} $n, m = \text{Poly}(p, \frac{1}{\epsilon}, \frac{1}{\alpha}, \log \frac{1}{\delta})$, such that there is an $(\epsilon, \delta)$-NLDP algorithm with estimation error less than $\alpha$ (with high probability).   

\subsection{Gaussian Case}
We first consider a simpler case that each data record is sampled from some unknown Gaussian distribution  $ \mathcal{N}(0, \Sigma)$. The idea of our method is motivated by the following result, which is derived from the Stein's lemma \citep{brillinger2012generalized}.

\begin{lemma}[\citep{brillinger2012generalized}]\label{lemma:1}
If $x\sim \mathcal{N}(0, \Sigma)$,  then $w^*$ in (\ref{eq:1}) can be written as 
$w^* = c_\Phi \times w^{ols},$
where $c_\Phi$ is the fixed point of $z\mapsto (\mathbb{E}[\Phi^{(2)}(\langle x, w^{ols}\rangle z )])^{-1}$ (if we assume that  $\mathbb{E}[\Phi^{(2)}(\langle x, w^{ols}\rangle z )]\neq 0$)  and $w^{ols}= \Sigma^{-1}\mathbb{E}[xy]$ is the Ordinary Least Squares (OLS) vector. \footnote{$\Phi^{(2)}$ is the second order derivative function of function $\Phi$, similar for $\Phi^{(3)}$  in the later sections.}
\end{lemma}
From Lemma \ref{lemma:1}, we can see that to estimate $w^*$, it is sufficient to estimate $w^{ols}$ and its corresponding constant $c_\Phi$. Specifically, to estimate $w^{ols}$ in a non-interactive local differentially private manner, 
a direct way is to let 
each player perturb her/his sufficient statistics, {\em i.e.,} $x_ix_i^T$ and $y_i x_i$. After receiving the private OLS estimator $\hat{w}^{ols}$,\footnote{Note that when $n$ is large enough we can show $\hat{w}^{ols}$ is well defined, see Appendix for details.}  the server can then estimate the constant $c_\Phi$ by using the public unlabeled data and $\hat{w}^{ols}$.
From the definition, it is easy to see that $c_\Phi$ is independent of the label $y$. Thus, $c_\Phi$ can be estimated by  using the empirical version of $\mathbb{E}[\Phi^{(2)}(\langle x, w^{ols}\rangle z )]$. That is,  find the root of the function  $1- \frac{c}{m}\sum_{j=n+1}^{n+m}\Phi^{(2)}(c\langle x_j, \hat{w}^{ols}\rangle )$.  Several methods are available for finding roots, and in our algorithms we will use the Newton's method which has 
a quadratic convergence rate. 

{However, there is a challenge for this approach. That is, Lemma \ref{lemma:1} needs to assume $x$ is Gaussian, which implies that the  sensitivity of the terms $\|x_ix_i^T\|_F$ and $\|y_ix_i\|_2$ could be unbounded. To address this issue, we will use the concentration inequality on the $\ell_2$-norm of Gaussian distributions, and clip each $x_i$ to let it has bounded $\ell_2$-norm.  Specifically, we are motivated by the following lemma:}
\begin{lemma}[Gaussian case of \citep{hsu2012tail}]\label{lemma:gaussian_bound}
	Let $x\sim \mathcal{N}(0, \Sigma)\in \mathbb{R}^p$. For all $t>0$, 
	\begin{equation}
		\mathbb{P}(\|x\|_2^2\geq \text{trace}(\Sigma)+2\sqrt{\text{trace}(\Sigma^2)t}+2\|\Sigma\|_2t)\leq e^{-t}. 
	\end{equation} 
\end{lemma}

{	Since $\text{trace}(\Sigma)\leq p\|\Sigma\|_2$ and $\text{trace}(\Sigma^2)\leq (\text{trace}(\Sigma))^2$, from Lemma \ref{lemma:gaussian_bound} we have with probability at least $1-\frac{1}{n^2}$, $\|x\|_2\leq r\equiv \sqrt{10 p\|\Sigma\|_2\log n}$. Motivated by this we clip each $x_i$ to $\bar{x}_i=x_i\min\{1, \frac{r}{\|x_i\|_2}\}$ and now the terms $\|\bar{x}_i\bar{x}_i^T\|_F$ and $\|y_i\bar{x}_i\|_2$ are bounded. 
 However, we can see the clipping threshold depends on the term of $\|\Sigma\|_2$, which is unknown in advance. To estimate this term, we can use the empirical covariance matrix of the public data $\{x_j\}_{j=n+1}^{n+m}$. See Algorithm \ref{alg:0} for details.}
 
 \begin{algorithm}[!ht]
\caption{Non-interactive LDP for smooth GLMs with public data (Gaussian)	\label{alg:0}}
	\begin{algorithmic}[1]
		\State {\bfseries Input:} Private data $\{(x_i, y_i)\}_{i=1}^n \in  (\mathbb{R}^p\times [0, 1])^n$, where  $|y_i|\leq 1$, $\{x_i\}_{j=1}^{n+m}\sim \mathcal{N}(0, \Sigma)$ for some unknown $\Sigma$ and  $\{x_j\}_{j=n+1}^{n+m}$ are public, loss function $\Phi:\mathbb{R}\mapsto \mathbb{R}$, privacy parameters $\epsilon, \delta$, and initial value $c\in \mathbb{R}$.
	\For {The server}
	\State Calculate $\Sigma_m=\frac{1}{m}\sum_{j=n+1}^{n+m}x_jx_j^T$ and send it to each user. 
	\EndFor 
     \For{Each user $i\in [n]$}
     \State Let $\bar{x}_i=x_i\min\{1, \frac{r}{\|x_i\|_2}\} $, where $r\equiv \sqrt{20 p\|\Sigma_m\|_2\log n}$. 
     
     \State Release $\widehat{{x}_i {x}_i^T}= \bar{x}_i\bar{x}_i^T + E_{1,i}$ and  $\widehat{x_iy_i}=\bar{x}_iy_i+ E_{2, i}$, where $E_{1,i} \in \mathbb{R}^{p\times p} $ is a symmetric matrix and each entry of  the upper triangle matrix is sampled from $\mathcal{N}(0, \frac{32r^4\log\frac{2.5}{\delta}}{\epsilon^2})$ and $E_{2,i}\in \mathbb{R}^{p}$ is sampled from $\mathcal{N}(0, \frac{32r^2\log \frac{2.5}{\delta}}{\epsilon^2}I_p)$.
     \EndFor 
     
     \For {The server}
     \State Let $\widehat{X^TX}=\sum_{i=1}^n \widehat{x_ix_i^T}$ and $\widehat{X^Ty}=\sum_{i=1}^n \widehat{x_iy_i}$. Calculate $\hat{w}^{ols}=(\widehat{X^TX})^{-1}\widehat{X^Ty}$. 
     
     \State Calculate $\tilde{y}_{j}= x_j^T\hat{w}^{ols}$ for each $j=n+1, \cdots, n+m$.\;
     
     \State  Find the root $\hat{c}_{\Phi}$ such that $1= \frac{\hat{c}_{\Phi}}{m}\sum_{j=n+1}^{n+m}\Phi^{(2)}(\hat{c}_{\Phi}\tilde{y}_j)$ by using Newton's root-finding method (or other methods):
     \For{$t=1, 2, \cdots$ until convergence}
    \State $c= c- \frac{c\frac{1}{m}\sum_{j=n+1}^{n+m}\Phi^{(2)}(c\tilde{y}_j)-1}{
     \frac{1}{m}\sum_{j=n+1}^{n+m}\{\Phi^{(2)}(c\tilde{y}_j)+c\tilde{y}_j\Phi^{(3)}(c\tilde{y}_j)\}}$. 
     \EndFor 
     \EndFor \\
     \Return {$\hat{w}^{glm}= \hat{c}_{\Phi}\cdot \hat{w}^{ols}$.} 
	\end{algorithmic}
\end{algorithm}

\begin{theorem}\label{thm:-1}
For any $0<\epsilon, \delta<1$, Algorithm \ref{alg:0} is $(\epsilon, \delta)$ non-interactive LDP. 
\end{theorem}
Next we will show the estimation error bound of the output in Algorithm \ref{alg:0}, before that we need the following assumptions for loss functions. 
\begin{assumption} \label{ass:1}
	 We assume
	\begin{itemize}
		\item  $|\Phi^{(2)}(\cdot)|\leq L$ and $\Phi^{(3)}(\cdot)$ is $G$-Lipschitz.
		\item There exist constants $\bar{c}$ and $\tau>0$, the function $f(c)=c\mathbb{E}[\Phi^{(2)}(\langle x, w^{ols}\rangle c)]$ satisfies the condition of $f(\bar{c})\geq 1+\tau$, where $w^{ols}$ is in Lemma \ref{lemma:1} and $x\sim \mathcal{N}(0, \Sigma)$.
		\item  The derivative of $f$ in the interval $[0, \bar{c}]$ does not change the sign, {\em i.e.,} its absolute value is lower bounded by some constant $M>0$. 
	\end{itemize}

\end{assumption}
Note that the first condition ensures that $\Phi^{(1)}$ is Lipschitz, and the second and the last conditions are to ensure that the function $f-1$ has a root and $\hat{c}_\Phi$ close to $c_\Phi$ for large enough $m$, see Theorem \ref{theorem:11} and \ref{theorem:12} for some concrete instances that satisfy the assumption. 

\begin{theorem}\label{thm:0}
	Let $x_1, \cdots, x_n\in \mathbb{R}^p$ be i.i.d realizations of a random vector $x\sim \mathcal{N}(0, \Sigma)$. Moreover, under Assumption \ref{ass:1},  for sufficiently large $m, n$ such that 
	\begin{align}
   & n \geq \tilde{\Omega} \big( \frac{  p^3 \|w^*\|_2^2 \log\frac{1}{\delta}\log \frac{1}{\xi}}{\epsilon^2}\big)\\
   &  m\geq \Omega \big( \|w^*\|_2^2 p).
\end{align}
Then for any $\zeta\in (0, 1)$, with probability at least $1-\exp(-\Omega(p))-\xi$ 

\begin{equation*}
	\|\hat{w}^{glm}-w^*\|_2\leq \tilde{O}\big(\frac{p^{\frac{3}{2}}\|w^*\|^2_2\sqrt{\log \frac{1}{\delta}\log \frac{1}{\xi}} 
        }{\epsilon\sqrt{n} }+\frac{\sqrt{p} \|w^*\|^2_2}{\sqrt{m}}\big), 
\end{equation*}
where Big-$\tilde{O}$ and Big-$\tilde{\Omega}$ notations  omit the terms of $\|\Sigma\|_2, G, L,  M, \bar{c}, \tau, { c_\Phi}$, $\lambda_{\min}(\Sigma)$ (the smallest eigenvalue of $\Sigma$) and other logarithmic factors  (see Appendix for the explicit form of $m$ and $n$). 
\end{theorem}
	Theorem \ref{thm:0} suggests that if  $\|w^*\|_2=O(1)$, then for any given error $\alpha$, there is an $(\epsilon, \delta)$-NLDP algorithm whose sample complexity of  private ($n$) and public unlabeled ($m$) data, to achieve the  $\ell_2$-norm error of $\alpha$, is $\tilde{O}(p^3\epsilon^{-2}\alpha^{-2})$ and $O(p\alpha^{-2})$, respectively. We note that $m\leq n$, which means that the sample complexity of the public data is less than that of the private data. We also note that the sample complexity of the public data is independent of the privacy parameters $\epsilon$ and $\delta$. 

	Actually, there is one possible way to improve the practical performance of Algorithm \ref{alg:0} (and all other algorithms in the paper). The key observation is that, in the procedure of estimating the OLS estimator, the empirical covariance matrix $\widehat{X^TX}$ does not depend on labels. Thus, we can further use those public unlabeled data to give a more precise estimator of the covariance matrix. That is, we can let $\widehat{X^TX}=\frac{1}{m+n}(\sum_{i=1}^n\widehat{x_ix_i^T}+\sum_{j=n+1}^{n+m} x_jx_j^T)$ and $\widehat{X^Ty}=\frac{1}{n}\sum_{i=1}^n\widehat{x_iy_i}$. However, by using the similar proof as in the proof of Theorem \ref{thm:0}, 
	we can see that the upper bound of error will be asymptotically the same as the bound in Theorem  \ref{thm:0} (and all other theorems in the paper). In the experimental section we will adopt this improved approach. 
\begin{remark}\label{remark:9}
	It is notable that the public  dataset is only used in line 11-15 of Algorithm \ref{alg:0} (similar to other algorithms), where we use it to find a root of some function. Actually we can adjust our idea to a 2-round LDP algorithm in the canonical model ({\em i.e.,} there is no public unlabeled data). That is, in the first round we get $\hat{w}^{ols}$ by using half privacy budget and the server sends it to all the users. In the second round, each user uses  another half  privacy budget to compute a noisy version of $\tilde{y}_j= x_j^T \hat{w}_{ols}$ and sends it to the server. Then the server uses these noisy version of $\tilde{y}_j$ to estimate the constant of ${c}_{\Phi}$.
	    We note that due to the noise we added in the second round for each term of $\tilde{y}_j$, there could be a large amount of error when using  $\hat{c}_{\Phi}$ to estimate $c_\Phi$, and this will cause the private estimator has large error. In the experiments section, we will practically show that this approach will leads worse performance. 
\end{remark}

\subsection{Sub-Gaussian Case}
The main weakness of the previous result is that  due to Lemma \ref{lemma:1}, Theorem \ref{thm:0} only holds for  Gaussian distributions. Fortunately, recently \citet{erdogdu2019scalable} generalized the Stein's lemma to bounded sub-Gaussian random vectors. Compared with the Gaussian case, in this case there is  an additional additive error of $O(\frac{\|w^*\|_\infty^2}{\sqrt{p}})$. Formally, we have the following lemma. 
\begin{algorithm}[!ht]
\caption{Non-interactive LDP for smooth GLMs with public data (General)	\label{alg:1}}
	\begin{algorithmic}[1]
		\State {\bfseries Input:} Private data $\{(x_i, y_i)\}_{i=1}^n\subset (\mathbb{R}^p\times [0, 1])^n$, where $\|x_i\|_1\leq r$ and $|y_i|\leq 1$, public unlabeled data $\{x_j\}_{j=n+1}^{n+m}$, loss function $\Phi:\mathbb{R}\mapsto \mathbb{R}$, privacy parameters $\epsilon, \delta$, and initial value $c\in \mathbb{R}$.
 \For{Each user $i\in [n]$}
     \State Release $\widehat{x_ix_i^T}= x_ix_i^T + E_{1,i}$, where $E_{1,i} \in \mathbb{R}^{p\times p} $ is a symmetric matrix and each entry of  the upper triangle matrix is sampled from $\mathcal{N}(0, \frac{32r^4\log\frac{2.5}{\delta}}{\epsilon^2})$.
     
    \State   Release $\widehat{x_iy_i}=x_iy_i+ E_{2, i}$, where $E_{2,i}\in \mathbb{R}^{p}$ is sampled from $\mathcal{N}(0, \frac{32r^2\log \frac{2.5}{\delta}}{\epsilon^2}I_p)$. 
     \EndFor 
     \For {The server}
     \State Let $\widehat{X^TX}=\sum_{i=1}^n\widehat{x_ix_i^T}$ 
     and $\widehat{X^Ty}=\sum_{i=1}^n\widehat{x_iy_i}$. Calculate $\hat{w}^{ols}=(\widehat{X^TX})^{-1}\widehat{X^Ty}$. 
     \State Calculate $\tilde{y}_{j}= x_j^T\hat{w}^{ols}$ for each $j=n+1, \cdots, n+m$. \;
     \State Find the root $\hat{c}_{\Phi}$ such that $1= \frac{\hat{c}_{\Phi}}{m}\sum_{j=n+1}^{n+m}\Phi^{(2)}(\hat{c}_{\Phi}\tilde{y}_j)$ by using Newton's root-finding method (or other methods):
     \For{$t=1, 2, \cdots$ until convergence}
    \State $c= c- \frac{c\frac{1}{m}\sum_{j=n+1}^{n+m}\Phi^{(2)}(c\tilde{y}_j)-1}{
     \frac{1}{m}\sum_{j=n+1}^{n+m}\{\Phi^{(2)}(c\tilde{y}_j)+c\tilde{y}_j\Phi^{(3)}(c\tilde{y}_j)\}}$. 
     \EndFor 
     \EndFor\\
     \Return {$\hat{w}^{glm}= \hat{c}_{\Phi}\cdot \hat{w}^{ols}$.} 
	\end{algorithmic}
\end{algorithm}

\begin{lemma}[\citep{erdogdu2019scalable}]\label{lemma:2}
Let $x_1, \cdots, x_n\in \mathbb{R}^p$ be i.i.d realizations of a random vector $x$ that is zero-mean sub-Gaussian with covariance matrix $\Sigma$ and satisfies Assumption \ref{ass:2}. Let $v=\Sigma^{-\frac{1}{2}}x$ be the whitened random vector of $x$ and denote  $\|v\|_{\psi_2}=\kappa_x$. If  the function $\Phi^{(2)}$ is Lipschitz with constant $G$, then for $c_{\Phi}=\frac{1}{\mathbb{E}[\Phi^{(2)}(\langle x_i, w^*\rangle)] }$ (assuming  $\mathbb{E}[\Phi^{(2)}(\langle x_i, w^*\rangle)] \neq 0$), the following holds for GLM in (\ref{eq:1})
{ 
\begin{equation}\label{eq:2}
    \|\frac{1}{c_{\Phi}}\cdot w^*-w^{ols}\|_{\infty}\leq O(Gr\kappa_x^3\sqrt{\rho_2}\rho_{\infty}\frac{\|w^*\|^2_\infty}{\sqrt{p}}),  
\end{equation}
}
where $\rho_q$ for $q=\{2, \infty\}$ is the conditional number of $\Sigma$ in $\ell_q$-norm, {\em i.e.,} $\rho_q=\|\Sigma\|_q \|\Sigma^{-1}\|_q$ where $\|A\|_q=\sup_{x\neq 0}\frac{\|Ax\|_q}{\|x\|_q}$ for matrix $A$,  and $w^{ols}=\Sigma^{-1}\mathbb{E}[xy]$ is the OLS vector.  
\end{lemma}

Lemma \ref{lemma:2} indicates that we can use the same idea as in the previous section to estimate $w^*$. Note that the forms of constant $c_\Phi$ in Lemma \ref{lemma:1} and \ref{lemma:2} are different while one depends on $w^{ols}$ and the other one depends on $w^*$. However, since by (\ref{eq:2}) we know $w^*$ and $c_\Phi w^{ols}$ are close. Thus, intuitively we can still use $\frac{1}{\mathbb{E}[\Phi^{(2)}(\langle x_i, w^{ols}\rangle \bar{c}_{\Phi}]) }$ to approximate $c_\Phi$, where $\bar{c}_\Phi$ is the root of $c\mathbb{E}[\Phi^{(2)}(\langle x_i, w^{ols}\rangle c)]-1$ which could be approximated by using public unlabeled data. 
Combining these ideas, we present Algorithm \ref{alg:1}. 

\begin{theorem}\label{thm:1}
For any $0<\epsilon, \delta<1$, Algorithm \ref{alg:1} is $(\epsilon, \delta)$ non-interactive LDP. 
\end{theorem}

The following theorem shows the sample complexity of the bounded sub-Gaussian distributions. Similar to Assumption \ref{ass:1}, we need the following assumptions for loss functions. 
\begin{assumption}\label{ass:3}
		 We assume
	\begin{itemize}
		\item  $|\Phi^{(2)}(\cdot)|\leq L$ and $\Phi^{(3)}(\cdot)$ is G-Lipschitz. 
		\item  For some constant $\bar{c}$ and $\tau>0$, the function $f(c)=c\mathbb{E}[\Phi^{(2)}(\langle x, w^{ols}\rangle c)]$ satisfies the condition of $f(\bar{c})\geq 1+\tau$,  where $w^{ols}$ is in Lemma \ref{lemma:2} and the distribution of $x$ satisfies Assumption \ref{ass:2}. 
		\item  The derivative of $f$ in the interval $[0, \max\{\bar{c}, c_\Phi\}]$ does not change the sign ({\em i.e.,} its absolute value is lower bounded by some constant $M>0$), where $c_\Phi$ is in Lemma \ref{lemma:2}. 
	\end{itemize}
\end{assumption}
It seems that Assumption \ref{ass:3} is almost the same as Assumption \ref{ass:1}. However, since  these two assumptions rely on the underlying distribution of $(x, y)$, which are different in these two cases. Thus, the two assumptions are different. Moreover, the third conditions in Assumption \ref{ass:3} and Assumption \ref{ass:1} are  different due to different intervals and  different forms of $c_\Phi$. 
\begin{theorem}\label{thm:3}
Under Assumption \ref{ass:2} and \ref{ass:3}, for sufficiently large $m, n$ such that 
\begin{align}\label{eq:5}
&m\geq \Omega \big(\|w^*\|^2_\infty\max\{1, \|w^*\|^2_\infty\} p^2\big), \nonumber \\
&n\geq \tilde{\Omega}\big( \frac{  p^2 \|w^*\|_\infty^2\max\{1, \|w^*\|_\infty^2\} \log\frac{1}{\delta}\log \frac{1}{\xi}}{\epsilon^2}\big). 
\end{align}
Then for any $\zeta\in (0, 1)$, with probability at least $1-\exp(-\Omega(p))-\xi$, the output $\hat{w}^{glm}$ in Algorithm \ref{alg:1} satisfies
{
\begin{multline}\label{eq:6} 
    \|\hat{w}^{glm}-w^*\|_\infty \leq 
     \tilde{O}\big( \frac{\|w^*\|^2_\infty\max\{1, \|w^*\|^2_\infty\} p
        }{\sqrt{m} } \\+
        \frac{ \|w^*\|^2_\infty\max\{1, \|w^*\|^2_\infty\} p\sqrt{\log \frac{1}{\delta}\log \frac{1}{\xi}} 
        }{\epsilon\sqrt{n} } + \frac{ \|w^*\|^3_\infty\max\{1, \|w^*\|_\infty\}}{\sqrt{p}} \big),  
\end{multline}
}
where Big-$\tilde{O}$ and Big-$\Omega$  notations omit the terms of  $\|\Sigma\|_2, \rho_2, \rho_{\infty}, G, L, \tau, M, \bar{c}, r, \kappa_x, {c_\Phi}$ and $\lambda_{\min}(\Sigma)$, and other logarithmic factors (see Appendix for the explicit forms of $m$ and $n$). 
\end{theorem}
\begin{corollary}
	Similar to the Gaussian case, Theorem \ref{thm:3} suggests that if we omit all the other terms and assume that  $\|w^*\|_\infty=O(1)$, then for any given error $\alpha\geq \Omega(\frac{1}{\sqrt{p}})$, there is an $(\epsilon, \delta)$-NLDP algorithm whose sample complexity of  private data  ($n$) and public unlabeled data ($m$)  to achieve an estimation error of $\alpha$ (in $\ell_\infty$-norm), is $\tilde{O}(p^2\epsilon^{-2}\alpha^{-2})$ and $O(p^2\alpha^{-2})$, respectively. While compared with the Gaussian case, here we need larger $m$. However,  as  we will see in the experiments section, in practice we do not need such large size for public data. 
	
	Compared with the complexity for private data in Gaussian case, it seems that the complexity in the sub-Gaussian case is less. However, due to different measure of estimation error ($\ell_2$-norm v.s. $\ell_\infty$-norm) and different assumptions ($\|w^*\|_2=O(1)$ v.s. $\|w^*\|_\infty=O(1)$), these two results are incomparable. 

\end{corollary}

Compared with the previous work on linear regression in NLDP model. It seems that our sample  complexities for the  general GLMs are worse than the previous results. However, these results are incomparable due to different settings and assumptions. Specifically, when $\|x_i\|_2\leq 1$ and $\|w^*\|_2\leq 1$,   \citet{smith2017interaction} proposed an algorithm with a sample complexity of $\tilde{O}(p\alpha^{-2}\epsilon^{-2})$ for the optimization error. While in this paper we mainly focus on the estimation error. Moreover, 
in the Gaussian case,  we have each $\|x_i\|_2\leq O(\sqrt{p})$ with high probability and in the sub-Gaussian case we assume $\|w^*\|_\infty\leq O(1)$, these assumptions are different with the assumptions in \citep{smith2017interaction}. \citet{zheng2017collect} proposed an algorithm whose sample complexity is $\tilde{O}(\alpha^{-4}\epsilon^{-2})$ for the optimization error, under the assumptions of $\|x_i\|_1\leq 1$ and $\|w^*\|_1\leq 1$, which are also different with ours. Recently, \citet{wangicml19linear} also considered the $\ell_2$-norm statistical error,  it relies on assumptions that $w^*$ is 1-sparse, which is not needed in our setting. Besides these differences, we also have to mention that in this paper we need some additional assumptions ({\em} i.e., Assumption \ref{ass:2}) on the data distribution compared with the those previous results.

\begin{remark}
Algorithm \ref{alg:0} and \ref{alg:1} have several advantages over the existing approaches. Firstly, different from the approaches that are based on (Stochastic) Gradient Descent methods to solve DP-ERM ({\em e.g.,} \citep{wang2017differentially}), our algorithm is parameter-free. That is, we do not need to choose a specific step size, an iteration number or initial vectors. Secondly, compared with some previous work on GLM in NLDP model such as \citep{zheng2017collect,smith2017interaction,wang2019noninteractive}, all of our above results do not need to assume that the loss function $\Phi(\cdot)$ is convex. Thirdly, since the private data only contributes to obtaining the OLS estimator, and only the constant $\hat{c}_\Phi$ depends on the loss function $\Phi$, these indicate that with probability at least $1-T\exp(-\Omega(p))-\xi$, our algorithm can simultaneously be implemented on $T$ different loss functions to achieve the same error $\alpha$ for each loss with almost the same sample complexity as in Theorem \ref{thm:3} or Theorem \ref{thm:0} (if they all satisfy the corresponding assumption).  This implies that we can answer at most $O(\exp(O(p))$ number of GLM queries with constant probability to achieve error $\alpha$ for each query with the same sample complexity as in Theorem \ref{thm:3} (Theorem \ref{thm:0}).  To our best knowledge, this is the first result which can answer multiple non-linear queries in the NLDP model with polynomial sample complexity. Previous results are either for linear queries  \citep{blasiok2019towards,bassily2018linear}, or in the central DP model  \citep{ullman2015private}. Moreover, we can see when the dimension $p$ increases, we could answer more GLMs queries. It sounds counter-intuitive that with 
larger dimension, one can handle more loss functions. However, we note that in this case we also need more data samples to achieve the fixed error $\alpha$.
\end{remark}

Note that in Theorem \ref{thm:3}, 
$\Phi^{(2)}$ is assumed to be bounded. Although this is  a  commonly-used assumption in the previous work such as \citep{wang2018empirical,wangicml19}, 
actually this condition can be further relaxed to the condition that 
$\Phi^{(2)}(\langle x, w\rangle)$ is sub-Gaussian in some range of $w$. 
\begin{assumption}\label{ass:40} For a random vector $x$ that is sub-Gaussian with zero mean and covariance matrix $\Sigma$, 
we assume the following conditions hold 	  
	\begin{itemize}
		\item $\sup_{w: \|w-\Sigma^{\frac{1}{2}}w^{ols}\|_2\leq 1}\|\Phi^{(2)}(\langle x, w\rangle)\|_{\psi_2}\leq \kappa_g$ for some constant $\kappa_g$ and $\Phi^{(3)}(\cdot)$ is G-Lipschitz. 
		\item  For some constant $\bar{c}$ and $\tau>0$, the function $f(c)=c\mathbb{E}[\Phi^{(2)}(\langle x, w^{ols}\rangle c)]$ satisfies the condition of $f(\bar{c})\geq 1+\tau$,  where $w^{ols}$ is in Lemma \ref{lemma:2} and the distribution of $x$ satisfies Assumption \ref{ass:2}. 
		\item  The derivative of $f$ in the interval $[0, \max\{\bar{c}, c_\Phi\}]$ does not change the sign ({\em i.e.,} its absolute value is lower bounded by some constant $M>0$), where $c_\Phi$ is in Lemma \ref{lemma:2}. 
	\end{itemize}
\end{assumption}

\begin{theorem}\label{thm:4}
Under Assumption \ref{ass:2} and \ref{ass:40}, for sufficiently large $m, n$ such that 
 \begin{align}
      &m\geq  \tilde{\Omega}  \big(\frac{\epsilon^2 np}{(\mathbb{E}[\|x\|_2])^2\|w^{ols}\|_2^2}\big) \label{eq:7}, \\ &n\geq  \tilde{\Omega} \big(\frac{ p^2\|w^*\|_\infty^2 \max\{ 1, \|w^*\|_\infty^2\} \log\frac{1}{\delta}\log\frac{1}{\xi}}{\epsilon^2}\big)\label{eq:7.5}. 
  \end{align}
 
 Then the following holds with probability at least $1-\exp(-\Omega(p))-\xi$, 

\begin{align}\label{eq:8}
     &\|\hat{w}^{glm}-w^*\|_\infty \leq 
   \tilde{O}\big( \frac{p \|w^*\|_\infty\max\{1, \|w^*\|^3_\infty\}\sqrt{\log\frac{1}{\delta}\log\frac{1}{\xi}}}{\epsilon\sqrt{n}} \nonumber +\\
     & \frac{\|w^*\|^2_\infty\max\{1,\|w^*\|^2_\infty\} }{\sqrt{p}}+ \|w^*\|_\infty\max\{1, \|w^*\|_\infty\} \frac{1}{{\mathbb{E}[\|x\|_2]}}\frac{p}{\sqrt{m}}\big), 
\end{align}
where the Big-$\tilde{O}$ and Big-$\tilde{\Omega}$  notations omit the terms of $\rho_2,\rho_\infty,  \|\Sigma\|_2$, $\lambda_{\min}(\Sigma)$, $r,\kappa_x, \kappa_g$, $ G, M, \tau, \bar{c}$, ${c_\Phi}$ and other logarithmic factors (see Appendix for the explicit forms of $m$ and $n$).
\end{theorem}
From the above theorem, we can see that even with more relaxed assumptions,  to achieve the $\ell_\infty$-norm error $\alpha$, the sample complexities in Theorem \ref{thm:4}  is asymptotically same 
as the ones in Theorem \ref{thm:3} up to some logarithmic factors (if we omit other terms and $m, n$ satisfy (\ref{eq:7}) and (\ref{eq:7.5})).


	A not so desirable issue of 
Theorem \ref{thm:0}, \ref{thm:3} and \ref{thm:4} is that  they need quite a few assumptions/conditions. Although some of them commonly appear in some related work, the assumptions on function $f$ seem to be a little weird. 
 Fortunately, this is a not big issue in both practice and theory.  For the theory side, in the following, motivated by \citep{erdogdu2019scalable}, we will provide two examples which satisfy Assumption \ref{ass:1}. 
 Moreover, for the practical side,  as we will see later, our
experiments show that the algorithm actually performs quite well for many loss functions that may not satisfy these assumptions (such as the cubic function).  Also, we note that the error bounds in Theorem \ref{thm:3} and \ref{thm:4} dependent on the $\ell_1$-norm of  of $x_i$,
while the previous results only depend on the $\ell_2$-norm bound  \citep{smith2017interaction,zheng2017collect}. We leave the problem of relaxing/lifting these assumptions for future research.

\begin{theorem}[Logistic Loss]\label{theorem:11}
Consider the model (\ref{eq:1}) where the function $\Phi(z)= \log (1+e^z)$ (then $|\Phi^{(2)}(\cdot)|\leq 1$  and $\Phi^{(2)}(\cdot)$ is $1$-Lipschitz), $x\sim \mathcal{N}(0, \frac{1}{p}I_p)$, $\|w^*\|_2=\frac{\sqrt{p}}{4}$ and $\|w^{ols}\|_2=\frac{\sqrt{p}}{20}$. Then  when $\bar{c}=6$ and $\tau=0.22$, the function $f(c)=c\mathbb{E}[\Phi^{(2)}(\langle x, w^{ols}\rangle c)] >1+\tau$. Moreover, $f'(z)$ is bounded by  constant $M=0.1$ on $[0, \bar{c}]$ from below and $ c_\Phi < \bar{c}$.

\end{theorem}

\begin{theorem}[Boosting Loss]\label{theorem:12}
		Consider the model (\ref{eq:1}) where the function $\Phi(z)= \frac{z}{2}+\sqrt{1+\frac{z^2}{4}}$ (then $|\Phi^{(2)}(\cdot)|\leq \frac{1}{4}$ and $\Phi^{(2)}(\cdot)$ is $\frac{3}{16}$-Lipschitz), $x\sim \mathcal{N}(0, \frac{1}{p}I_p)$, $\|w^*\|_2=\frac{\sqrt{p}}{4}$ and $\|w^{ols}\|_2=\frac{\sqrt{p}}{20}$. Then  when $\bar{c}=6$ and $\tau=0.22$, the function $f(\bar{c})=\bar{c}\mathbb{E}[\Phi^{(2)}(\langle x, w^{ols}\rangle \bar{c})] >1+\tau$. Moreover, $f'(z)$ is bounded by  constant $M=0.1$ on $[0, \bar{c}]$ from below and $ c_\Phi< \bar{c}$.
\end{theorem}

\section{Privately Estimating Non-linear Regressions}
In this section, we extend our ideas in the previous section to the problem of estimating non-linear regressions in  NLDP model with public unlabeled data. Specifically, 
we assume that there is an underlying vector $w^*\in \mathbb{R}^p$ with $\|w^*\|_2\leq 1$ such that
\begin{equation}\label{eq:9}
	y= f(\langle x, w^* \rangle)+ \sigma,
\end{equation}
where $x$ is the feature vector sampled from some distribution (for simplicity, we assume that its mean is zero) and $y$ is the response. $\sigma$ is a zero-mean noise which is independent of $x$ and is bounded by some constant $C=O(1)$ ({\em i.e.,} $\sigma\in [-C, C]$).
$f$ is some known differentiable link function with $f(0)\neq \infty$ \footnote{This assumption can be relaxed to "there is a point $x$ such that $f(x)\neq 0$".}. It is notable  that these assumptions have also been used in some previous work such as \citep{wangicml19linear,duchi2018right} in other privacy models. In our model, the goal is to obtain some estimator $w^{\text{priv}}$ of $w^*$, based on the private dataset $D=\{(x_i, y_i)\}_{i=1}^n$ and the public unlabeled dataset $D'=\{x_j\}_{j=n+1}^{n+m}$ via some NLDP algorithm.
\subsection{Gaussian Case}
Similar to the previous section, we first consider the case where $x\sim N(0, \Sigma)$ with some unknown $\Sigma\in \mathbb{R}^{p\times p}$. Motivate by Lemma \ref{lemma:1}, we first show the following result via the Stein's lemma. 
\begin{theorem}\label{thm:nlrNew1}
If $x\sim \mathcal{N}(0, \Sigma)$,  then $w^*$ in (\ref{eq:9}) can be written as 
$w^* = c_f \times w^{ols},$
where $c_f$ is the fixed point of $z\mapsto (\mathbb{E}[f'(\langle x, w^{ols}\rangle z )])^{-1}$ (if we assume that  $\mathbb{E}[f'(\langle x, w^{ols}\rangle z )]\neq 0$)  and $w^{ols}= \Sigma^{-1}\mathbb{E}[xy]$ is the OLS vector. 
\end{theorem}
We can see that the result in Theorem \ref{thm:nlrNew1} is similar to Lemma \ref{lemma:1} where we replace the function $\Phi^{(2)}(\cdot)$ by function $f'(\cdot)$. Thus, based on the idea of Algorithm \ref{alg:0} we have Algorithm \ref{alg:1.5}.

\begin{algorithm}[!ht]
	\caption{Non-interactive LDP for smooth non-linear regression with public data (Gaussian)\label{alg:1.5}}
	\begin{algorithmic}[1]
		\State {\bfseries Input:} Private data $\{(x_i, y_i)\}_{i=1}^n\subset \mathbb{R}^p\times \mathbb{R}$  with $\{x_i\}_{j=1}^{n+m}\sim \mathcal{N}(0, \Sigma)$ for some unknown $\Sigma$ and  $\{x_j\}_{j=n+1}^{n+m}$ are public, link function $f:\mathbb{R}\mapsto \mathbb{R}$, privacy parameters $\epsilon, \delta$, and initial value $c\in \mathbb{R}$.
	 	\For {The server}
 \State	Calculate $\Sigma_m=\frac{1}{m}\sum_{j=n+1}^{n+m}x_jx_j^T$ and send it to each user. 
	\EndFor
     \For{Each user $i\in [n]$}
  
     \State Let $\bar{x}_i=x_i\min\{1, \frac{r}{\|x_i\|_2}\} $, where $r\equiv \sqrt{20 p\|\Sigma_m\|_2\log n}$. 
  
    \State  Release $\widehat{{x}_i {x}_i^T}= \bar{x}_i\bar{x}_i^T + E_{1,i}$, where $E_{1,i} \in \mathbb{R}^{p\times p} $ is a symmetric matrix and each entry of  the upper triangle matrix is sampled from $\mathcal{N}(0, \frac{32r^4\log\frac{2.5}{\delta}}{\epsilon^2})$.

     \State $\widehat{x_iy_i}=\bar{x}_iy_i+ E_{2, i}$, where the vector $E_{2,i}\in \mathbb{R}^{p}$ is sampled from $\mathcal{N}(0, \frac{32r^2(Lr+|f(0)|+C)^2\log \frac{2.5}{\delta}}{\epsilon^2}I_p)$. 
     \EndFor
     \For {The server}
    \State Denote $\widehat{X^TX}=\sum_{i=1}^n\widehat{x_ix_i^T}$ and $\widehat{X^Ty}=\sum_{i=1}^n\widehat{x_iy_i}$. Calculate $\hat{w}^{ols}=(\widehat{X^TX})^{-1}\widehat{X^Ty}$. 
     \State Calculate $\tilde{y}_{j}= x_j^T\hat{w}^{ols}$ for each $j=n+1, \cdots, n+m$.
  \State Find the root $\hat{c}_{f}$ such that $1= \frac{\hat{c}_{f}}{m}\sum_{j=n+1}^{n+m}f'(\hat{c}_{f}\tilde{y}_j)$ using Newton's root finding method: 
  \For {$t=1, 2, \cdots$ until convergence}
     \State  $c= c- \frac{c\frac{1}{m}\sum_{j=n+1}^{n+m}f'(c\tilde{y}_j)-1}{
     \frac{1}{m}\sum_{j=n+1}^{n+m}\{f'(c\tilde{y}_j)+c\tilde{y}_jf^{(2)}(c\tilde{y}_j)\}}$. 
     \EndFor
     \EndFor \\ 
    \Return {$\hat{w}^{nlr}= \hat{c}_{f}\cdot \hat{w}^{ols}$. }
	\end{algorithmic}
\end{algorithm}

Just as in the previous section, we need the following assumptions for function $f(\cdot)$. 
\begin{assumption}\label{ass:3.5}
		 We assume
	\begin{itemize}
		\item  $|f'(\cdot)|\leq L$ and $f^{(2)}(\cdot)$ is G-Lipschitz. 
		\item  For some constant $\bar{c}$ and $\tau>0$, the function $\ell(c)=c\mathbb{E}[f'(\langle x, w^{ols}\rangle c)]$ satisfies the condition of $\ell(\bar{c})\geq 1+\tau$,  where $w^{ols}$ is in Theorem  \ref{thm:nlrNew1}.
		\item  The derivative of $\ell$ in the interval $[0, \bar{c}]$ does not change the sign, {\em i.e.,} its absolute value is lower bounded by some constant $M>0$. 
	\end{itemize}
\end{assumption}
\begin{theorem}\label{thm:nlrnew2}
For any $0<\epsilon, \delta<1$, Algorithm \ref{alg:1.5} is $(\epsilon, \delta)$ non-interactive LDP.  Moreover, let $x_1, \cdots, x_n\in \mathbb{R}^p$ be i.i.d realizations of a random vector $x\sim \mathcal{N}(0, \Sigma)$, under Assumption \ref{ass:3.5},  for sufficiently large $m, n$ such that 
	\begin{align}
   & n \geq \tilde{\Omega} \big( \frac{\|\Sigma\|^3_2  p^3 \|w^*\|_2^2 \log\frac{1}{\delta}\log \frac{1}{\xi}}{\epsilon^2}\big)\\
   &  m\geq \Omega \big( \|w^*\|_2^2 p).
\end{align}
Then for any $\zeta\in (0, 1)$, with probability at least $1-\exp(-\Omega(p))-\xi$  we have 

\begin{equation*}
	\|\hat{w}^{nlr}-w^*\|_2\leq \tilde{O}\big(\frac{p^{\frac{3}{2}}\|w^*\|^2_2\log \frac{1}{\xi} \sqrt{\log \frac{1}{\delta}\log \frac{1}{\xi}} 
        }{\epsilon\sqrt{n} }+\|w^*\|^2_2 \sqrt{\frac{p}{m}}\big), 
\end{equation*}
where Big-$\tilde{O}$ and Big-$\tilde{\Omega}$ notations omit the terms of 
$C, \|\Sigma\|_2, c_f, {\tau}, G, L,  M, \bar{c}$ and $\lambda_{\min}(\Sigma)$, and other logarithmic factors (see Appendix for the explicit form of $m$ and $n$). 
\end{theorem}
\begin{remark}\label{remark:20}
We can see the sample complexity of public and private data to achieve an $\ell_2$-norm estimation error of $\alpha$ is ${O}(p\alpha^{-2})$ and $\tilde{O}(p^{3}\alpha^{-2}\epsilon^{-2})$ respectively (if we omit other terms). Compared with Theorem \ref{thm:0}, they
are asymptotically the same as the bounds in GLMs case. However, it is notable that non-linear regression models are quite different with GLMs as their conditional density functions of the response $y$ cannot be written as in exponential forms. The main reason of this similarity is the similar conclusions in Theorem 
\ref{thm:nlrNew1} and Lemma \ref{lemma:1}. And this is due to that both of GLMs and non-linear regressions satisfy the property of $   \mathbb{E}[xy] = \mathbb{E}[xg(\langle x, w^*\rangle) ]$ for some function $g$ (where $g(\cdot)=f(\cdot)$ in non-linear regressions and $g(\cdot)=\Phi'(\cdot)$ in GLMs). Thus, Theorem \ref{thm:nlrNew1} could be considered as the non-linear regression version of the Stein's lemma, which may could be used in other machine learning and statistics problems. 
\end{remark} 
\subsection{Sub-Gaussian Case}
We then consider estimating non-linear regressions in the case where $x$ is 
sub-Gaussian. We will first use the zero-bias transformation \citep{goldstein1997stein} and the techniques in \citep{erdogdu2019scalable} to get a lemma  which is similar to Lemma \ref{lemma:2} and could be though as a generalization of Theorem \ref{thm:nlrnew2} to the sub-Gaussian covariates.

\begin{definition}[Zero-bias Transformation]\label{def:13}
Let $z$ be a random variable with mean 0 and variance $\sigma^2$. Then, there exists a random variable $z^*$ that satisfies $\mathbb{E}[zf(z)]= \sigma^2\mathbb{E}[f'(z^*)]$ for all differentiable functions $f$. The distribution of $z^*$ is called the $z$-zero-bias distribution.  
\end{definition}
Note that when $z$ is Gaussian, then $z^*=z$, this is just the Stein's lemma. 

\begin{theorem}\label{thm:5}
Let $x_1, \cdots, x_n\in \mathbb{R}^p$ be i.i.d realizations of a random vector $x$ that is zero-mean sub-Gaussian with covariance matrix $\Sigma$ and satisfies Assumption \ref{ass:2}. Let $v=\Sigma^{-\frac{1}{2}}x$ be the whitened random vector of $x$ and denote  $\|v\|_{\psi_2}=\kappa_x$. If each  $v_i$ has constant first and second conditional moments and  function $f'$ is Lipschitz continuous with constant G, then  for $c_f=\frac{1}{\mathbb{E}[f'(\langle x_i, w^*\rangle)] }$, the following holds, where $w^{ols}$ is the OLS vector.  
{
   $$\|\frac{1}{c_f}\cdot w^*-w^{ols}\|_{\infty}\leq O(Gr\kappa_x^3\sqrt{\rho_2}\rho_{\infty}\frac{\|w^*\|^2_\infty}{\sqrt{p}}). $$
  } 
\end{theorem}
\begin{algorithm}[ht]
\caption{Non-interactive LDP for smooth non-linear regression with public data (General)\label{alg:2}}
	\begin{algorithmic}[1]
		\State {\bfseries Input:} Private data $\{(x_i, y_i)\}_{i=1}^n\subset \mathbb{R}^p\times \mathbb{R}$ with $\|x_i\|_1\leq r$,  public unlabeled data $\{x_j\}_{j=n+1}^{n+m}$, link function $f:\mathbb{R}\mapsto \mathbb{R}$, privacy parameters $\epsilon, \delta$, and initial value $c\in \mathbb{R}$. 
    \For{Each user $i\in [n]$}
   \State  Release $\widehat{x_ix_i^T}= x_ix_i^T + E_{1,i}$, where $E_{1,i} \in \mathbb{R}^{p\times p} $ is a symmetric matrix and each entry of  the upper triangle matrix is sampled from $\mathcal{N}(0, \frac{32r^4\log\frac{2.5}{\delta}}{\epsilon^2})$. 
    \State Release $\widehat{x_iy_i}=x_iy_i+ E_{2, i}$, where the vector $E_{2,i}\in \mathbb{R}^{p}$ is sampled from $\mathcal{N}(0, \frac{32r^2(Lr+|f(0)|+C)^2\log \frac{2.5}{\delta}}{\epsilon^2}I_p)$. 
     \EndFor
     \For {The server}
    \State   Denote $\widehat{X^TX}=\sum_{i=1}^n\widehat{x_ix_i^T}$ and $\widehat{X^Ty}=\sum_{i=1}^n\widehat{x_iy_i}$. Calculate $\hat{w}^{ols}=(\widehat{X^TX})^{-1}\widehat{X^Ty}$. 
     
  \State  Calculate $\tilde{y}_{j}= x_j^T\hat{w}^{ols}$ for each $j=n+1, \cdots, n+m$.
  
   \State  Find the root $\hat{c}_{f}$ such that $1= \frac{\hat{c}_{f}}{m}\sum_{j=n+1}^{n+m}f'(\hat{c}_{f}\tilde{y}_j)$ using Newton's root finding method:
     \For {$t=1, 2, \cdots$ until convergence}
    \State  $c= c- \frac{c\frac{1}{m}\sum_{j=n+1}^{n+m}f'(c\tilde{y}_j)-1}{
     \frac{1}{m}\sum_{j=n+1}^{n+m}\{f'(c\tilde{y}_j)+c\tilde{y}_jf^{(2)}(c\tilde{y}_j)\}}$
    \EndFor
    \EndFor \\
    \Return {$\hat{w}^{nlr}= \hat{c}_{f}\cdot \hat{w}^{ols}$.}
	\end{algorithmic}
\end{algorithm}
From 
Theorem \ref{thm:5}, we can see that it shares the same phenomenon as in Lemma  \ref{lemma:2} ({\em i.e.,} the OLS vector with some constant could approximate $w^*$ well).  Thus, a similar idea to Algorithm \ref{alg:1} can be used to solve this problem, which gives us Algorithm \ref{alg:2} and the following theorem. Similar to the previous section, we need the following assumptions for function $f(\cdot)$. 
\begin{assumption}\label{ass:4}
		 We assume the following conditions hold: 
	\begin{itemize}
		\item  $|f'(\cdot)|\leq L$ and $f^{(2)}(\cdot)$ is G-Lipschitz. 
		\item  For some constant $\bar{c}$ and $\tau>0$, the function $\ell(c)=c\mathbb{E}[f'(\langle x, w^{ols}\rangle c)]$ satisfies the condition of $\ell(\bar{c})\geq 1+\tau$,  where $w^{ols}$ is in Theorem  \ref{thm:5}. 
		\item  The derivative of $\ell$ in the interval $[0, \max\{\bar{c}, c_f\}]$ does not change the sign ({\em i.e.,} its absolute value is lower bounded by some constant $M>0$), where $c_f$ is in Theorem \ref{thm:5}. 
	\end{itemize}
\end{assumption}
\begin{theorem}\label{thm:7}
For any $0<\epsilon, \delta<1$, Algorithm \ref{alg:2} is $(\epsilon, \delta)$ non-interactive LDP. Under the assumptions of Theorem \ref{thm:5}, and if the link function $f$ satisfies Assumption \ref{ass:4}, then for sufficiently large $m, n$ such that 
{
\begin{align}
&m\geq \Omega \big(\|w^*\|^2_\infty\max\{1, \|w^*\|^2_\infty\} p^2 \big), \\
& n\geq \tilde{\Omega}\big( \frac{ p^2 \|w^*\|_\infty^2\max\{1, \|w^*\|_\infty^2\} \log\frac{1}{\delta}\log \frac{1}{\xi}}{\epsilon^2} 
\big). 
\end{align}
}
Then for any $\zeta\in (0, 1)$, with probability at least $1-\exp(-\Omega(p))-\xi$, the output of Algorithm \ref{alg:2} satisfies   
{
\begin{multline}
   \|\hat{w}^{nlr}-w^*\|_\infty
   \leq 
    \tilde{O}\big( \frac{ \|w^*\|^2_\infty\max\{1, \|w^*\|^2_\infty\} p
        }{\sqrt{m} }+
         \\
        \frac{\|w^*\|^2_\infty\max\{1, \|w^*\|^2_\infty\} p\sqrt{\log \frac{1}{\delta}\log \frac{1}{\xi}} 
        }{\epsilon \sqrt{n} }  + \frac{\|w^*\|^3_\infty\max\{1, \|w^*\|_\infty\}}{\sqrt{p}}  \big), 
\end{multline}
}
where Big-$\tilde{O}$ and Big-$\tilde{\Omega}$ notations omit the terms of $\lambda_{\min}(\Sigma)$, $\|\Sigma\|_2$, $\rho_2$, $\rho_\infty$, $G$, $L$, $\tau$, $M$, $\bar{c}$, $r, \kappa_x, C$ and ${c_f}$, and other logarithmic factors (see Appendix for the explicit form of $m$ and $n$).
\end{theorem}
\begin{remark}
Similar to Theorem \ref{thm:4}, we can see the  the sample complexity of public and private data for Algorithm \ref{alg:2} to achieve an $\ell_\infty$-norm estimation error of $\alpha$ is ${O}(p^2\alpha^{-2})$ and $\tilde{O}(p^2\alpha^{-2}\epsilon^{-2})$  respectively, if $\alpha$ is not too small 
({\em i.e.,} $\alpha\geq \Omega(\frac{1}{\sqrt{p}})$). Such similarity is due to the similar conclusions in Theorem \ref{thm:5} and Lemma \ref{lemma:2}. In more details, consider the simplest case where  the covariate vector $x$ has i.i.d. entries
with mean 0, and variance 1. Then by using the zero-bias transformation  to the $j$-th coordinate of $\mathbb{E}[yx]$ we have 
\begin{align}\label{eq:166}
 \mathbb{E}[yx_j]=\mathbb{E}[x_jf(\langle x, \hat{w}^*\rangle )]= {w}_j^* \mathbb{E}[f'((x_j^*-x_j) {w}_j^*+ x_j{w}_j^*+\sum_{i\neq j}x_i{w}_i^* ) ]. \end{align}
 If $w^*$ is well spread, it
turns out that taken together, with $j = 1, \cdots,  p$ the right-hand side in (\ref{eq:166}) behaves similar to the Gaussian case, where
the proportionality relationship given in Theorem \ref{thm:nlrNew1}  holds. As we mentioned in Remark \ref{remark:20}, Theorem \ref{thm:nlrNew1} is similar to Lemma \ref{lemma:1}. Thus,  Theorem \ref{thm:5} behaves similar to Lemma \ref{lemma:2}. 
\end{remark}

In the following we will provide an instance of $f(\cdot)$ and $x$ that satisfies the assumptions in Theorem \ref{thm:nlrnew2}. 
\begin{theorem}[Sigmoid Link Function]\label{theorem:13}
		Consider the model (\ref{eq:9}) where the link function $f(z)= \frac{1}{1+e^{-z}}$, $x\sim \mathcal{N}(0, \frac{1}{p}I_p)$, $\|w^*\|_2=\frac{\sqrt{p}}{4}$ and $\|w^{ols}\|_2=\frac{\sqrt{p}}{20}$. Then  when $\bar{c}=6$ and $\tau=0.22$, the function $\ell(c)=c\mathbb{E}[f'(\langle x, w^{ols}\rangle c)] >1+\tau$. Moreover, $\ell'(z)$ is bounded by  constant $M=0.1$ on $[0, \bar{c}]$ from below and $c_f\leq \bar{c}$.
\end{theorem}

\section{Experiments}
In this section, we will evaluate the performance of our methods on both synthetic and real-world datasets. The experiments  demonstrate the previous utility results of our algorithms and suggest that they are  efficient. Moreover, we will show that our algorithms only need small number of public unlabeled data to achieve outstanding performance. 
\subsection{Experimental Settings}
\paragraph{Link functions:} In this paper, we mainly study estimating GLMs and non-linear regressions, and we will use  variants of loss functions and link functions for each algorithm. Specifically, 
\begin{itemize}
    \item For Algorithm \ref{alg:0} we consider the binary logistic regression where we have $\Phi\Paren{\langle x, w\rangle} = \ln \Paren{1+\exp\Paren{\langle x, w\rangle}}$ in (\ref{eq:1}). 
    \item For Algorithm \ref{alg:1} we consider the binary logistic regression and the exponential regression where $\Phi\Paren{\langle x, w\rangle} = e^{\langle x, w\rangle}$ in (\ref{eq:1}). 
    \item For Algorithm \ref{alg:1.5} we consider the sigmoid link function, {i.e.,} $f(\langle x, w \rangle)=\frac{1}{1+e^{\langle x, w\rangle}}$  in (\ref{eq:9}). 
    \item For Algorithm \ref{alg:2} we consider the cubic link function where $f(\langle x, w \rangle)=\frac{1}{3}\langle x, w \rangle^3$ and the logistic function  where $f(\langle x, w \rangle)=\log (1+e^{-\langle x, w\rangle})$ in (\ref{eq:9}). 
\end{itemize}

\paragraph{Synthetic data generation:} In this paper, we assume the distribution of data feature vector is either Gaussian or sub-Gaussian with bounded $\ell_1$-norm. Specifically, for each case we consider the following procedure for feature vectors generation.  
\begin{itemize}
    \item For Gaussian distribution, we consider two cases for the covariance matrix: (1) the  covariance matrix is diagonal, {i.e.,} $\Sigma=\sigma^2 I_p$ where $\sigma$ is sampled from the uniform distribution in $[0, 1]$; (2) the  covariance matrix is non-diagonal, and we assume $\Sigma=U+I_p$, where $U\in \mathbb{R}^{p\times p}$ is a random orthogonal matrix. 
    \item In the sub-Gaussian case, each entry of each feature vector is generated independently from a Bernoulli distribution $\probof{x_{i,j}=\pm\frac{1}{p}}=0.5$. 
\end{itemize}
After we get feature vectors, we generate the underlying parameter $w^*$, which is a random unit vector. After that, we generate the responses $\{y_i\}_{i=1}^n$ as follows. 
\begin{itemize}
    \item For GLMs, each response is generated according to its definition in (\ref{eq:1}). Specifically, for the logistic regression we generate $y_i=\frac{e^{\langle w, x_i\rangle }}{1+e^{\langle w, x_i\rangle }}$. For exponential regression we generate $y_i=e^{\langle x_i, w^* \rangle}$. It it notable that here we only consider the exponential loss to sub-Gaussian case where $\langle x, w^* \rangle\leq 1$. Thus, in both cases $\{y_i\}_{i=1}^n$ are bounded. 
    \item For non-linear regressions, each $y_i$ is generated according to the model (\ref{eq:9}) where $\sigma$ is bounded by $C=0.05$. It is notable that for all link functions we considered,  $\{y_i\}_{i=1}^n$ are bounded. 
\end{itemize}

\paragraph{Experimental settings for synthetic data:} 
Motivated by the results in previous sections, for data with Gaussian features vectors,  we will use the (squared) relative  $\ell_2$-norm  error $\frac{\|\hat{w}-w^*\|^2_2}{\|w^*\|^2_2}$ to measure performance, otherwise we will use (squared) relative  $\ell_\infty$-norm  error $\frac{\|\hat{w}-w^*\|^2_\infty}{\|w^*\|^2_\infty}$. For privacy parameters, we will choose $\epsilon$ between 4 to 15 and set $\delta=\frac{1}{n^{1.1}}$.\footnote{Note that in the studies on LDP ERM, $\epsilon$ is always chosen as a large value such as \citep{bhowmick2018protection}. Moreover, we can use the shuffling technique in \citep{erlingsson2019amplification} for privacy amplification.} For the dimension $p$
we choose from the set $\{5, 10, 15, 20, 25, 30, 40, 50, 60\}$. 
For different experiments we will vary different private sample size $n$. However we will always set the size of public unlabeled data $m$ be much smaller than $n$. Specifically, we will always set {\bf $m=\lfloor \frac{n}{p^2} \rfloor$}. For each experiments above, we run 20 times and take the average of the errors. 

\paragraph{Experimental settings for real-world data:} 
We conduct experiments on binary logistic regression for GLMs on the Covertype dataset~\citep{Dua:2019}, the SUSY dataset~\citep{baldi2014searching} and 
the Skin Segmentation dataset~\citep{Dua:2019}. 

For the Covertype dataset, before running our algorithm, we first normalize the data and remove some co-related features. After the pre-processing, the dataset contains 581,012 samples and 44 features. There are seven possible values for the label. Here we consider a weaker test, which is to classify whether the label is Lodgepole Pine (type 2) or not. We divide the data into a training data and a test data, where $n_{\text{training}} = 350,000$ and $n_{\text{testing}} = 200,000$  (other data will be used as public unlabeled data) and  we randomly choose the sample size $n \in 10^{4} \cdot \{10, 15, 20, 25, 30, 35\}$ from the training data and set the size of public unlabeled data as $m=10^4$. Regarding the privacy parameter, we take $\delta = \frac1{n^{1.1}}$ and let $\epsilon$ take value 
from $ \{4, 6, 10, 15\}$. We measure the performance by the prediction accuracy. For each experiment, we repeated 20 times. 

For the SUSY dataset, the task is to classify whether the class label is signal or background. After pre-processing and sampling, the dataset contains $500,000$ samples and 18 features. Then we divide the data into a training data and a test data, where $n_{\text{training}} = 450,000$ and $n_{\text{testing}} = 30,000$ (other data will be used as public unlabeled data) and  we randomly choose the sample size $n \in 10^{4} \cdot \{10, 15, \cdots, 45\}$ from the training data and set the size of public unlabeled data as $m=10^4$. Regarding the privacy parameter, we take $\delta = \frac1{n^{1.1}}$ and let $\epsilon$ takes value from $ \{2, 3, 5, 10\}$. We measure the performance by the prediction accuracy. For each experiment, we repeated 20 times. 

For the Skin Segmentation dataset, the task is to classify where the class label is Skin or Nonskin image.  After  pre-processing, the dataset contains $245,057$ samples and $3$ features. We divide the data into a training data and a test data, where $n_{\text{training}} = 180,000$ and $n_{\text{testing}} = 5,000$ (other data will be used as public unlabeled data) and we  randomly choose the sample size $n \in 10^{4} \cdot \{2, 4, \cdots, 18\}$ from the training data and set the size of public unlabeled data as $m=5,000$. For the privacy parameter, we take $\delta = \frac1{n^{1.1}}$. As the dimension of feature vector is only $3$, here we consider the high privacy regime and let $\epsilon$ take value from $ \{0.2, 0.3, 0.5, 0.7\}$. We measure the performance by the prediction accuracy. For each experiment, we repeated 20 times. 

\paragraph{Baseline and other methods:} Note that for synthetic data, there is no need to conduct baseline methods as we know the underlying parameter $w^*$ and we use the relative error to measure the utility. For real-world data, as we mentioned previously there is no previous work which provides efficient methods. Thus, here we will only compare our methods with the non-private method, which is the Logistic Regression classifier in the scikit-learn library \citep{scikit-learn}. 

Besides the non-private method, as we mentioned in Remark \ref{remark:9}, we can adopt our idea to design a 2-round LDP algorithm without using public unlabeled data. That is, in the first round  we get $\hat{w}^{ols}$ by using half privacy budget and the server sends it to all the users. In the second round, each user uses  another half  privacy budget to compute  $\tilde{y}_j= x_j^T \hat{w}_{ols}$, then performs the clipping step to $\tilde{y}_j$ to project $\tilde{y}_j$ on the range of $y$ and adds Gaussian noise to the clipped value. Finally, each user sends the noisy version of $\tilde{y}_j$ to the server. Then the server uses these perturbed $\tilde{y}_j$ to estimate the constant of ${c}_{\Phi}$. We provide the details of the algorithm for the Gaussian covariates case in GLMs as an example in Appendix \ref{sec:2-round}, the other algorithms are similar. We call such algorithms as 2-round algorithms. We will  compare our methods with these 2-round algorithms.

\subsection{Experimental Results}\label{sec:exp_synthetic}
We consider the following questions through experiments: (1) When the dimension $p$ and the privacy budget $\epsilon$ are fixed, for synthetic data, what is the trend of the relative $\ell_2$-norm or $\ell_\infty$-norm error with different private data size $n$?  (2) For real-world data, what is the trend of accuracy when $n$ or $\epsilon$ increases? What is the difference between the accuracy of our private estimator and the accuracy of the non-private method? 
(3) When the private data size $n$ and private parameter $\epsilon$ are fixed. How will the dimension $p$ affect the utility? (4) How will the number of public unlabeled data size $m$ affect the (relative) error and the accuracy? (5) While those 2-round algorithms are heuristic, do they have good performance?  Moreover, compared the 2-round algorithms, do our methods have better performance? 

We conduct experiments for each of our methods to answer the above questions, see Figure \ref{fig:alg1_diag_epsilon_whole}-\ref{fig:alg5_compare} for details. Specifically, in Figure \ref{fig:alg1_diag_epsilon_whole}-\ref{fig:alg1_org_public_whole} we consider the performance of Algorithm \ref{alg:0} for Gaussian data whose covariance matrix is either diagonal or non-diagonal. In Figure \ref{fig:alg2_exp_epsilon_whole}-\ref{fig:alg2_log_public_whole} we consider the performance of Algorithm \ref{alg:1} for Bernoulli data where the loss function could be either the exponential loss or the logistic loss. The results for Algorithm \ref{alg:1.5} are presented in Figure \ref{fig:alg3_diag_epsilon_whole}-\ref{fig:alg3_diag_public}, where the link function is the sigmoid function and the covariance matrix of Gaussian is diagonal. For Algorithm \ref{alg:2}, its experimental results in the case where the link function is either sigmoid or logistic are shown in Figure \ref{fig:alg4_cubic_epsilon_whole}-\ref{fig:alg4_logistic_public_whole}. Besides the synthetic data, in Figure \ref{fig:alg2_log_publicdata_whole} and \ref{fig:alg2_log_public_pulic_whole} we show Algorithm \ref{alg:1} for binary logistic regression on several real-world datasets. Finally, in Figure \ref{fig:alg5_compare} we compare our algorithms with their corresponding  2-round LDP algorithms. 

From (a), (b) and (c) in Figure \ref{fig:alg1_diag_epsilon_whole}, \ref{fig:alg1_org_epsilon_whole}, \ref{fig:alg2_exp_epsilon_whole}, \ref{fig:alg2_log_epsilon_whole}, \ref{fig:alg3_diag_epsilon_whole}, \ref{fig:alg4_cubic_epsilon_whole} and \ref{fig:alg4_logistic_epsilon_whole}, firstly we can see that with different link functions, data distributions and covariance matrices,  when the dimension $p$ is fixed, although there are some exceptions such as when $n=7\times 10^5$ and $\epsilon=4$ in (b) of Figure \ref{fig:alg1_diag_epsilon_whole}, in general the (squared) relative ($\ell_2$-norm or $\ell_\infty$) error will decreases when $n$ becomes larger, which means the private estimator will be sufficiently closed to the underlying parameter. Moreover, when $n$ gets more larger, the error will tends to be unchanged. This is due to that besides the private data size $n$, in theory the error also depends on  the public data size $m$. Secondly, from the above results we also observe that the relative error is proportional to $\frac{1}{\epsilon^2}$, which matches our theoretical results. However, we can also see that when in the low privacy regime, i.e., when $\epsilon$ is large (e.g. $\epsilon=10$) the relative error only decreases slightly and its curve becomes flat  when $n$ becomes larger. From our previous theoretical results we can see this is due to that in this case the error will be dominated by the term related to $m$ instead of $n$ and $\epsilon$. Besides the relative error, in (a), (b) and (c) of Figure \ref{fig:alg2_log_publicdata_whole} we compare the classification accuracy on test data. Here we can get similar conclusions as in the synthetic data case. Furthermore we can see that when the private data size  $n$ and the privacy parameter $\epsilon$ is large enough, the accuracy of our private estimator will be closed to the accuracy of the non-private  method. For example, for Covertype data, the accuracy of the non-private logistic regression is about 75\% where our private estimator could achieve about $72.5\%$ accuracy when $\epsilon=10$ and $n=3\times 10^5$. 

In (a) of Figure \ref{fig:alg1_diag_public_whole}, \ref{fig:alg1_org_public_whole}, \ref{fig:alg2_exp_public_whole}, \ref{fig:alg2_log_public_whole}, \ref{fig:alg3_diag_public}, \ref{fig:alg4_cubic_public_whole} and \ref{fig:alg4_logistic_public_whole} we present the results of relative error w.r.t different $n$ and dimension $p$. From all the figures we can see that the relative error increases as the dimension increases. However, it may seem a little weird that the relative error is not linear in the dimension, which was shown in the previous sections theoretically. We note that as in theory the error depends on lots of terms. Thus, when the dimension $p$ changes, some other parameters, for example, the $l_2$ norm of the covariance matrix and $\|w^*\|_{\infty}$ also change, which bring other effects to the relative error. Moreover, we can see from some results, such as  $p=40$ in Figure \ref{fig:alg4_logistic_public_whole}, even when $n=7\times 10^5$ the error is still unsatisfactory. This is due to that in theory, the number of efficient sample size is $\sqrt{n}$ since the dependency on $n$ is $\frac{1}{\sqrt{n}}$ in the error bound, which means the efficient sample size in this case is only about $836$. However, as we mentioned in the Related Work section (Section \ref{sec:related work}), even in the interactive LDP model,  the dependency on $n$ is also $\frac{1}{\sqrt{n}}$ and this is optimal \citep{duchi2013local}. Thus, large scale of data is essential for LDP model, not only for our algorithms.

Next, we consider the effect of public unlabeled data. 
As we mentioned earlier, in the experiments on synthetic data, we always set $m=\lfloor \frac{n}{p^2}\rfloor$, which means $m$ is far less than $n$. Thus, from (a), (b) and (c) in Figure \ref{fig:alg1_diag_epsilon_whole}, \ref{fig:alg1_org_epsilon_whole}, \ref{fig:alg2_exp_epsilon_whole}, \ref{fig:alg2_log_epsilon_whole}, \ref{fig:alg3_diag_epsilon_whole}, \ref{fig:alg4_cubic_epsilon_whole} and \ref{fig:alg4_logistic_epsilon_whole} we can see even smaller public data size $m$ could already achieve outstanding performance, i.e., there is no need to use as large amount of public data as our theoretical result requires to guarantee  good performance. Thus, we conjecture that theoretically we can further improve the bound on $m$ and we will leave it as future research. Moreover, we evaluate the performance of our algorithms with different $m$, see (b) and (c) of Figure \ref{fig:alg1_diag_public_whole}, \ref{fig:alg1_org_public_whole}, \ref{fig:alg2_exp_public_whole}, \ref{fig:alg2_log_public_whole}, \ref{fig:alg3_diag_public}, \ref{fig:alg4_cubic_public_whole}, \ref{fig:alg4_logistic_public_whole} and \ref{fig:alg2_log_public_pulic_whole} for details. Unlike the conclusions in the previous paragraphs, we can see with different size of public data, the trend of error becomes complicated. Specifically, in the case when $\epsilon$ is large (such as $\epsilon=10$), we can see  we can use even more smaller size than $\lfloor \frac{n}{p^2}\rfloor$ of public data to achieve good performance. For example, in (b) and (c) of Figure \ref{fig:alg1_diag_public_whole} we can see that when $m=200$ the algorithm could achieve the similar performance as in the case when $m=1800$. However, such phenomenon does not always hold when $m$ is sufficiently small. For example, in (c) of Figure 
\ref{fig:alg1_org_public_whole} and \ref{fig:alg2_exp_public_whole} we can see when $m$ increases from $200$ to $400$ the error decreases.   When  $\epsilon$ is small, we can see the trend of the relative  error when $m$ increases becomes more unstable. In some cases, larger $m$ may could  decrease the relative error such as $\epsilon=4$ in (c) of Figure \ref{fig:alg1_org_public_whole} while in some cases larger $m$ may could even increase the error. However, no matter when larger $m$ increases or decreases the error,  we can see that the effect of such  change is limited, unless $m$ is sufficiently small.  In total, our conclusion is when $m$ is sufficiently small, larger $m$ may could improve the performance of our algorithms. However, when $m$ is increasing, its effect on the performance is limited and it may could have slightly negative effect, especially when $\epsilon$ is large. And in practice we do not need large size of $m$ as we showed in theory.

Finally, we compare our algorithms with the above 2-round LDP algorithms in Figure \ref{fig:alg5_compare}. From all of those four figures we can see that in most cases the relative error of the 2-round LDP algorithm is quite large compared with our methods and its curve is quite unstable. In Figure \ref{fig:alg4 vs alg5_cubic_p20} we can see the performance of the 2-round algorithm becomes acceptable under the setting of Algorithm \ref{alg:2} for  cubic  link function with $p=30$ for Bernoulli data. However, our method still significantly outperforms the 
the 2-round algorithm in this case. 


\section{Conclusion and Open Problems}
In this paper, motivated by  the Stein's lemma and its variants, we proposed the first efficient algorithm with polynomial sample complexity for Generalized Linear Models estimation in the Non-interactive Local Differential Privacy model with some public unlabeled data. The main idea of our algorithm is to use the OLS (Ordinary Least Square) estimator to approximate the underlying one. The key observation is that, after multiplying the OLS vector by some constant, we can get a new estimator is sufficiently close to the underlying estimator. Thus, in our approach, we use the private data to estimate the OLS vector and the public unlabeled data to estimate the constant. Moreover, we adopted similar ideas to the problem of estimating non-linear regressions and showed similar theoretical results. Finally, we provided intensive experiments of our methods on both synthetic data and real-world data. Most of results support our theoretical analysis and show the effectiveness of our methods. 

Besides the open problems we mentioned in the previous sections, there are still many other open problems left. First, in this paper we mainly focused on the low dimensional case, where $n\gg p$. How to generalize to the high dimensional sparse case, that is $n\ll p$ and $\|w^*\|_0\leq k$? In this case since the Stein's lemma will not be hold, so we need new techniques. Second, from the experimental results we can see that, even if the loss function and the dataset do not satisfy our assumptions, they will still have good performance. Thus, how to relax these assumptions  and reduce the sample complexity of public unlabeled data in our theoretical results? Finally, for the sub-Gaussian case in both GLMs and non-linear regressions, our estimators are biased and the error is $\Omega(\frac{1}{\sqrt{p}})$, can we get unbiased and consistent estimators? 

\acks{Di Wang and Lijie Hu were support in part by the baseline funding BAS/1/1689-01-01, funding from the CRG grand URF/1/4663-01-01, FCC/1/1976-49-01 from the Computational Bioscience Research Center (CBRC) and funding from the AI Initiative REI/1/4811-10-01 of King Abdullah University of Science and Technology (KAUST).  Jinhui Xu was supported in part by the National Science Foundation (NSF) under Grant No. CCF-1716400 and IIS-1919492. Part of the work was done when Di Wang and Marco Gaboardi were visiting the Simons Institute of the Theory for Computing.}

\begin{figure*}[!ht]
\centering
\subfigure[ $p=10$\label{fig:alg1_logistic_diag_p10}]
{\includegraphics[width=0.32\textwidth, height=0.16\textheight]{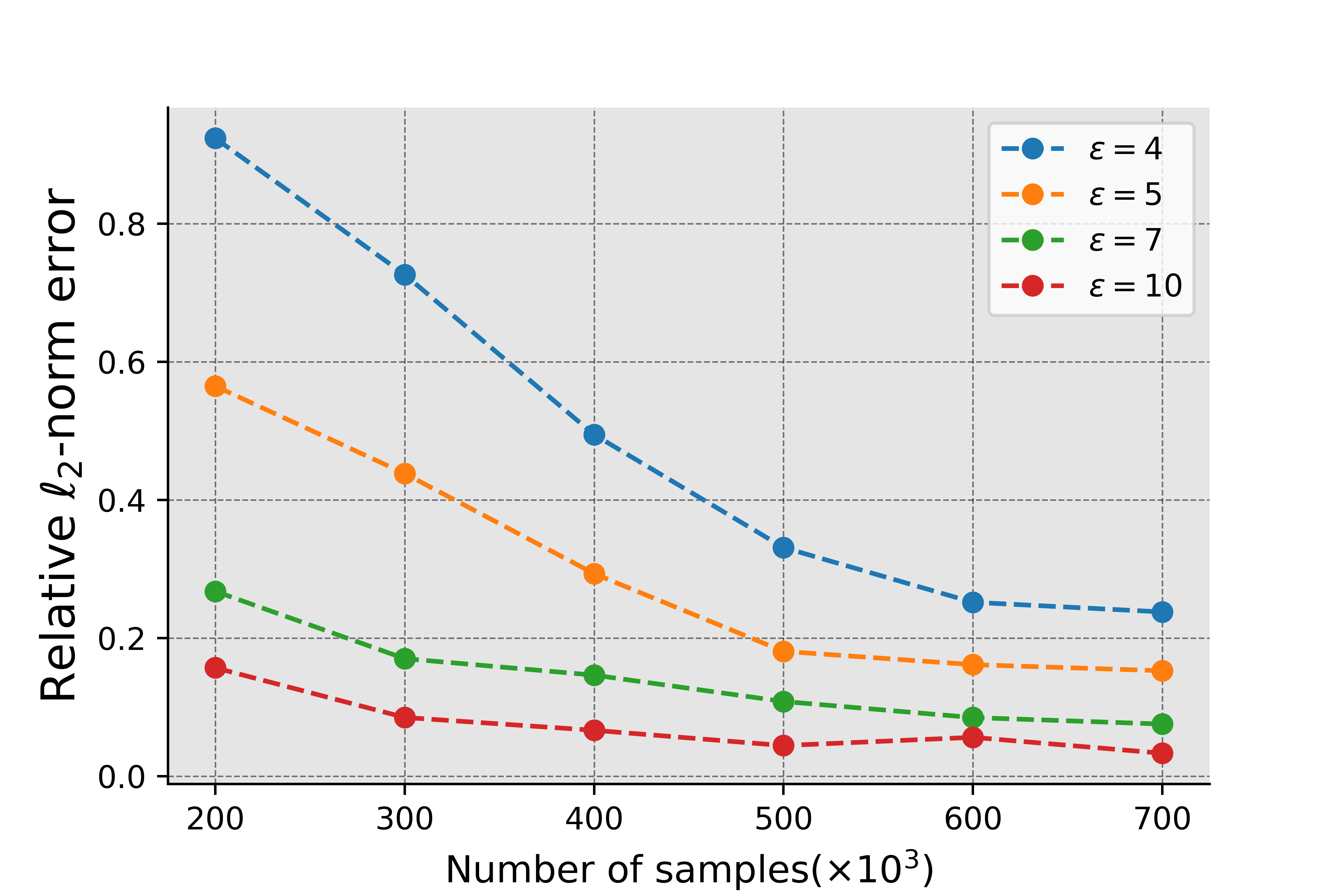}
}
\subfigure[ $p=20$  \label{fig:alg1_logistic_diag_p20}]
{\includegraphics[width=0.32\textwidth, height=0.16\textheight]{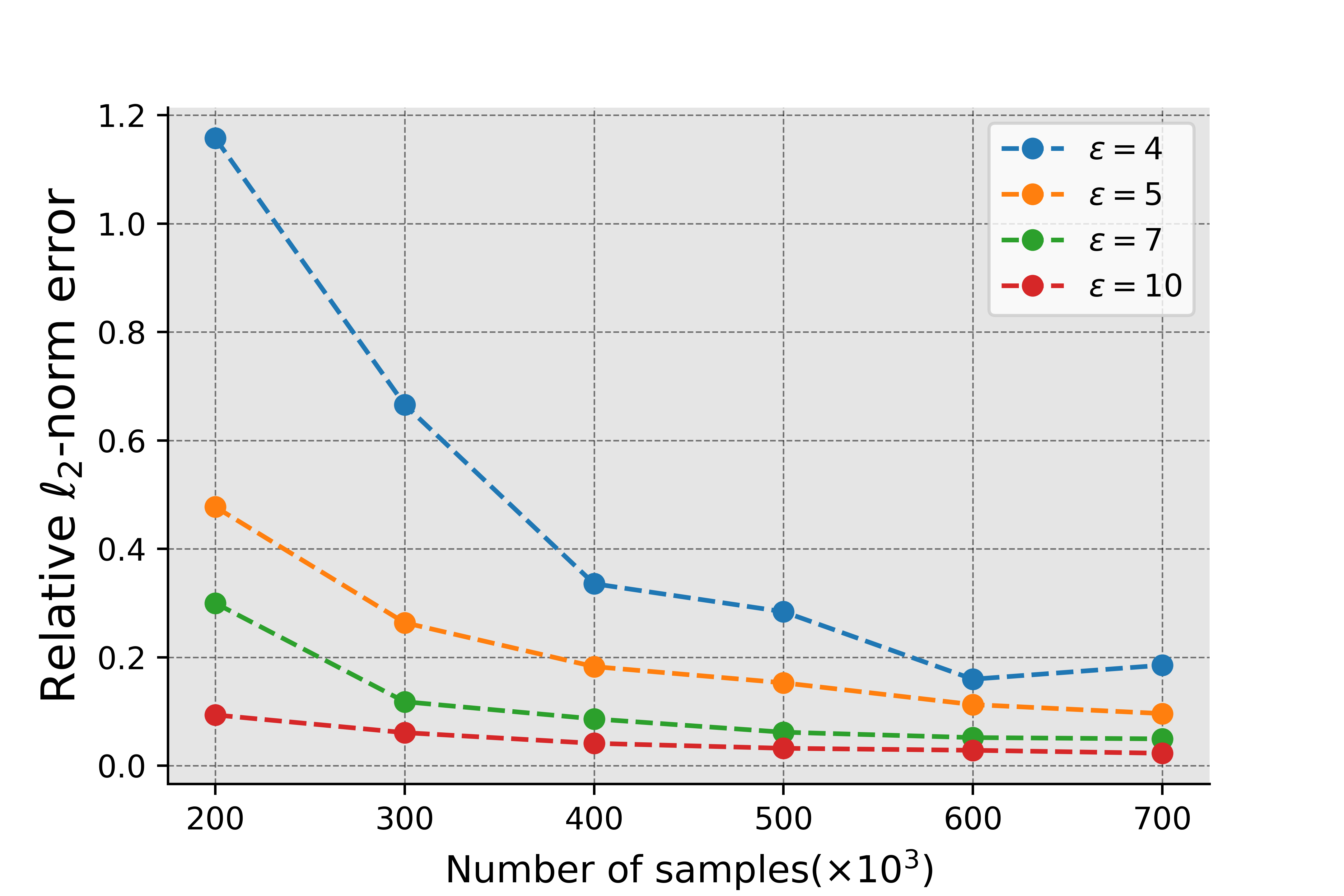}}
\subfigure[ $p=30$  \label{fig:alg1_logistic_diag_p30}]
{\includegraphics[width=0.32\textwidth, height=0.16\textheight]{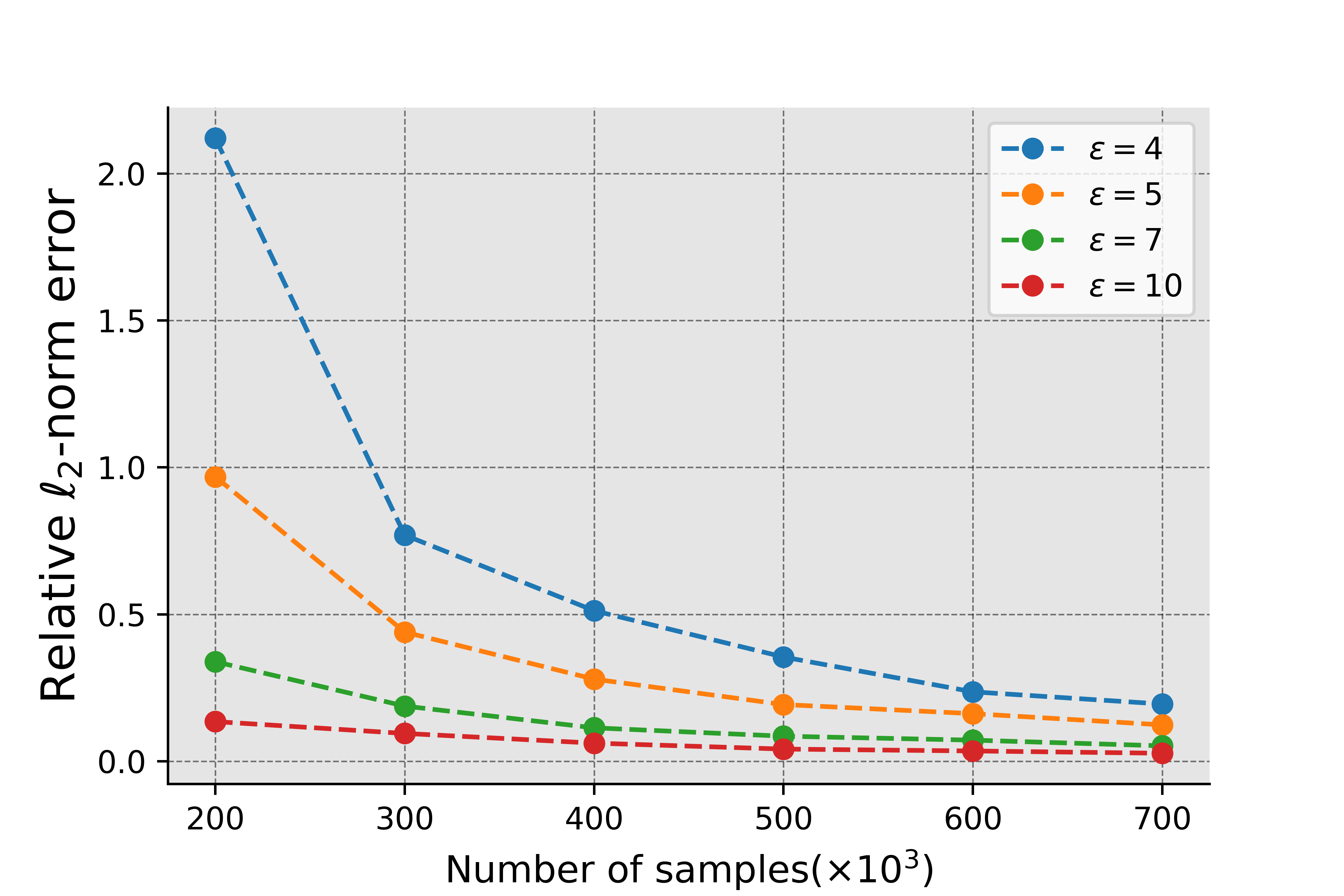}}
    \caption{ Algorithm \ref{alg:0} for logistic regression where the covariance matrix of Gaussian distribution is diagonal under different dimension $p$. \label{fig:alg1_diag_epsilon_whole}
    }
\end{figure*}

\begin{figure*}[!ht]
\centering
\subfigure[ $\epsilon=5$ \label{fig:alg1_logistic_diag_e5}]
{\includegraphics[width=0.32\textwidth, height=0.16\textheight]{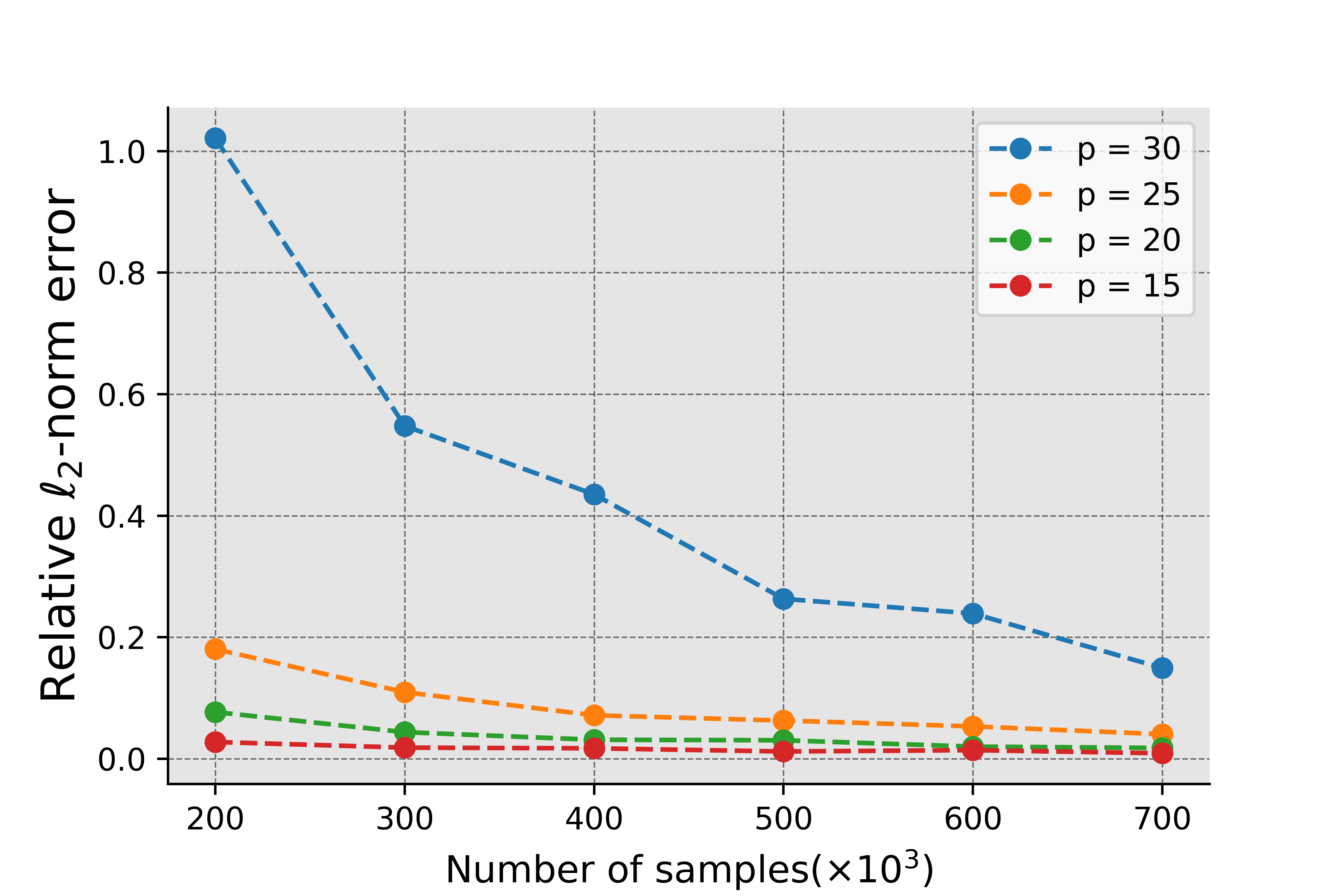}}
\subfigure[  $n=10^5$ and $p=20$ \label{fig:alg1_logistic_diag_n10w_p20}]
{\includegraphics[width=0.32\textwidth, height=0.16\textheight]{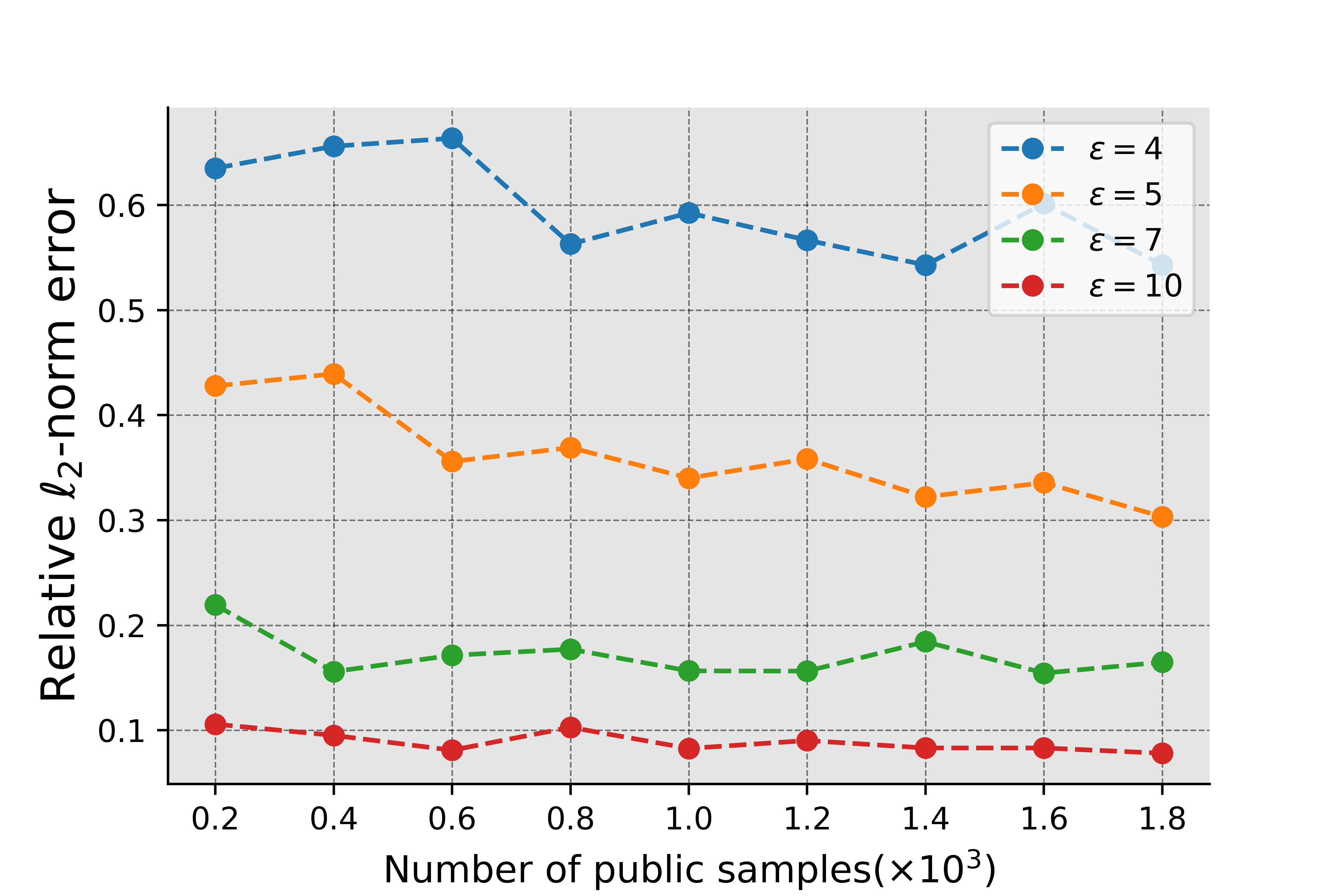}}
\subfigure[ $n=4\times 10^5$ and $p=20$  \label{fig:alg1_logistic_diag_n40w_p20_2}]
{\includegraphics[width=0.32\textwidth, height=0.16\textheight]{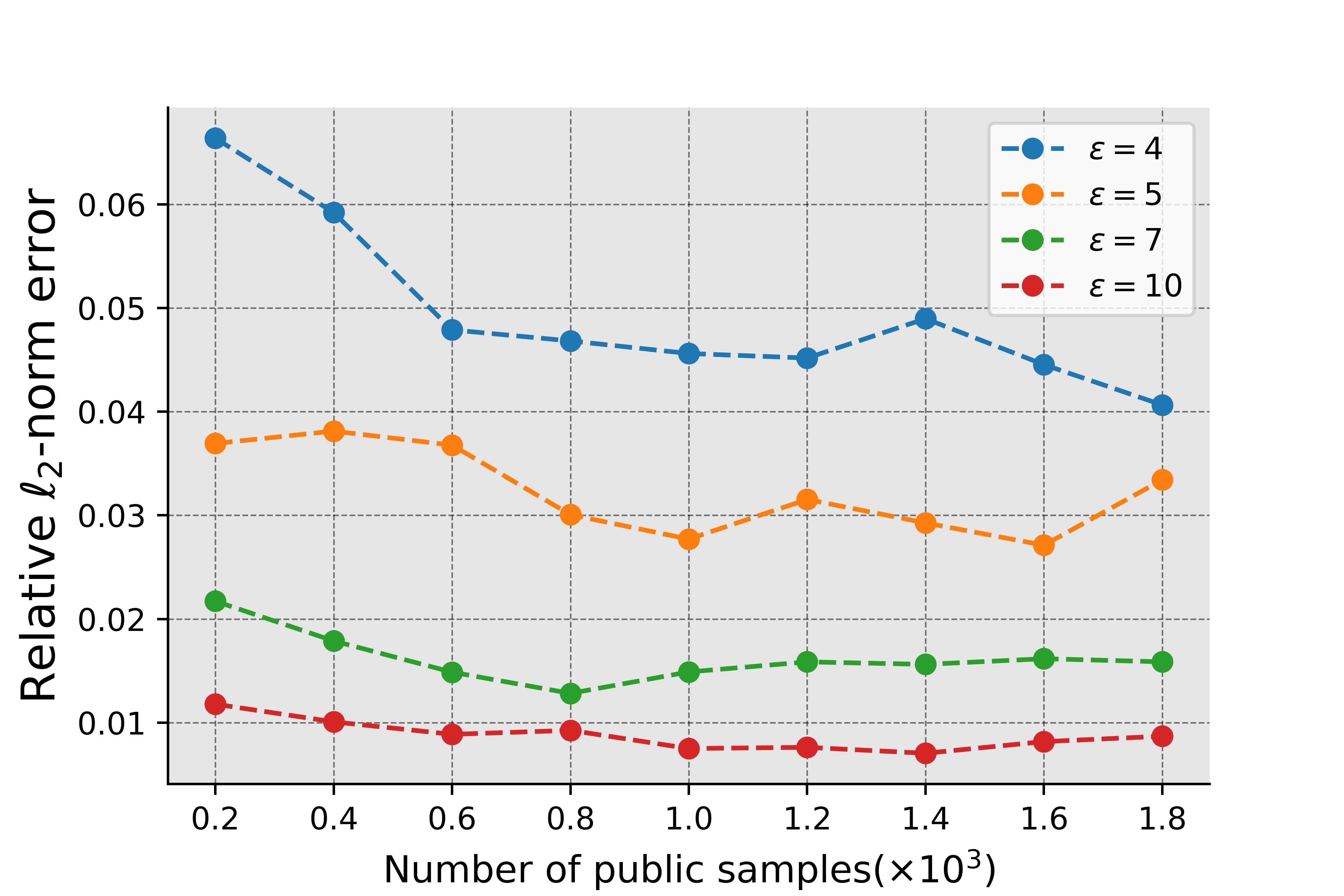}}
    \caption{ Algorithm \ref{alg:0} for logistic regression where the covariance matrix of Gaussian distribution is diagonal. The left plot shows the relative error with different dimension $p$. The middle and the right plots show the relative error with different size of public data $m$  when $n$ and $p$ are fixed.  \label{fig:alg1_diag_public_whole}
    }
\end{figure*}

\begin{figure*}[!ht]
\centering
\subfigure[ $p=10$ \label{fig:alg1_logistic_org_p10}]
{\includegraphics[width=0.32\textwidth, height=0.17\textheight]{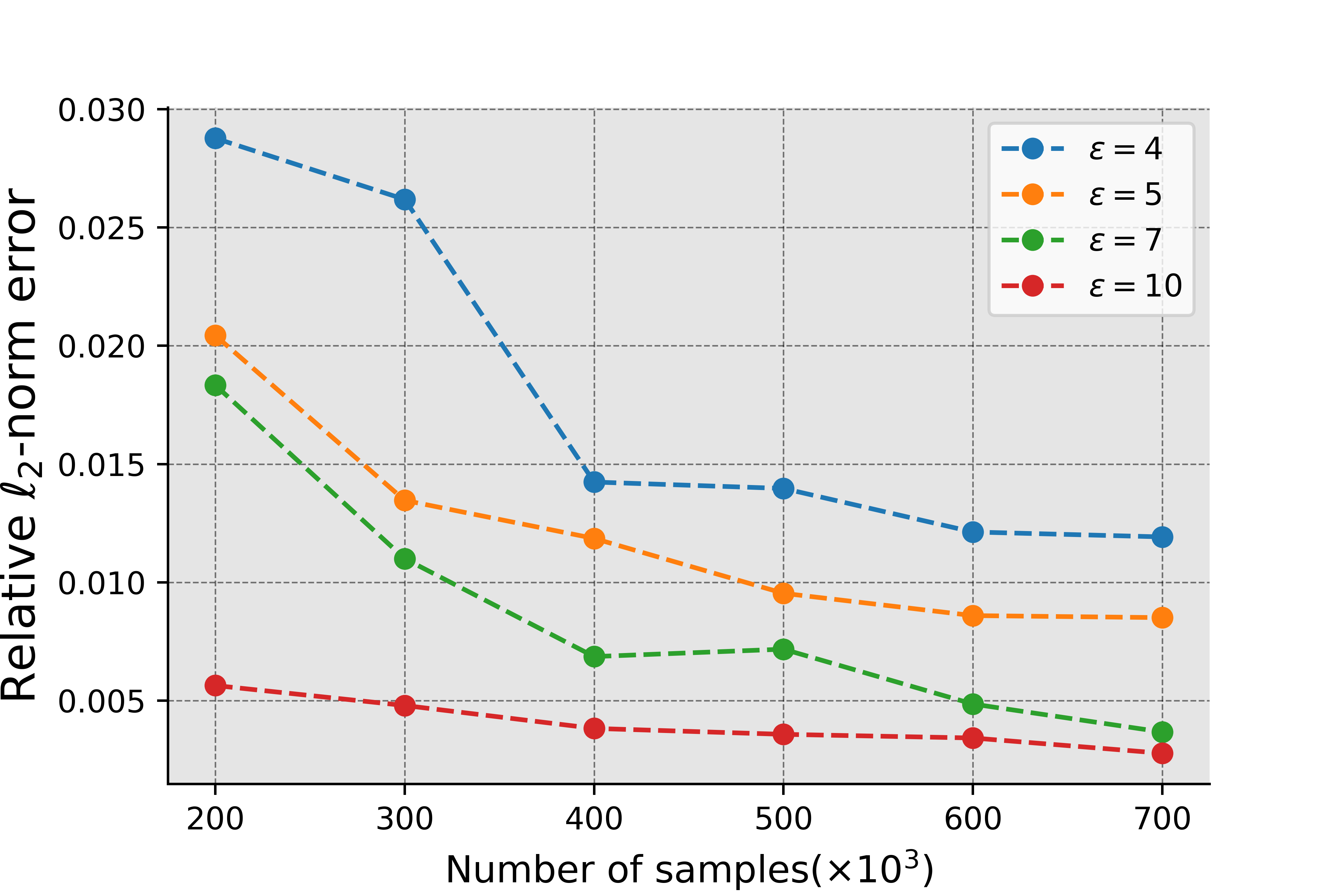}}
\subfigure[ $p=20$  \label{fig:alg1_logistic_org_p20}]
{\includegraphics[width=0.32\textwidth, height=0.17\textheight]{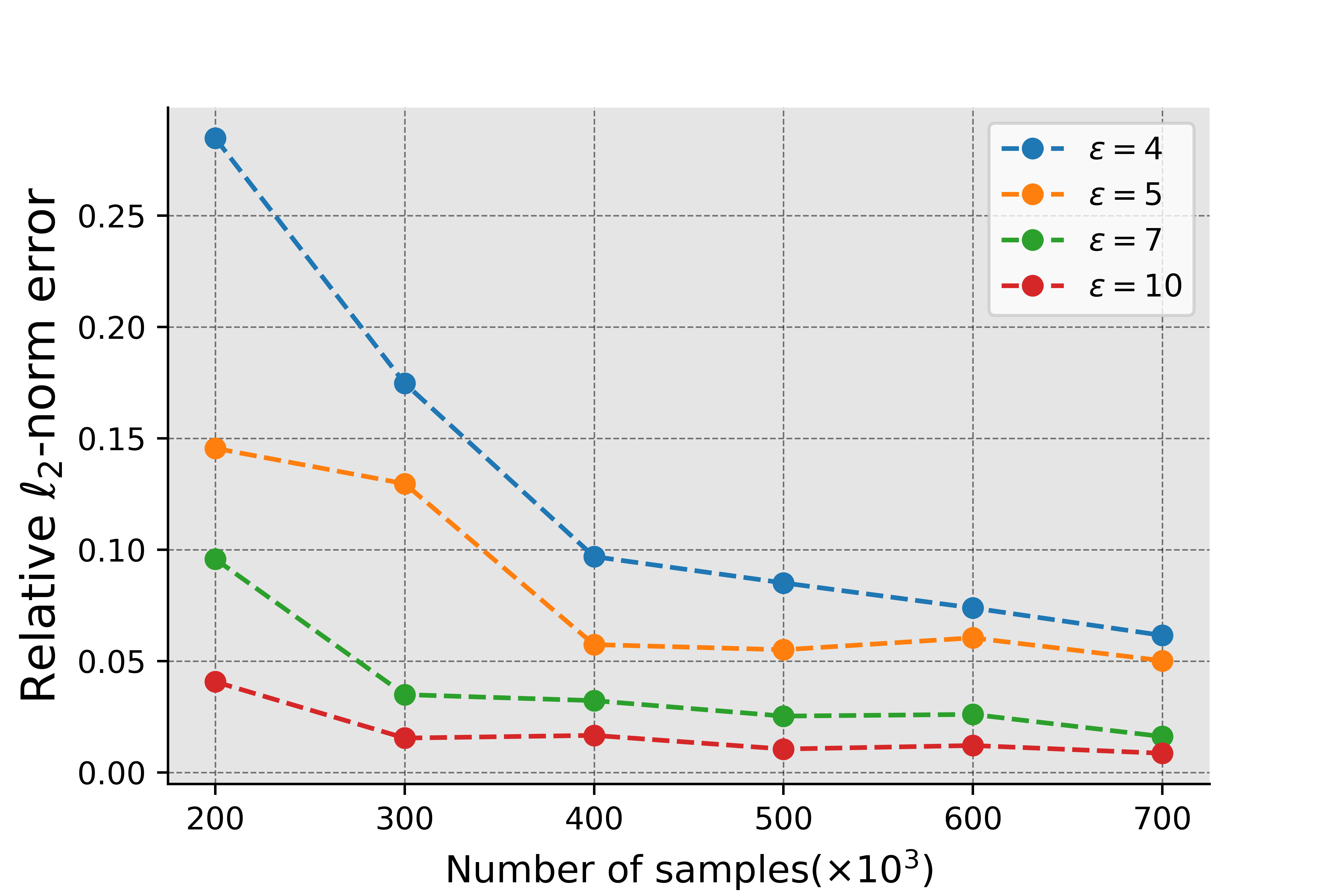}}
\subfigure[ $p=30$  \label{fig:alg1_logistic_org_p30}]
{\includegraphics[width=0.32\textwidth, height=0.17\textheight]{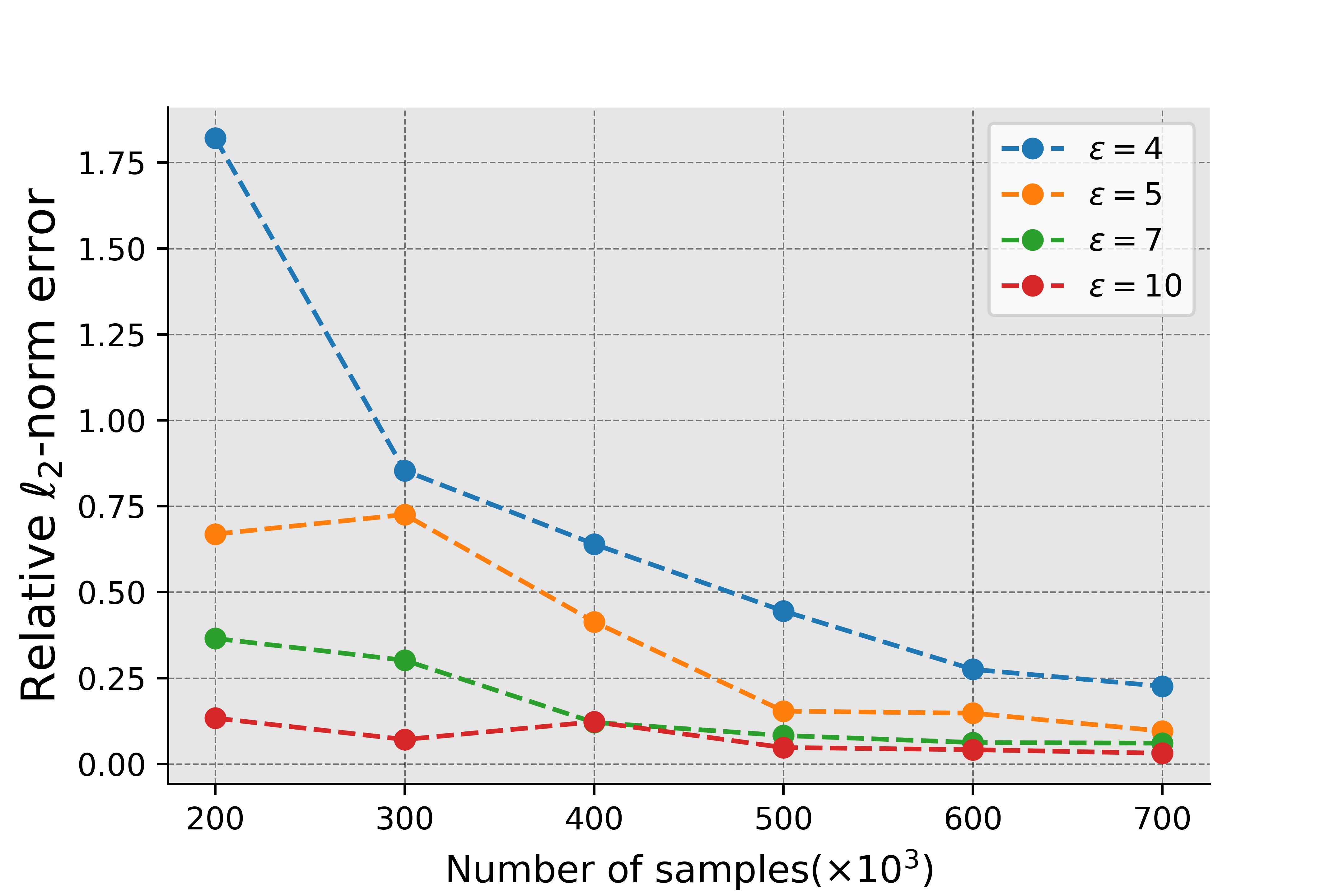}}
    \caption{  Algorithm \ref{alg:0} for logistic regression where the covariance matrix of Gaussian distribution is non-diagonal under different dimension $p$.  \label{fig:alg1_org_epsilon_whole}
    }
\end{figure*}

\begin{figure*}[!ht]
\centering
\subfigure[$\epsilon=5$ \label{fig:alg1_logistic_org_e5}]
{\includegraphics[width=0.32\textwidth, height=0.16\textheight]{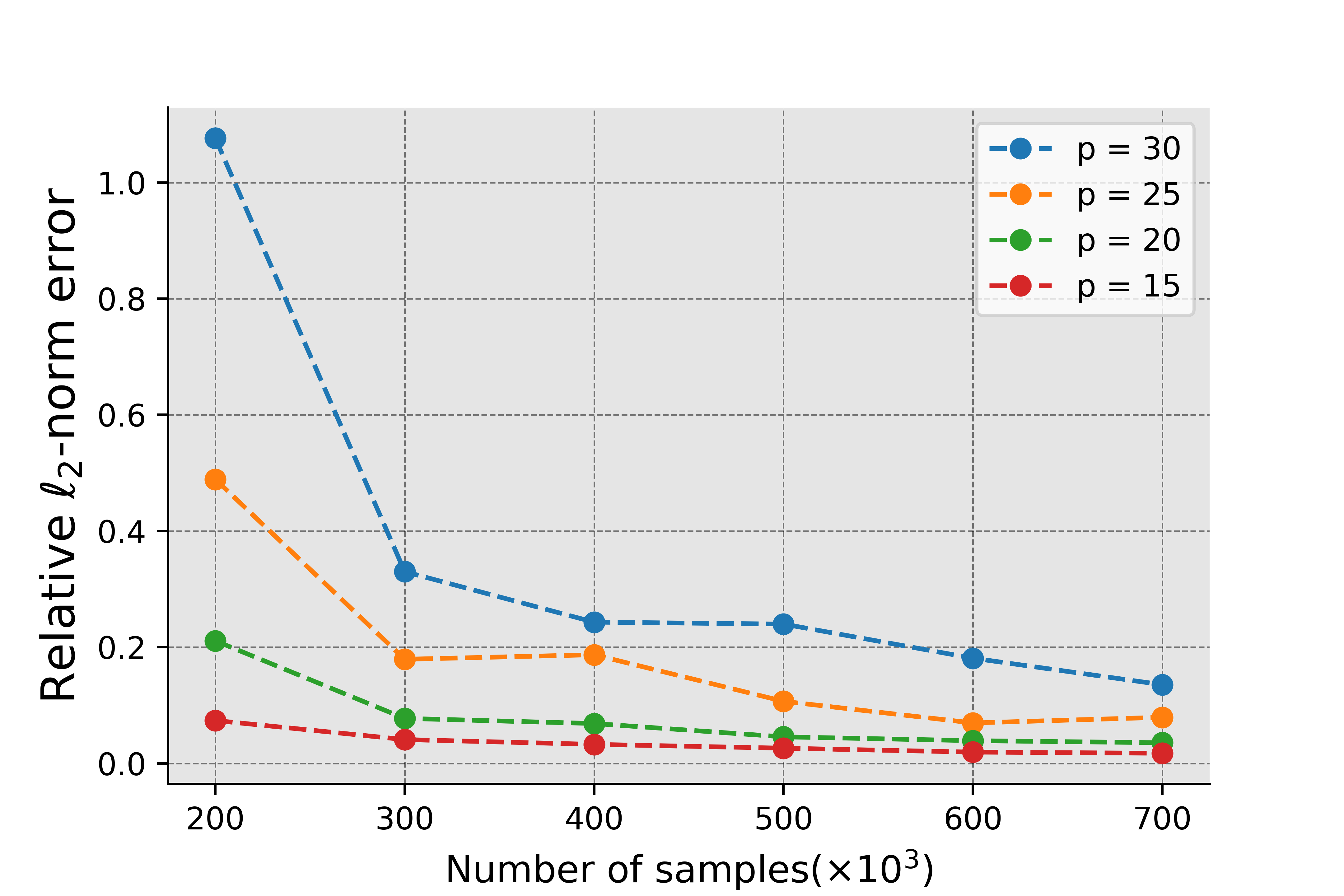}}
\subfigure[ $n=5\times 10^4$ and $p=20$  \label{fig:alg1_logistic_diag_n5w_p20}]
{\includegraphics[width=0.32\textwidth, height=0.16\textheight]{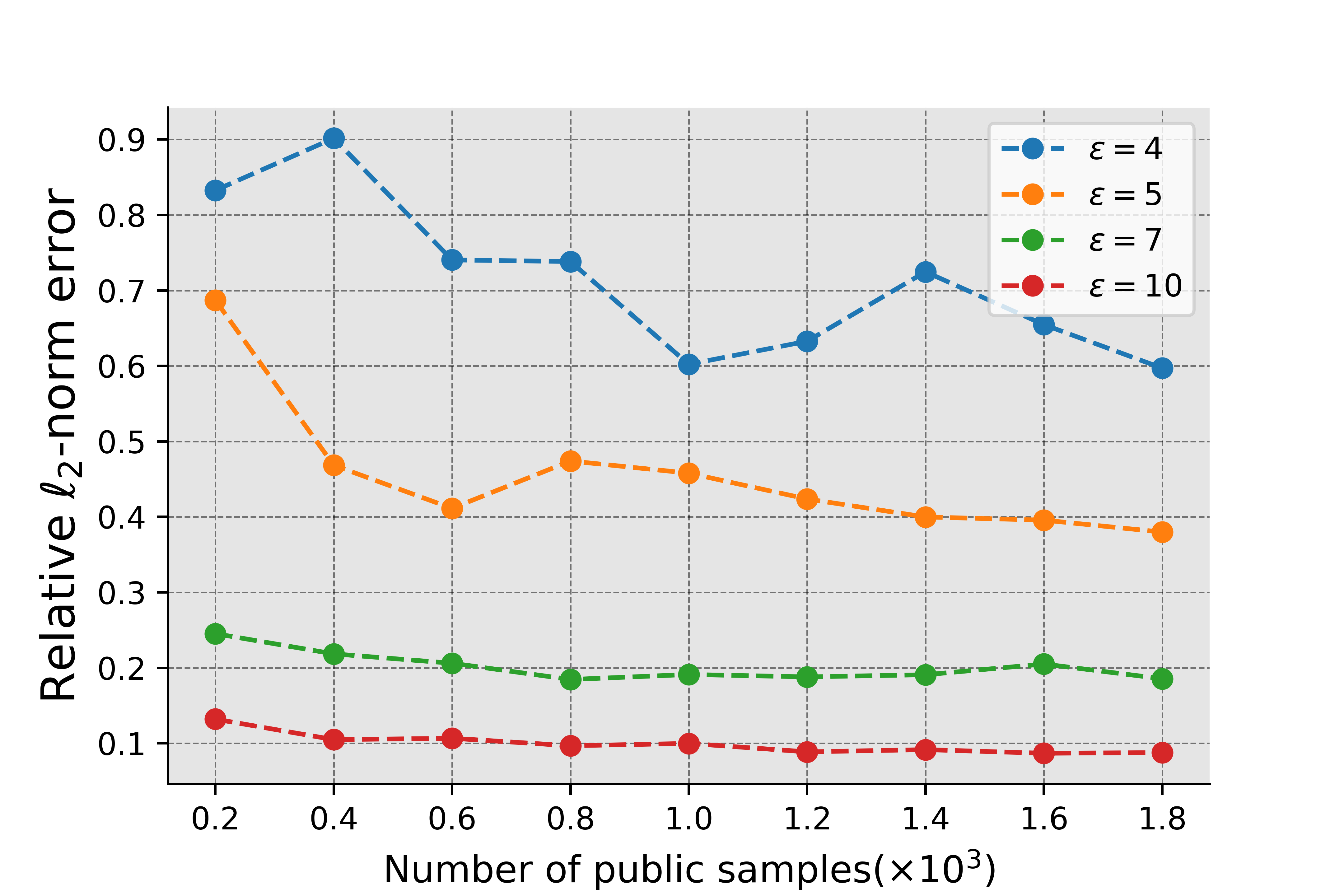}}
\subfigure[  $n=4\times 10^5$ and $p=20$ \label{fig:alg1_logistic_org_n40w_p20_2}]
{\includegraphics[width=0.32\textwidth, height=0.16\textheight]{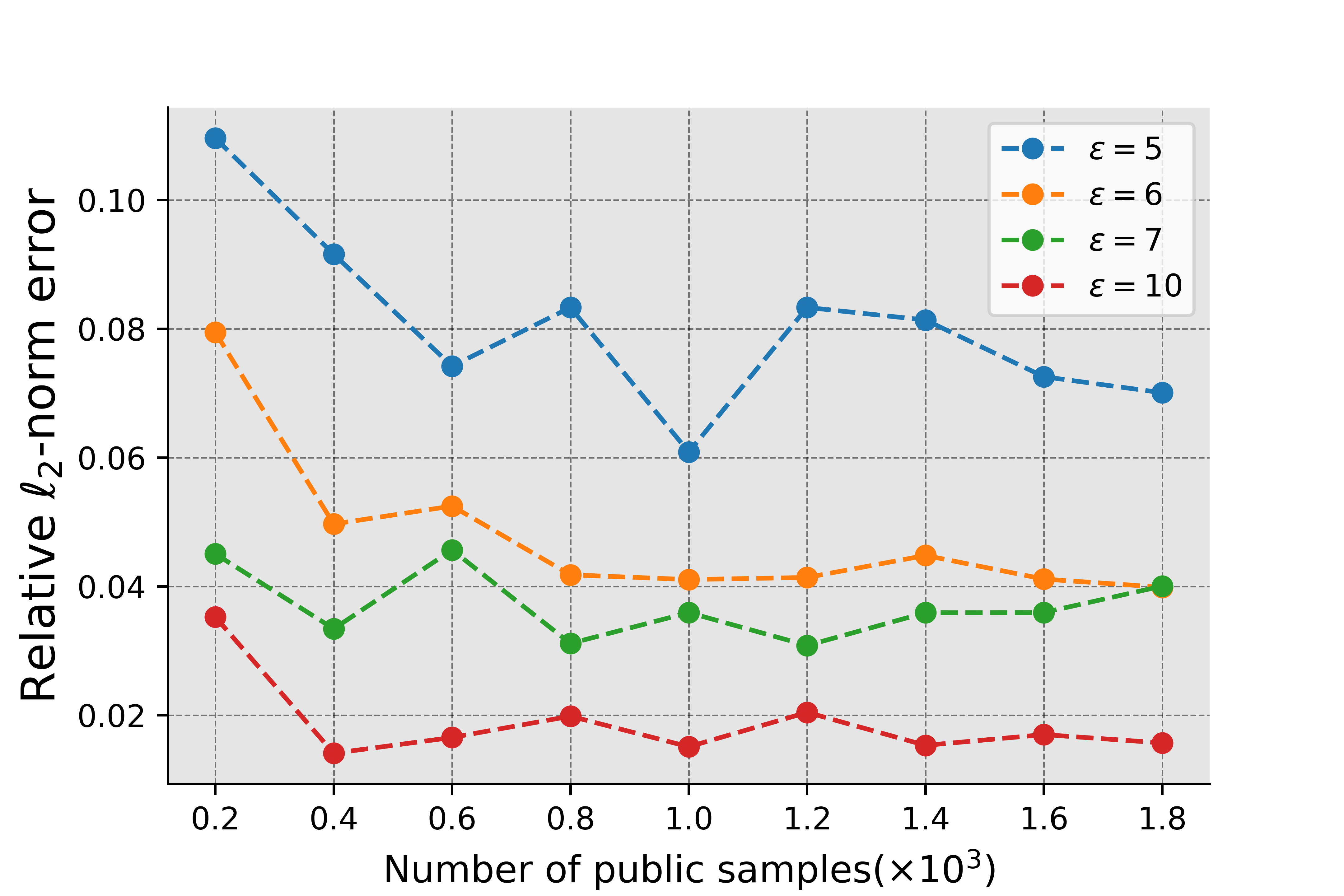}}
    \caption{  Algorithm \ref{alg:0} for logistic regression where the covariance matrix of Gaussian distribution is non-diagonal. The left plot shows the relative error with different dimension $p$. The middle and the right plots show the relative error with different size of public data $m$  when $n$ and $p$ are fixed. \label{fig:alg1_org_public_whole}
    }
\end{figure*}

\begin{figure*}[!ht]
\centering
\subfigure[ $p=20$ \label{fig:alg2_exponential_p20}]
{\includegraphics[width=0.32\textwidth, height=0.16\textheight]{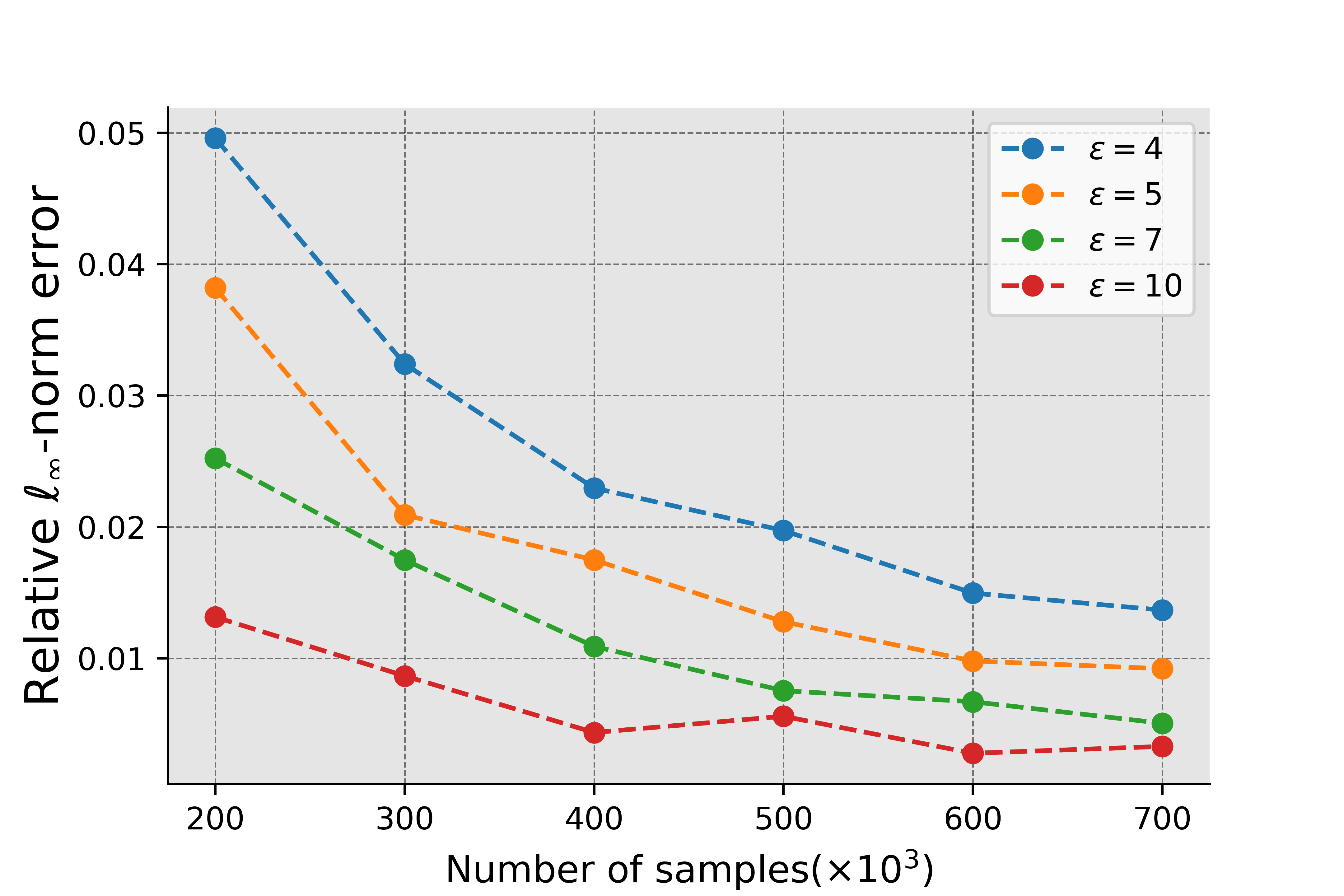}}
\subfigure[ $p=30$  \label{fig:alg2_exponential_p30}]
{\includegraphics[width=0.32\textwidth, height=0.16\textheight]{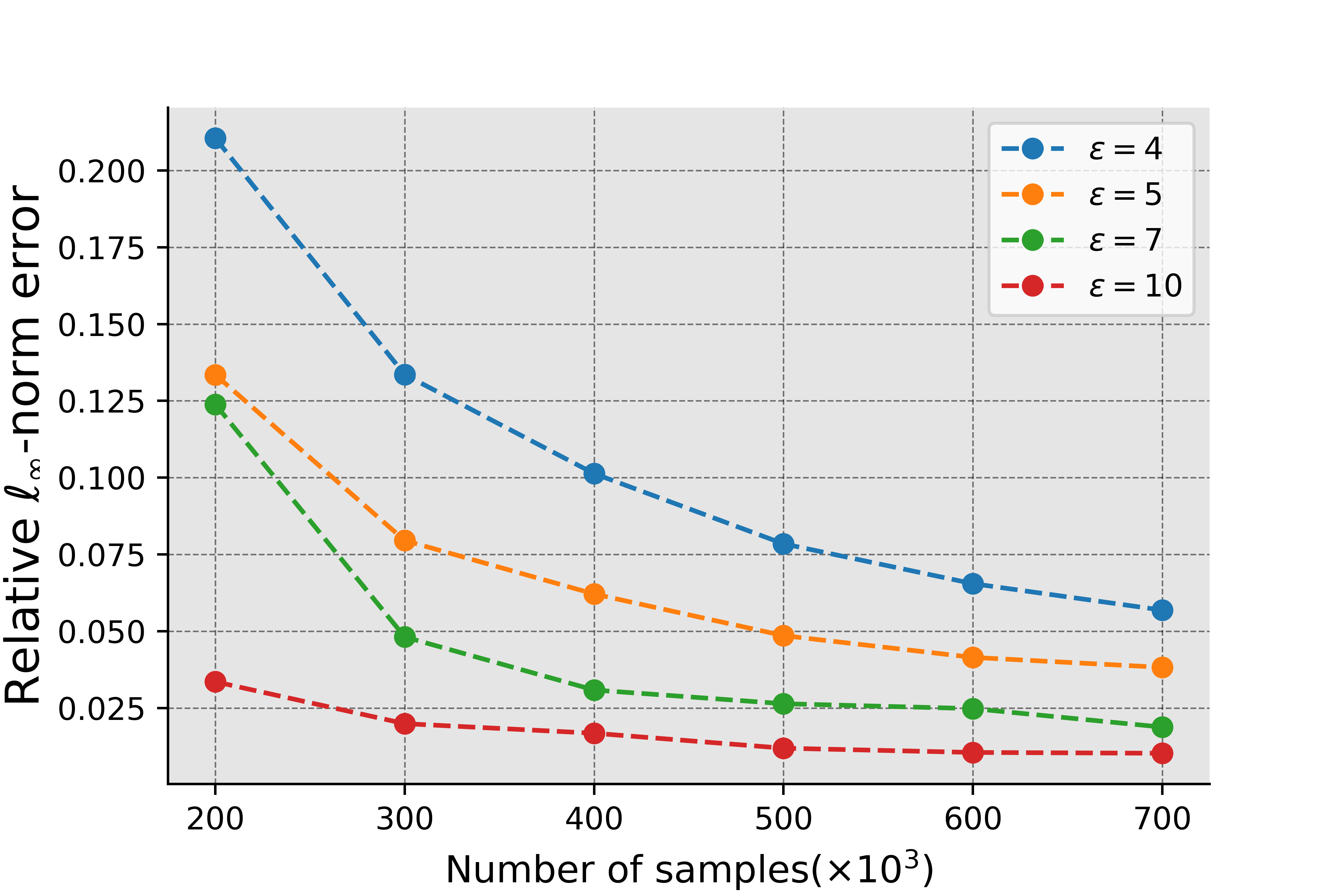}}
\subfigure[ $p=40$  \label{fig:alg2_exponential_p40}]
{\includegraphics[width=0.32\textwidth, height=0.16\textheight]{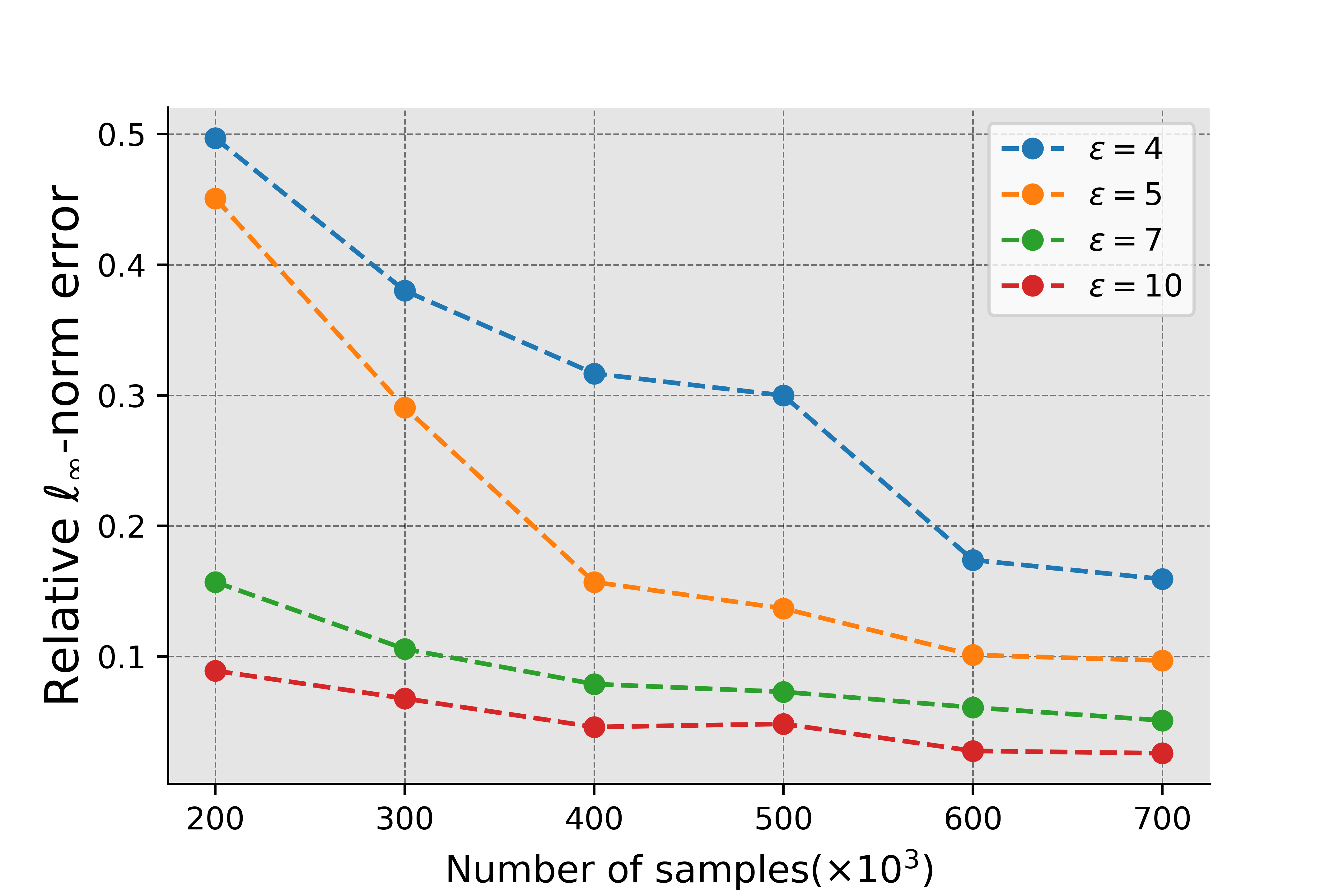}}
    \caption{  Algorithm \ref{alg:1} for exponential regression with Bernoulli data under different dimension $p$.  \label{fig:alg2_exp_epsilon_whole}
    }
\end{figure*}

\begin{figure*}[!ht]
\centering
\subfigure[$\epsilon=5$ \label{fig:alg2_exponential_e5}]
{\includegraphics[width=0.32\textwidth,  height=0.16\textheight]{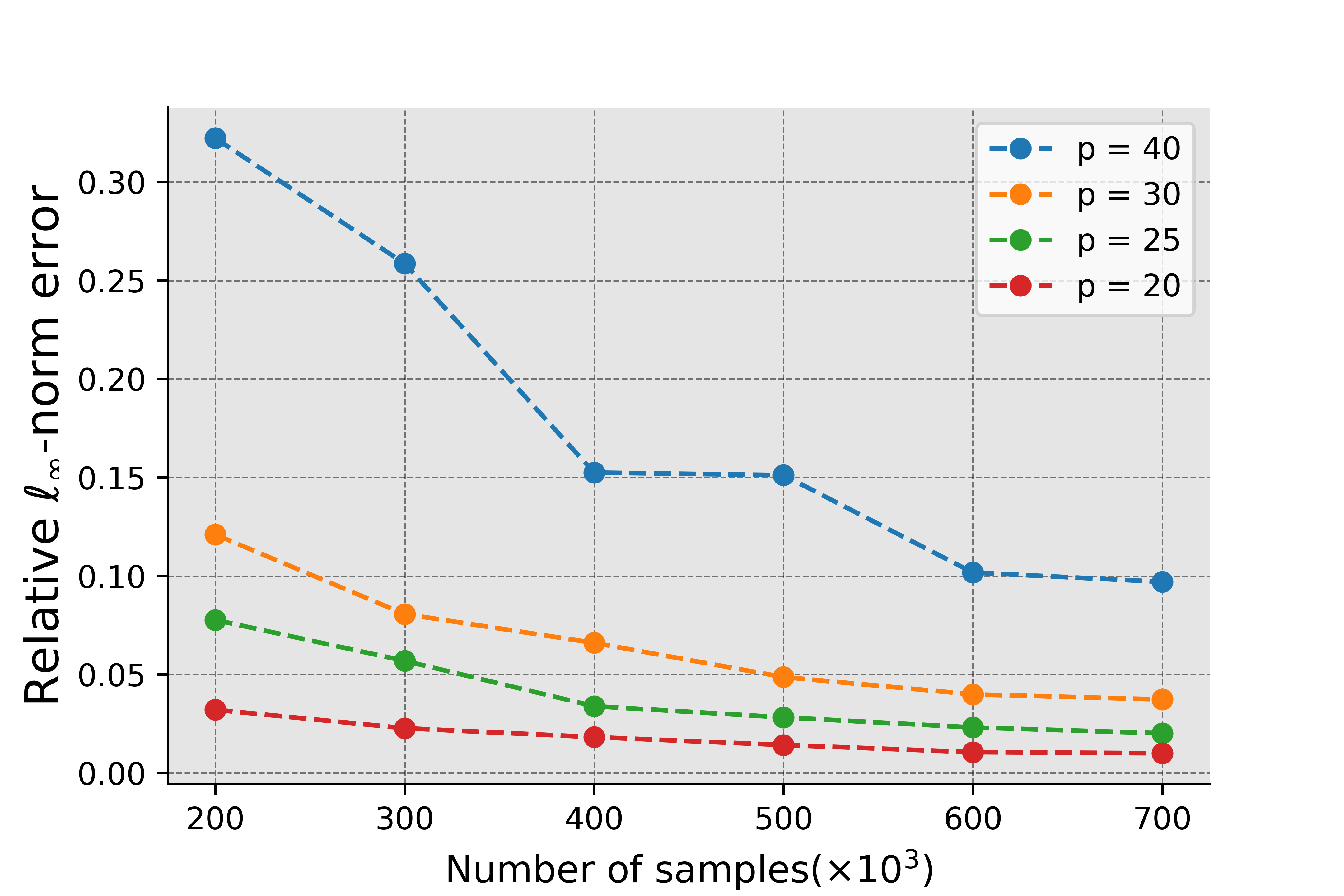}}
\subfigure[ $n=10^5$ and $p=20$  \label{fig:alg2_exponential_n10w_p20}]
{\includegraphics[width=0.32\textwidth, height=0.16\textheight]{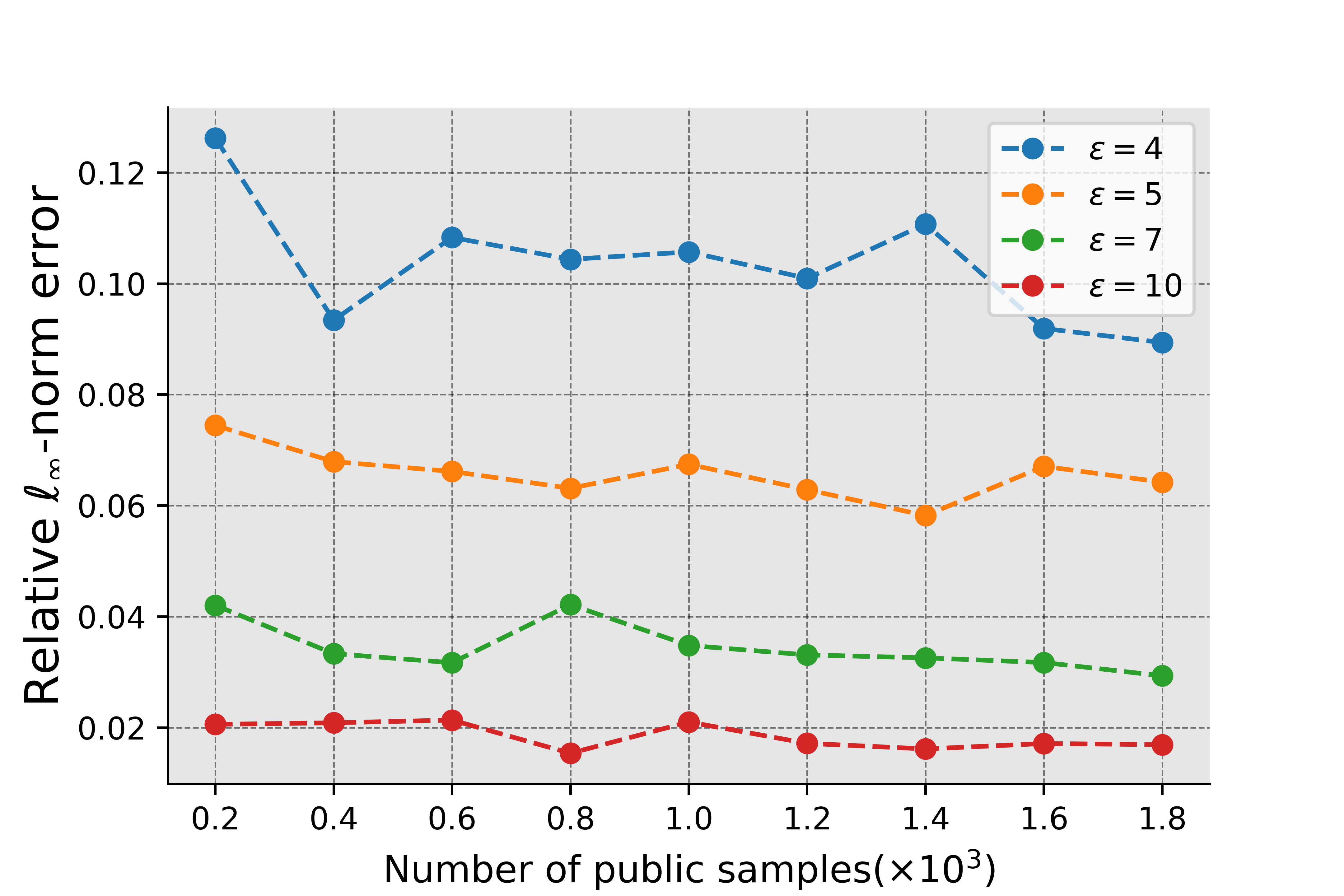}}
\subfigure[  $n=4\times 10^5$ and $p=20$ \label{fig:alg2_exponential_n40w_p20_2}]
{\includegraphics[width=0.32\textwidth, height=0.16\textheight]{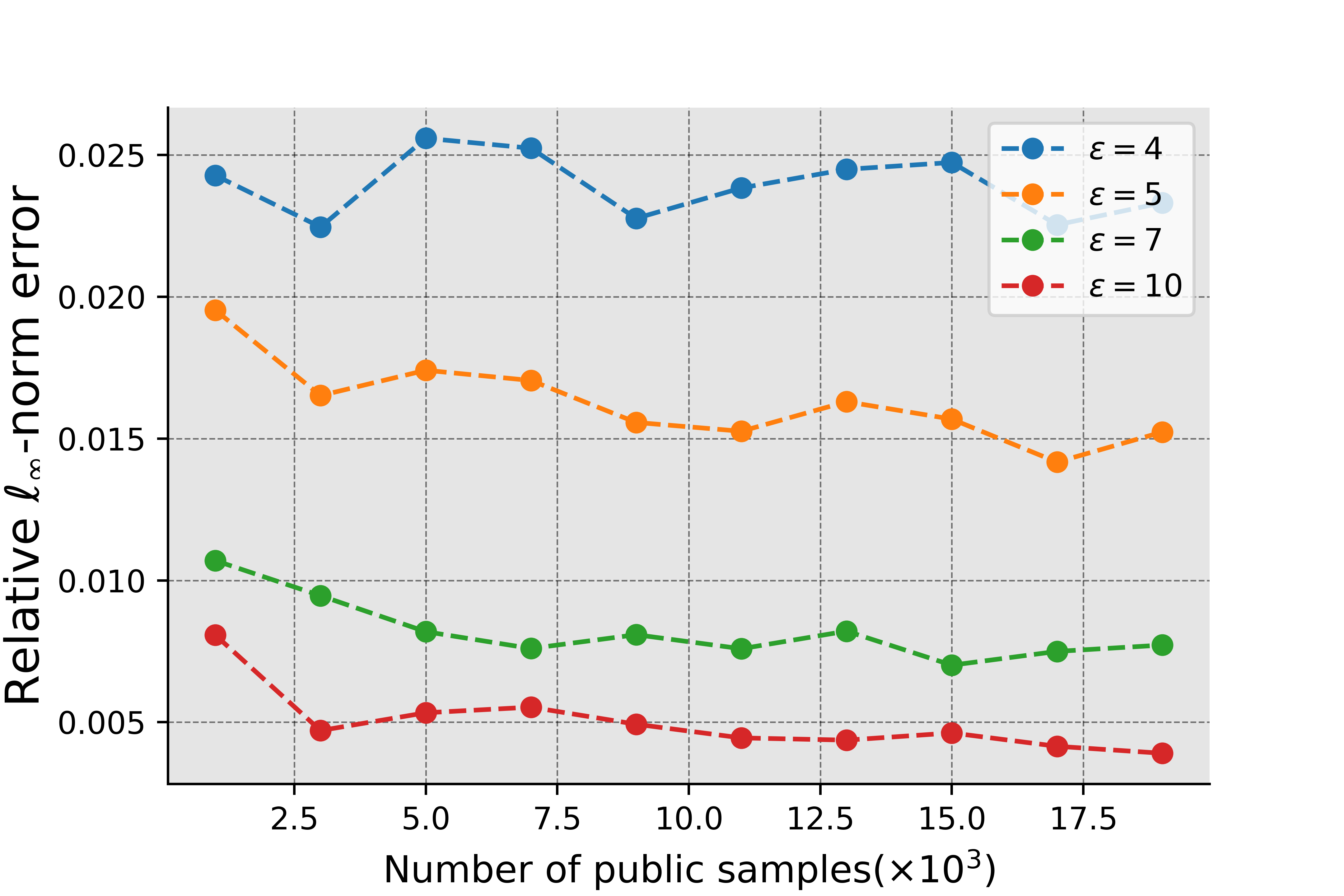}}
    \caption{  Algorithm \ref{alg:1} for exponential regression with Bernoulli data. The left plot shows the relative error with different dimension $p$. The middle and the right plots show the relative error with different size of public data $m$  when $n$ and $p$ are fixed. \label{fig:alg2_exp_public_whole}
    }
\end{figure*}
\begin{figure*}[!ht]
\centering
\subfigure[ $p=20$ \label{fig:alg2_logistic_p20}]
{\includegraphics[width=0.32\textwidth, height=0.18\textheight]{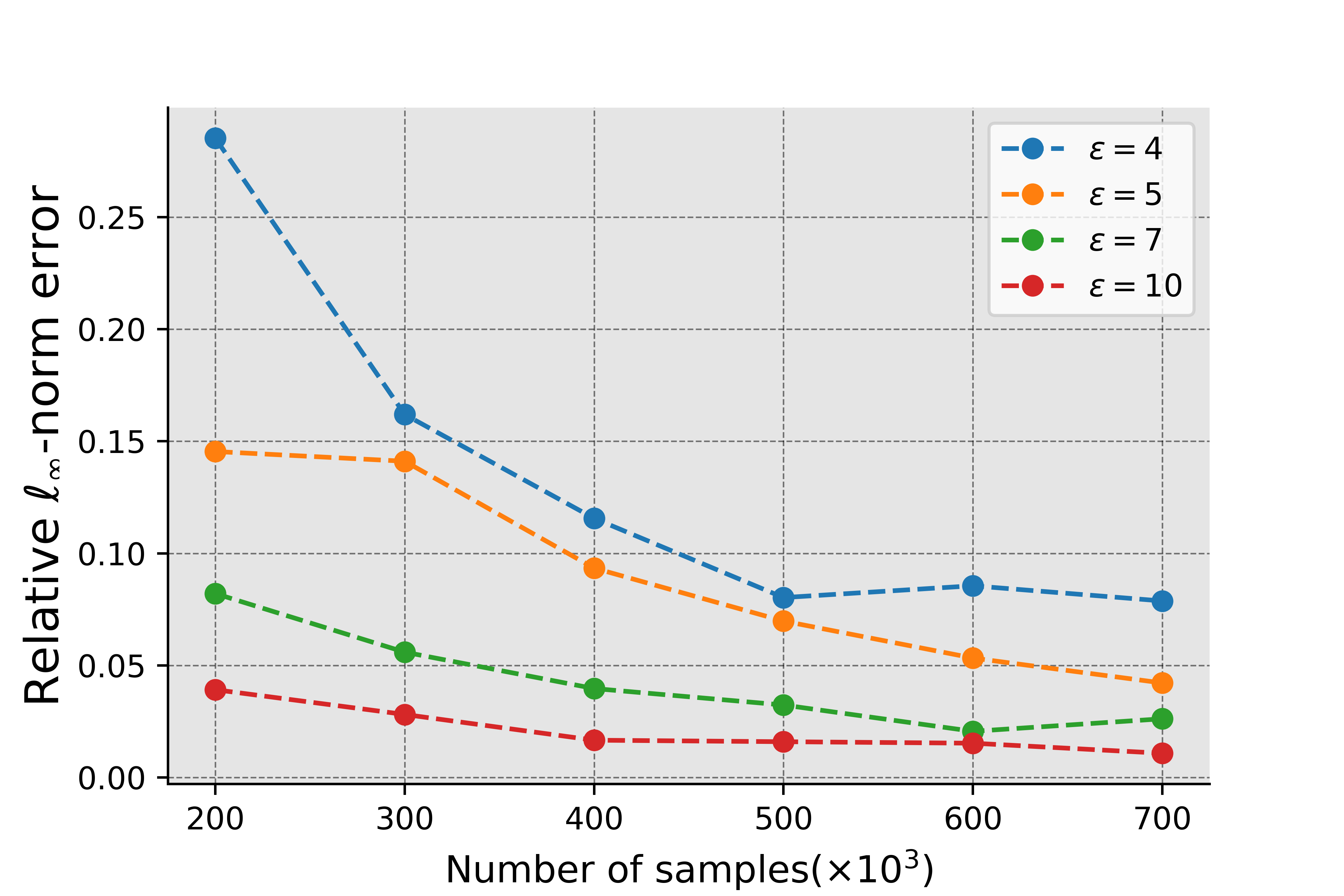}}
\subfigure[ $p=30$  \label{fig:alg2_logistic_p30}]
{\includegraphics[width=0.32\textwidth, height=0.18\textheight]{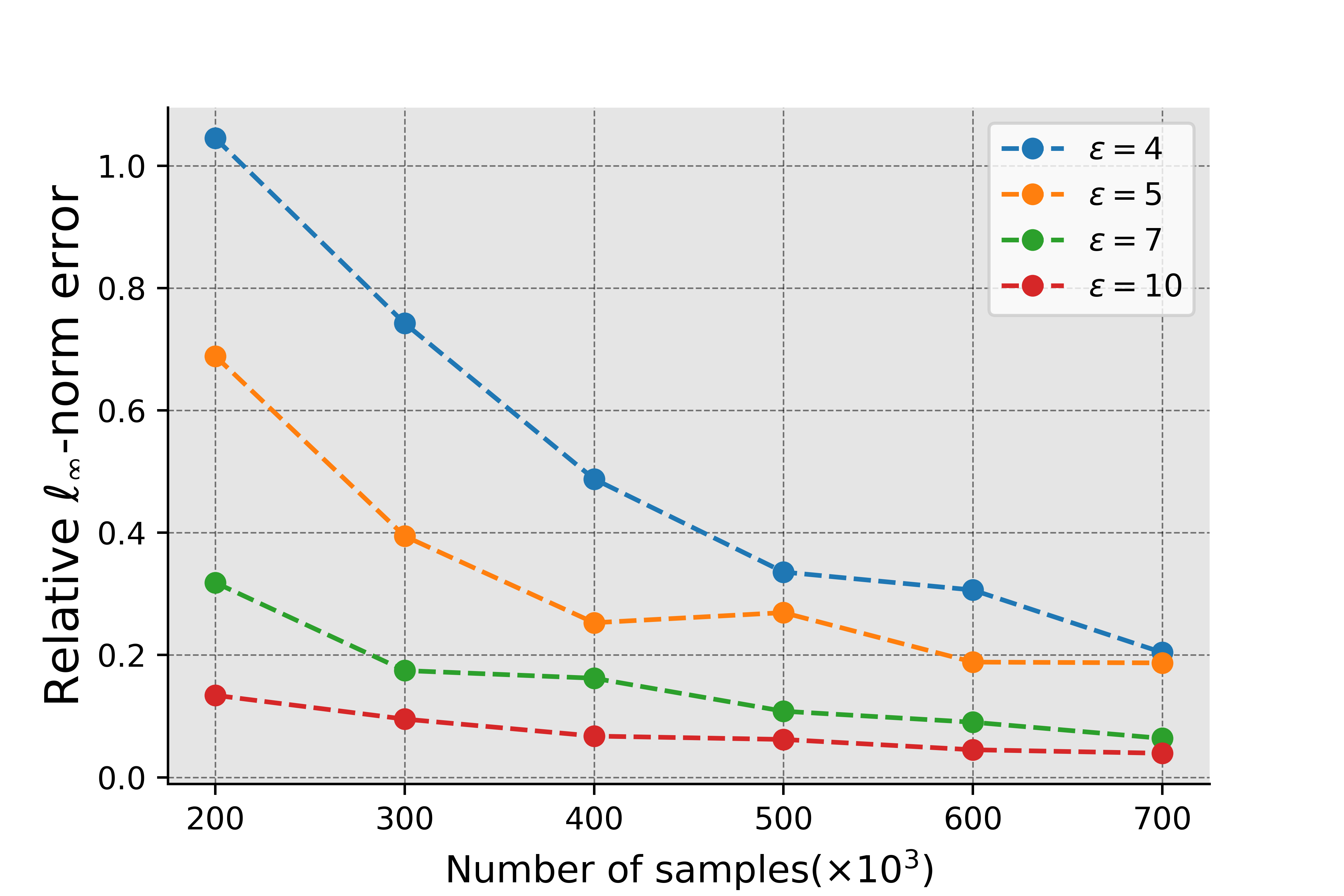}}
\subfigure[ $p=40$  \label{fig:alg2_logistic_p40}]
{\includegraphics[width=0.32\textwidth, height=0.18\textheight]{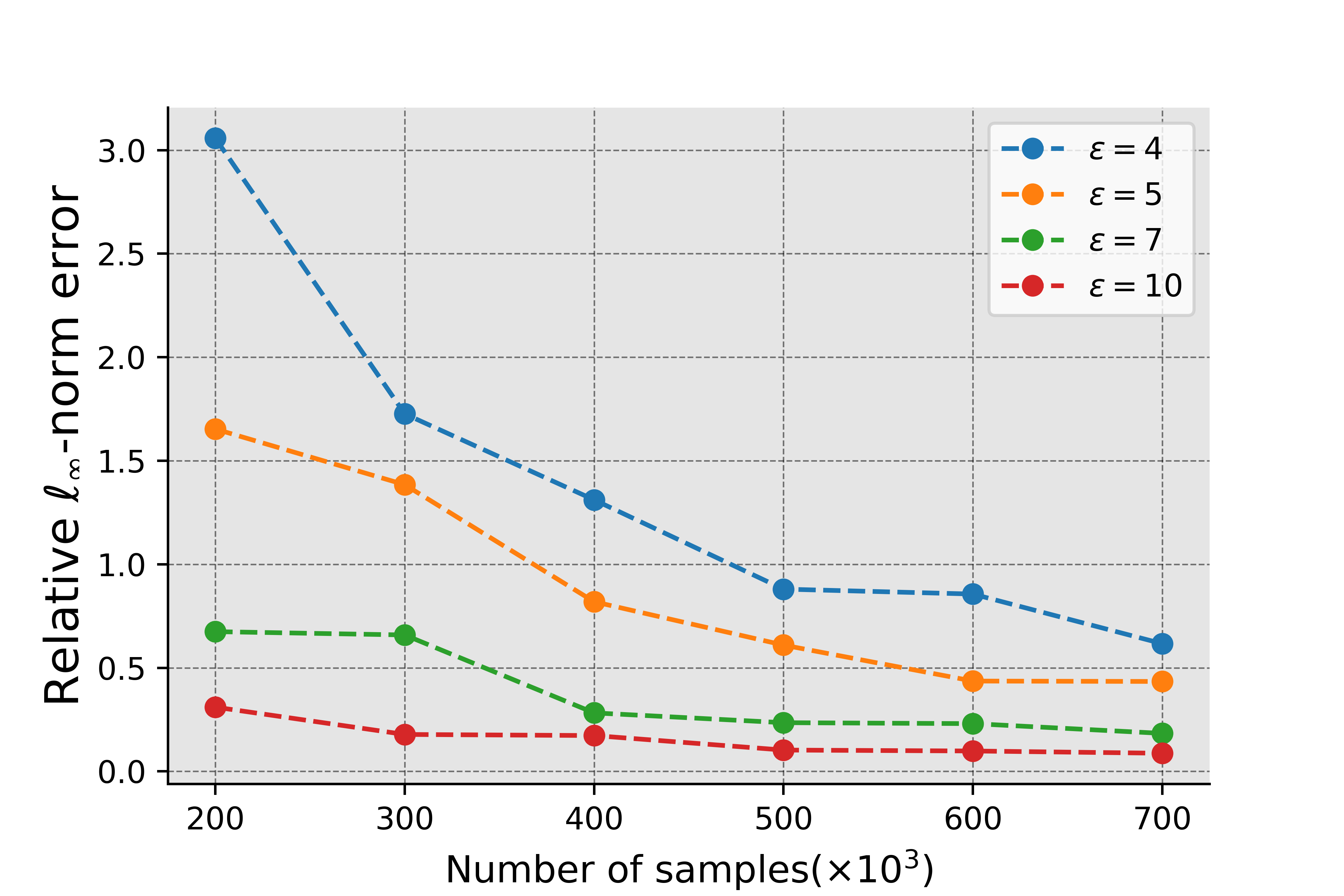}}
    \caption{  Algorithm \ref{alg:1} for logistic regression with Bernoulli data under different dimension $p$.  \label{fig:alg2_log_epsilon_whole}
    }
\end{figure*}

\begin{figure*}[!ht]
\centering
\subfigure[$\epsilon=5$ \label{fig:alg2_logistic_e5}]
{\includegraphics[width=0.32\textwidth,  height=0.18\textheight]{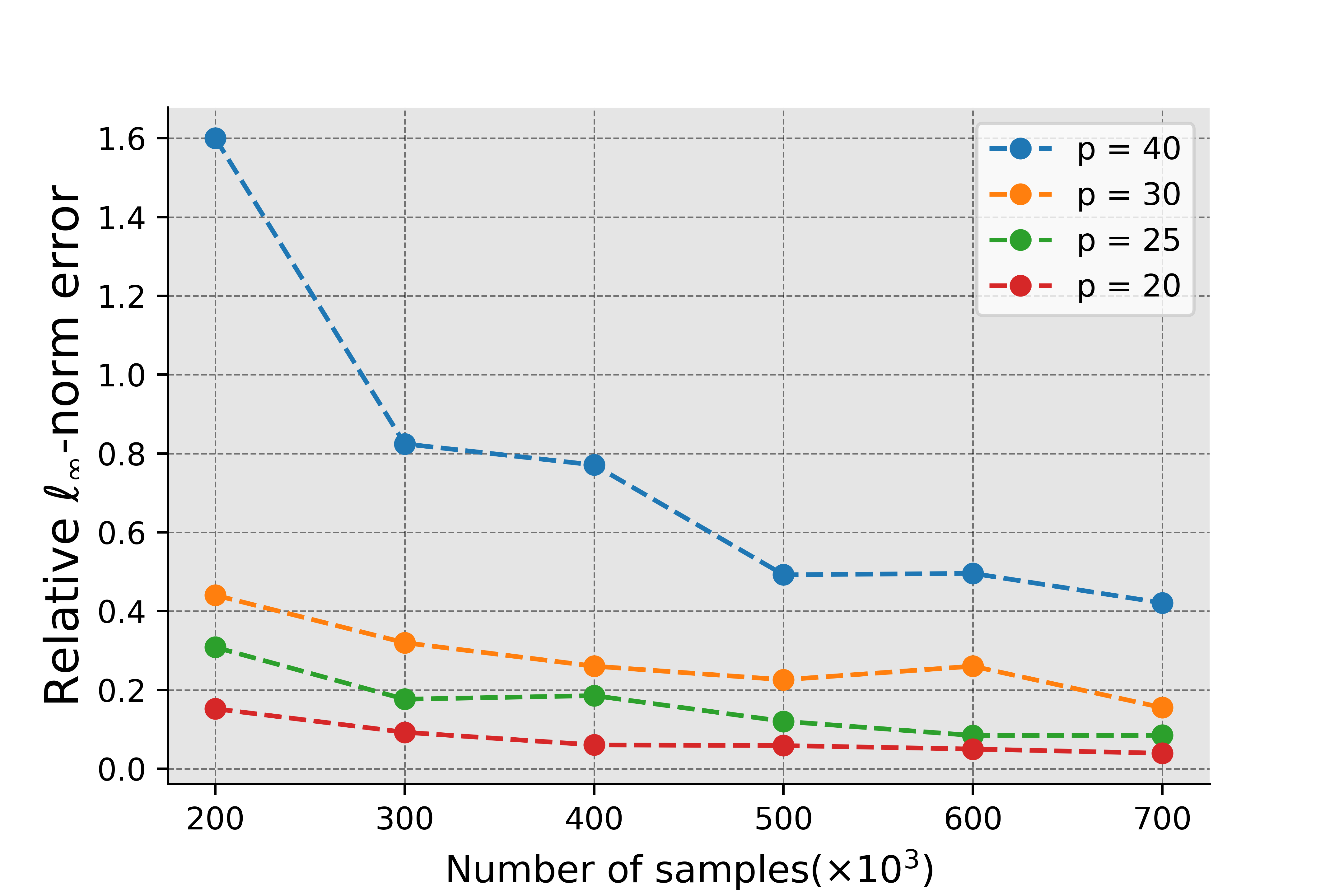}}
\subfigure[ $n=10^5$ and $p=20$  \label{fig:alg2_logistic_n10w_p20}]
{\includegraphics[width=0.32\textwidth, height=0.18\textheight]{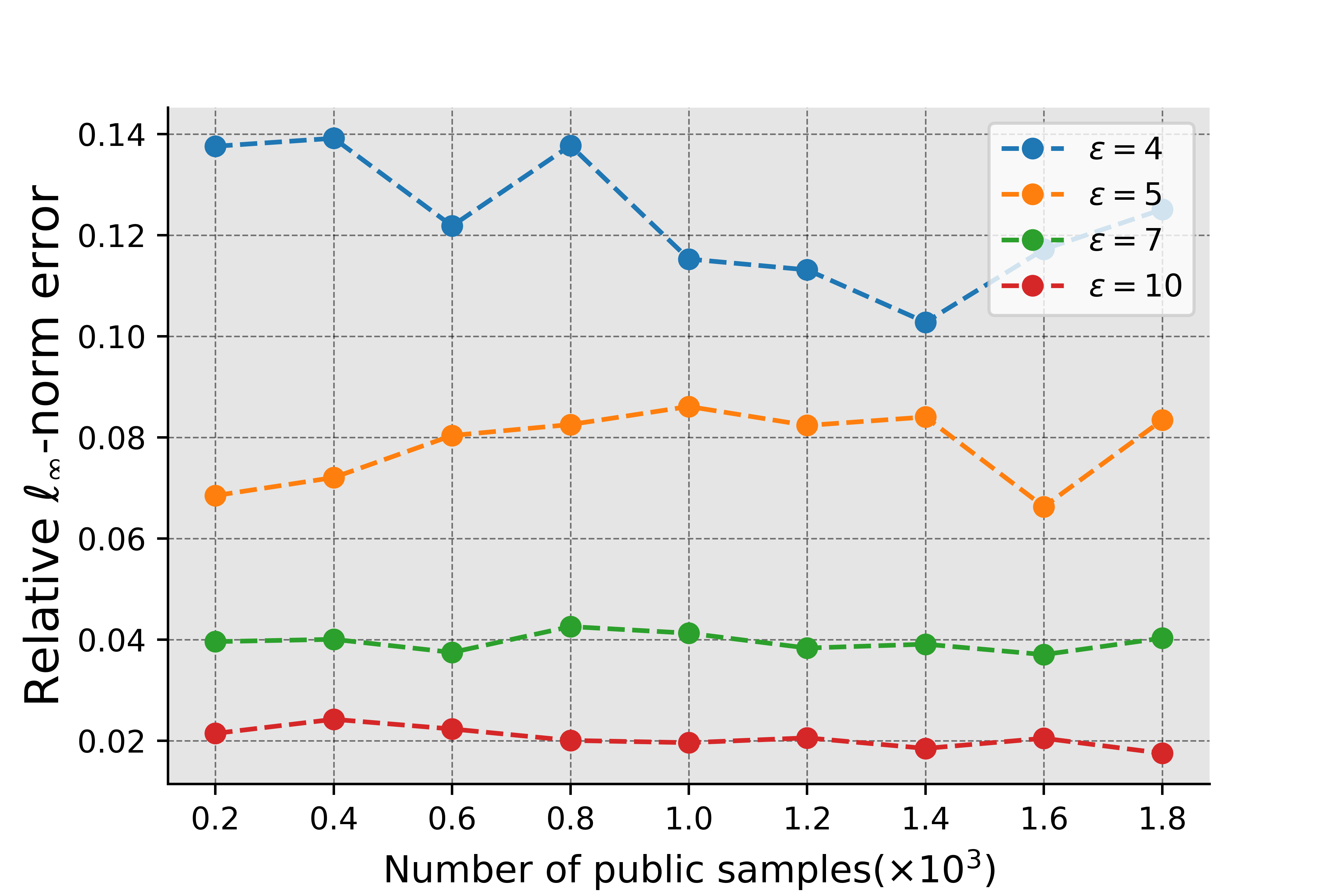}}
\subfigure[  $n=4\times 10^5$ and $p=20$ \label{fig:alg2_logistic_n40w_p20_2}]
{\includegraphics[width=0.32\textwidth, height=0.18\textheight]{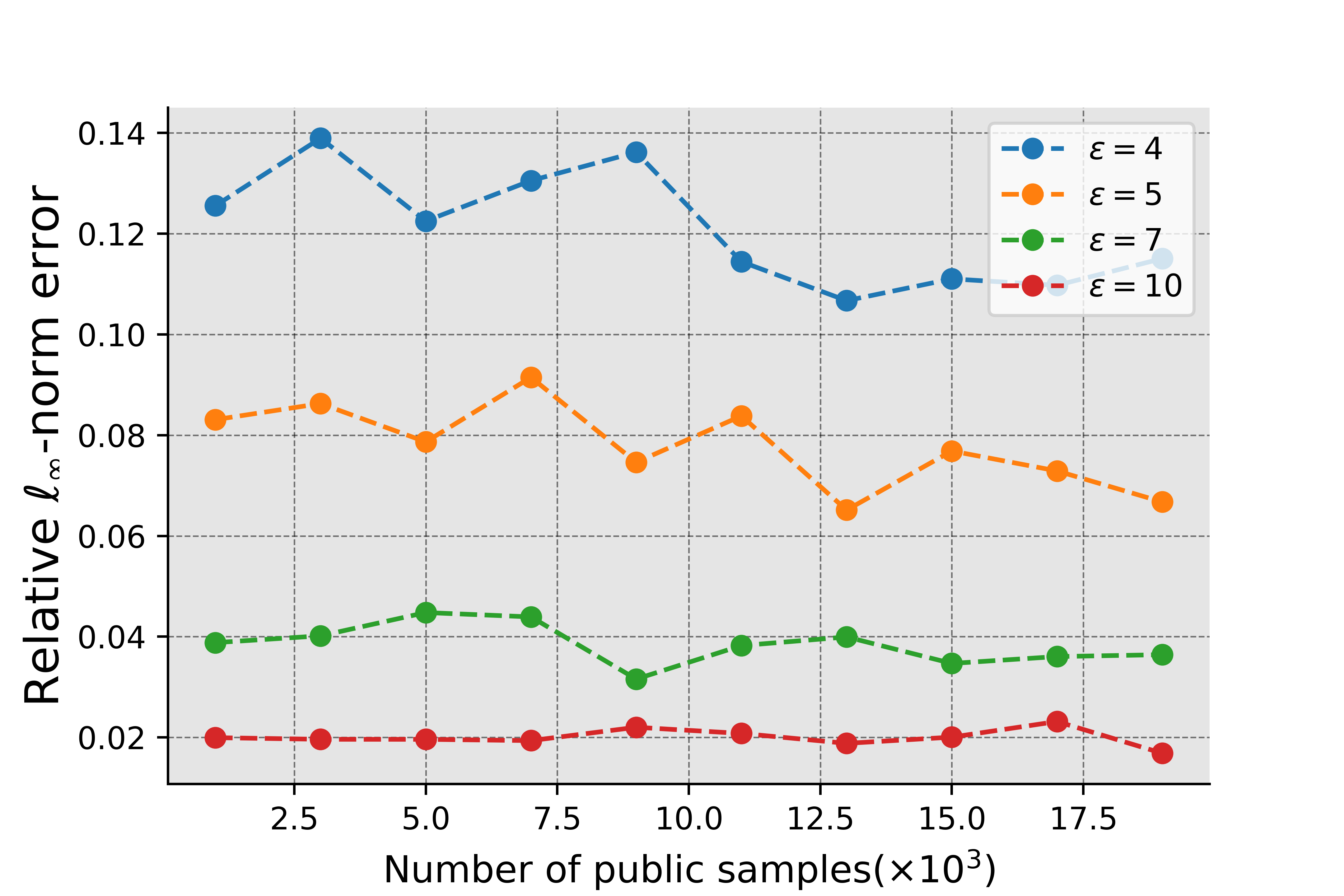}}
    \caption{  Algorithm \ref{alg:1} for logistic regression with Bernoulli data. The left plot shows the relative error with different dimension $p$. The middle and the right plots show the relative error with different size of public data $m$ when $n$ and $p$ are fixed. \label{fig:alg2_log_public_whole}
    }
\end{figure*}
\begin{figure*}[!ht]
\centering
\subfigure[ $p=10$\label{fig:alg3_sigmoid_diag_p10}]
{\includegraphics[width=0.32\textwidth, height=0.18\textheight]{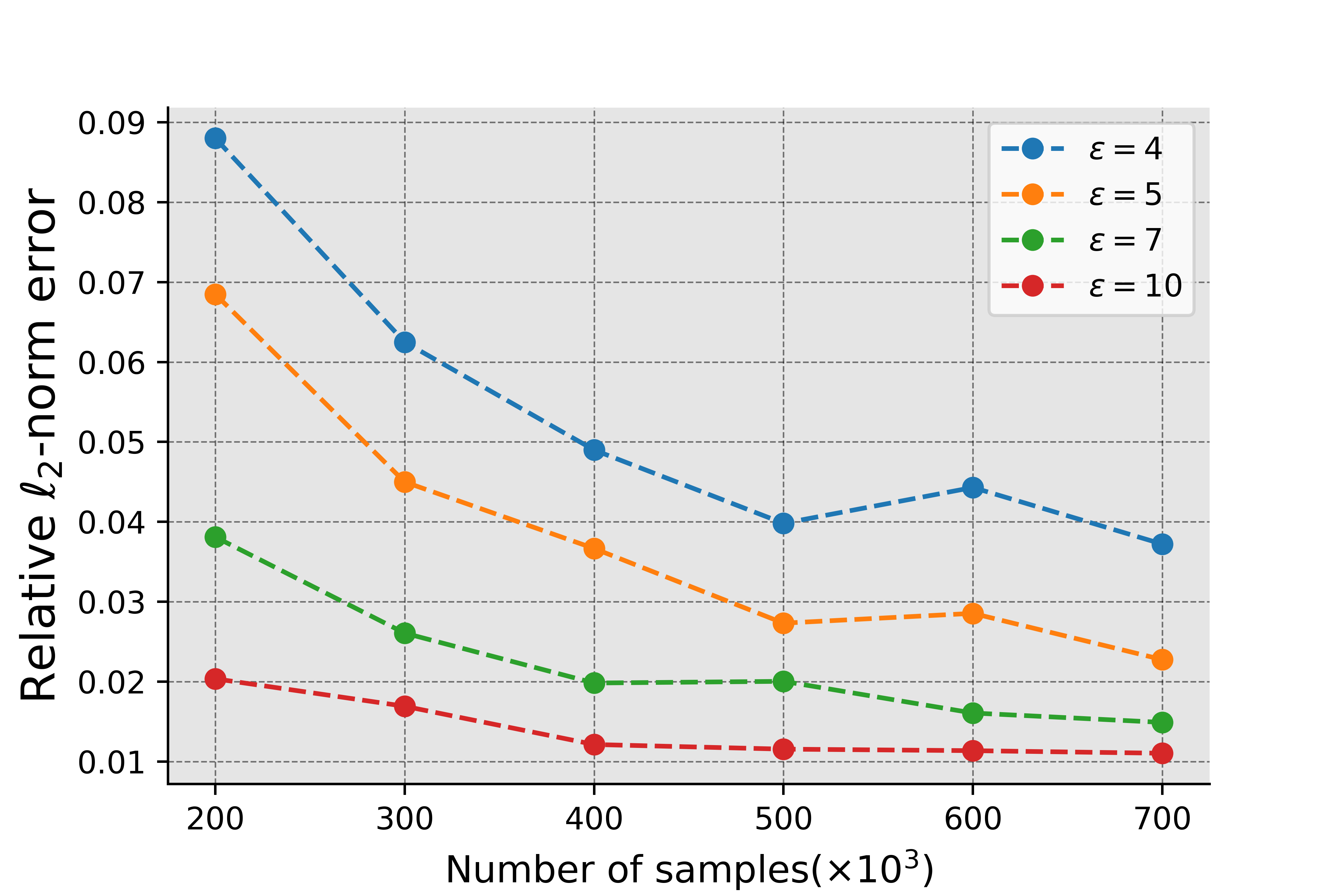}
}
\subfigure[ $p=20$  \label{fig:alg3_sigmoid_diag_p20}]
{\includegraphics[width=0.32\textwidth, height=0.18\textheight]{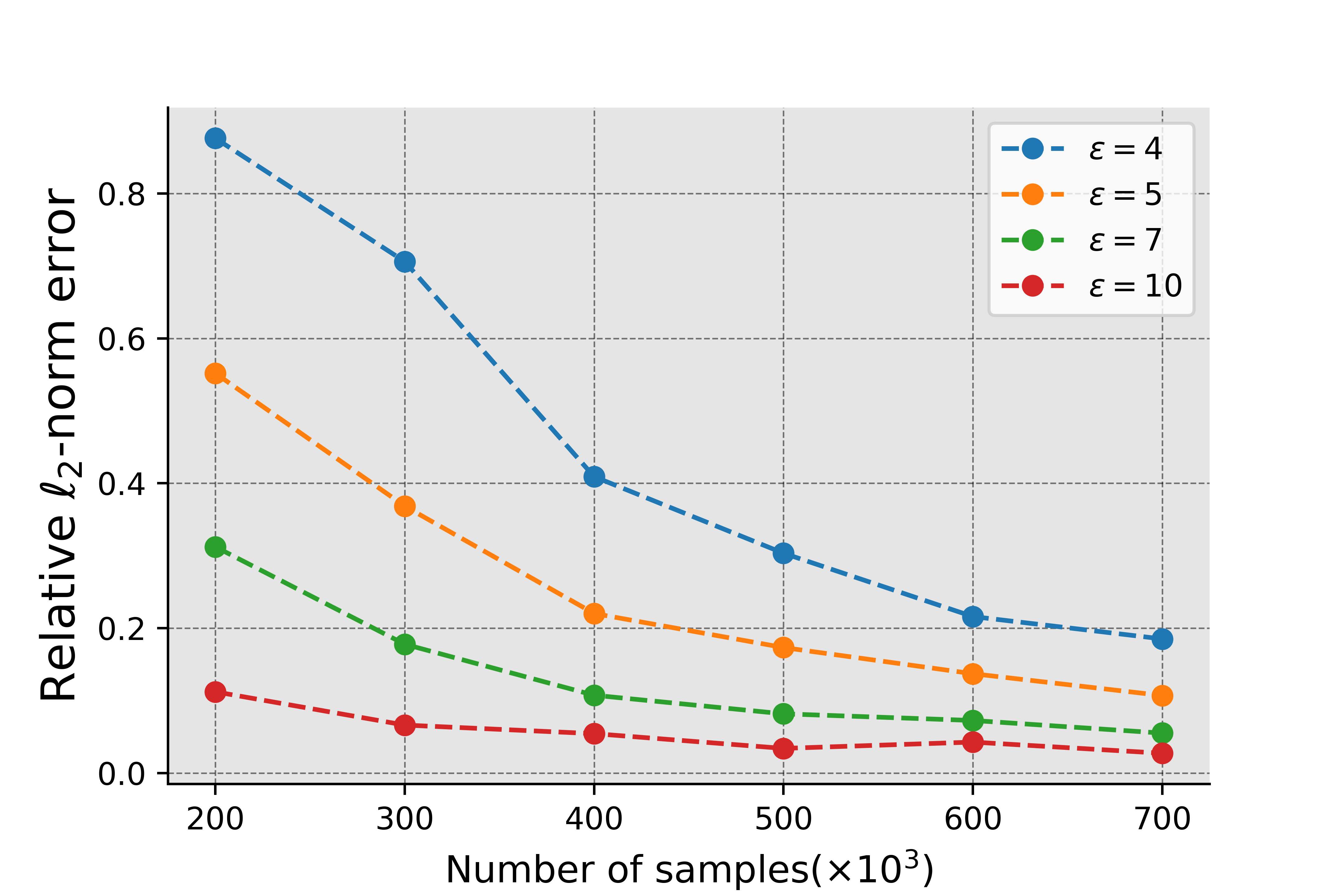}}
\subfigure[ $p=30$  \label{fig:alg3_sigmoid_diag_p30}]
{\includegraphics[width=0.32\textwidth, height=0.18\textheight]{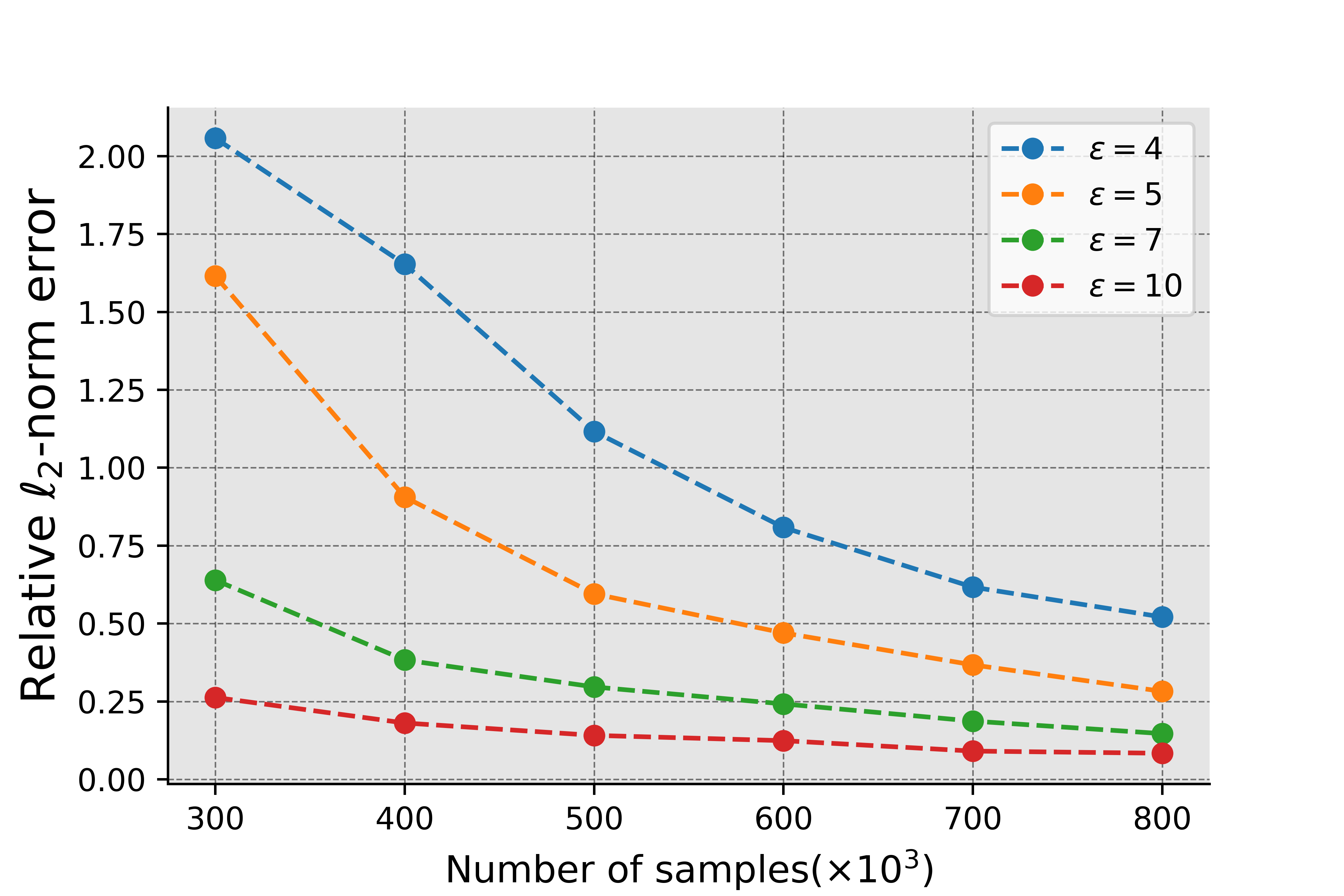}}
    \caption{ Algorithm \ref{alg:1.5} with sigmoid link function where the covariance matrix of Gaussian distribution is diagonal under different dimension $p$. \label{fig:alg3_diag_epsilon_whole}
    }
\end{figure*}

\begin{figure*}[!ht]
\centering
\subfigure[ $\epsilon=5$ \label{fig:alg3_sigmoid_diag_e5}]
{\includegraphics[width=0.32\textwidth, height=0.16\textheight]{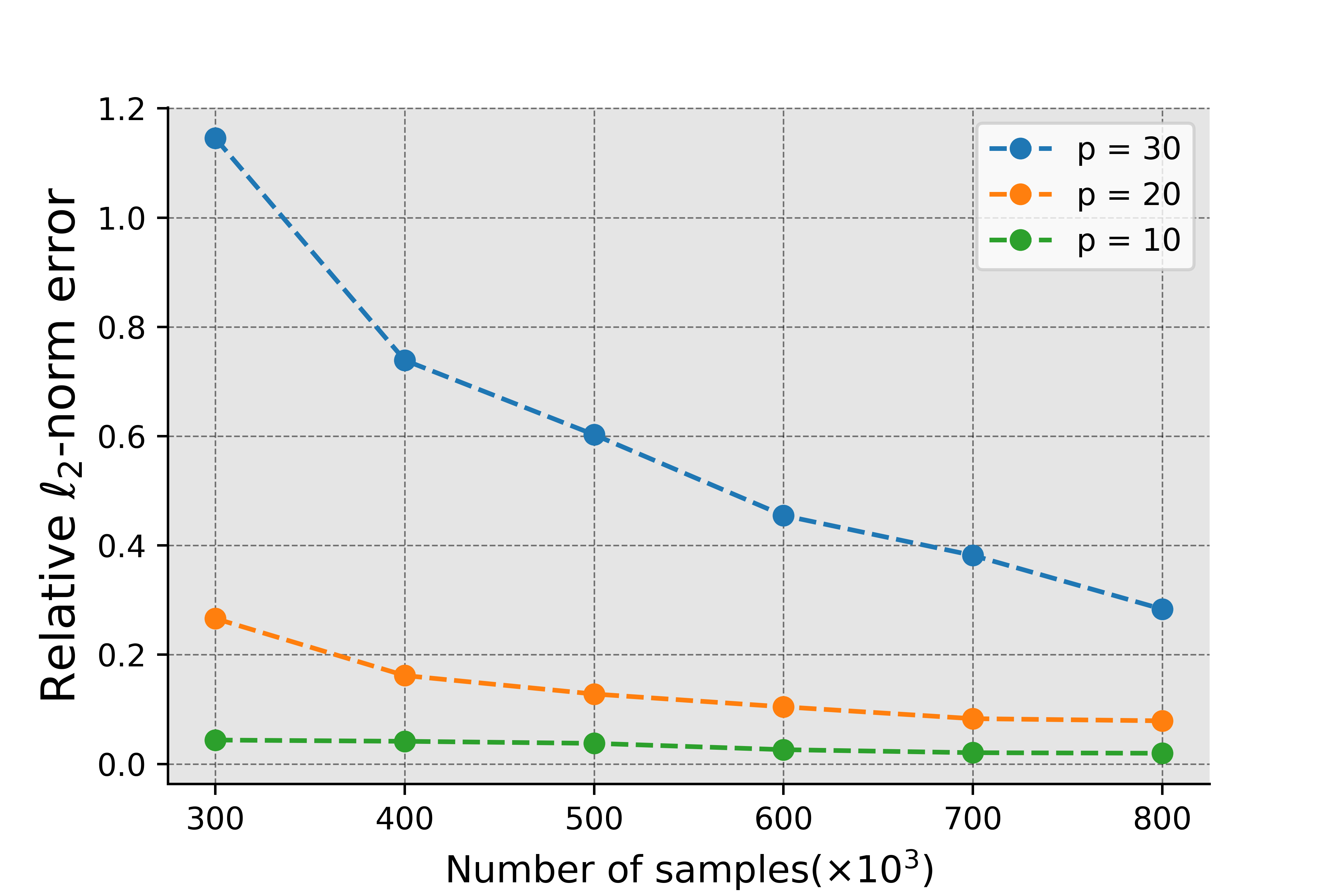}}
\subfigure[  $n=2\times 10^5$ and $p=20$ \label{fig:alg3_sigmoid_diag_n20w_p20}]
{\includegraphics[width=0.32\textwidth, height=0.16\textheight]{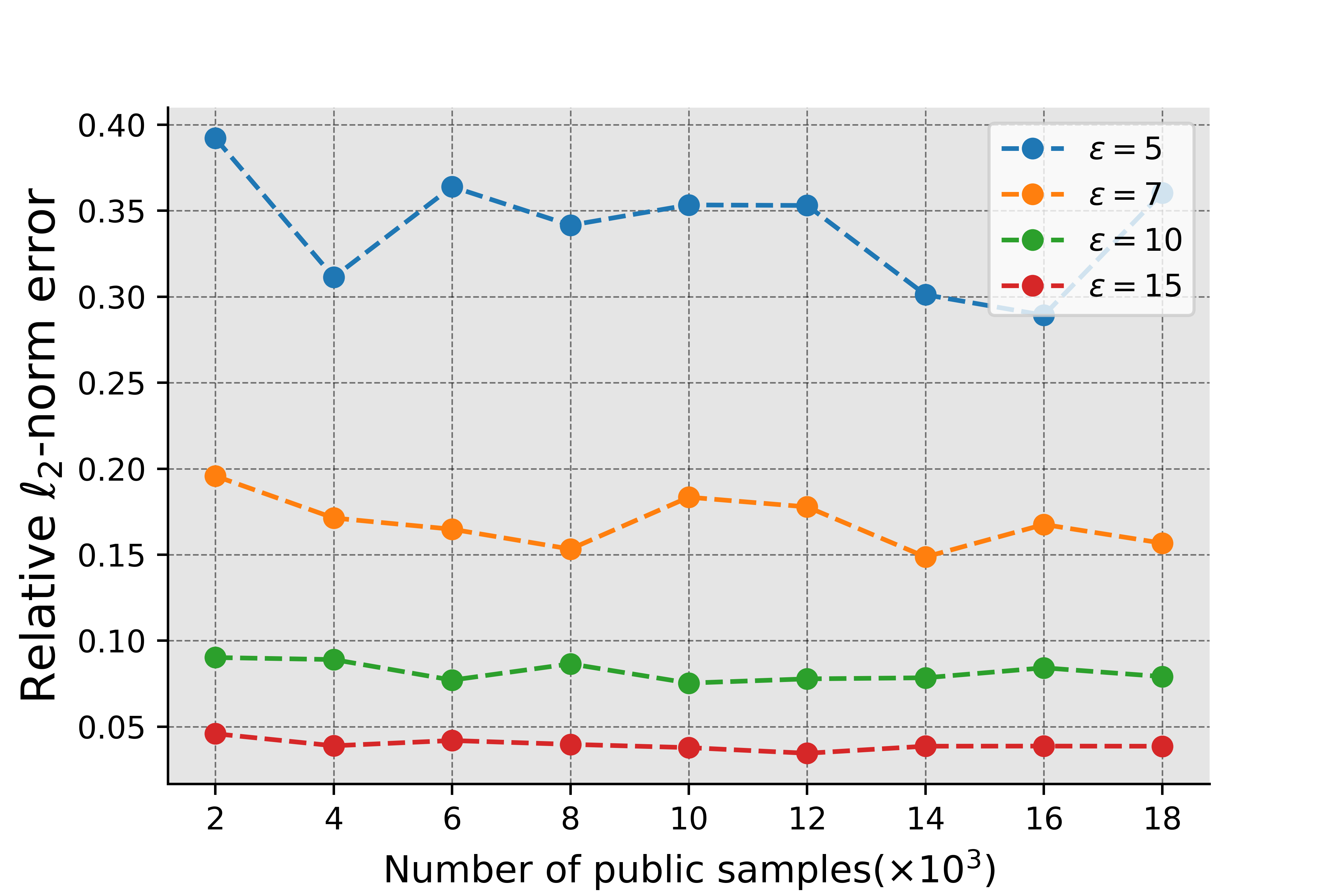}}
\subfigure[ $n=4\times 10^5$ and $p=20$  \label{fig:alg3_sigmoid_diag_n40w_p20}]
{\includegraphics[width=0.32\textwidth, height=0.16\textheight]{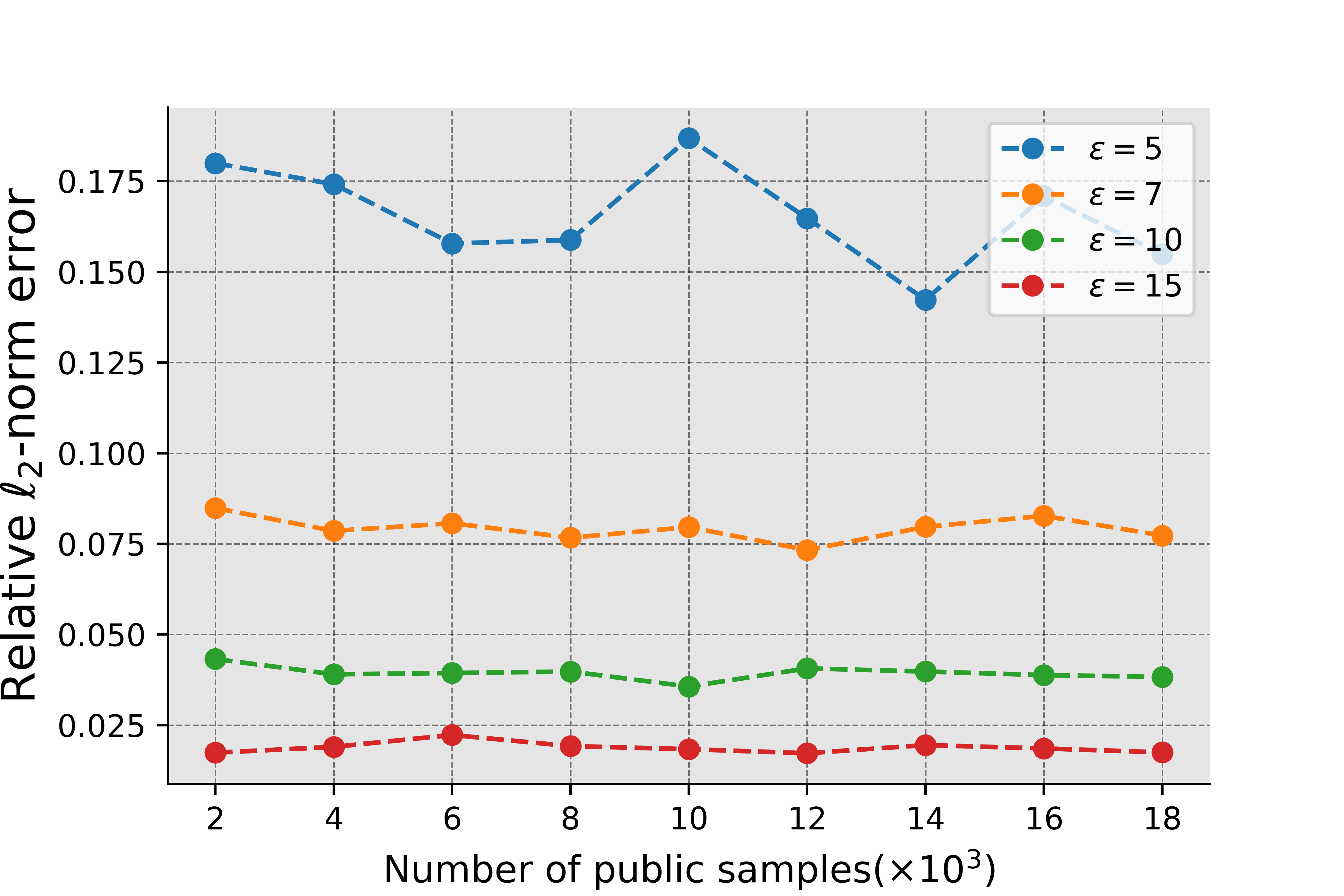}}
    \caption{ Algorithm \ref{alg:1.5}  with sigmoid link function where the covariance matrix of Gaussian distribution is diagonal. The left plot shows the relative error with different dimension $p$. The middle and the right plots show the relative error with different size of public data $m$  when $n$ and $p$ are fixed.  \label{fig:alg3_diag_public}
    }
\end{figure*}

\begin{figure*}[!ht]
\centering
\subfigure[ $p=20$ \label{fig:alg4_cubic_p20}]
{\includegraphics[width=0.32\textwidth, height=0.16\textheight]{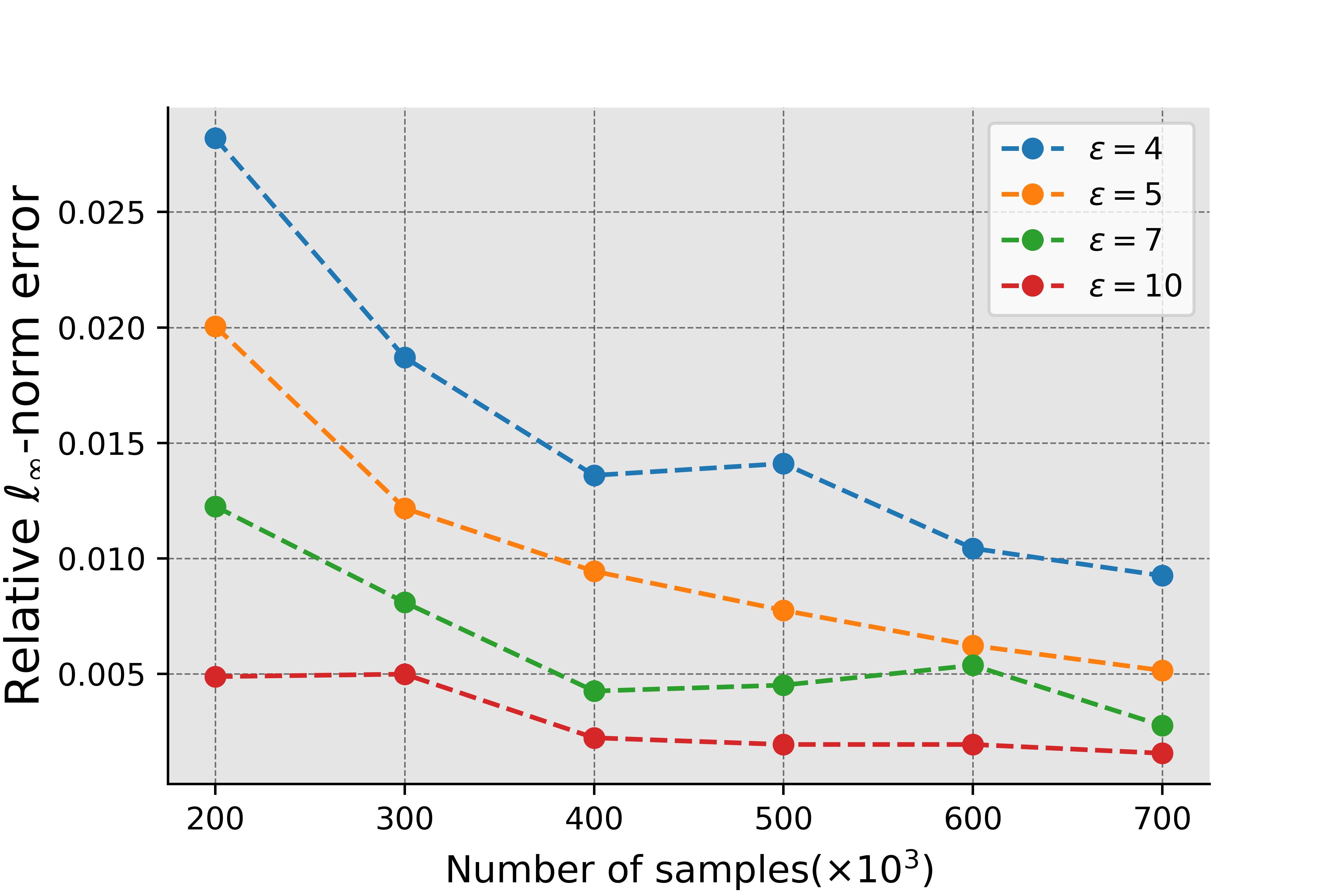}}
\subfigure[ $p=30$  \label{fig:alg4_cubic_p30}]
{\includegraphics[width=0.32\textwidth, height=0.16\textheight]{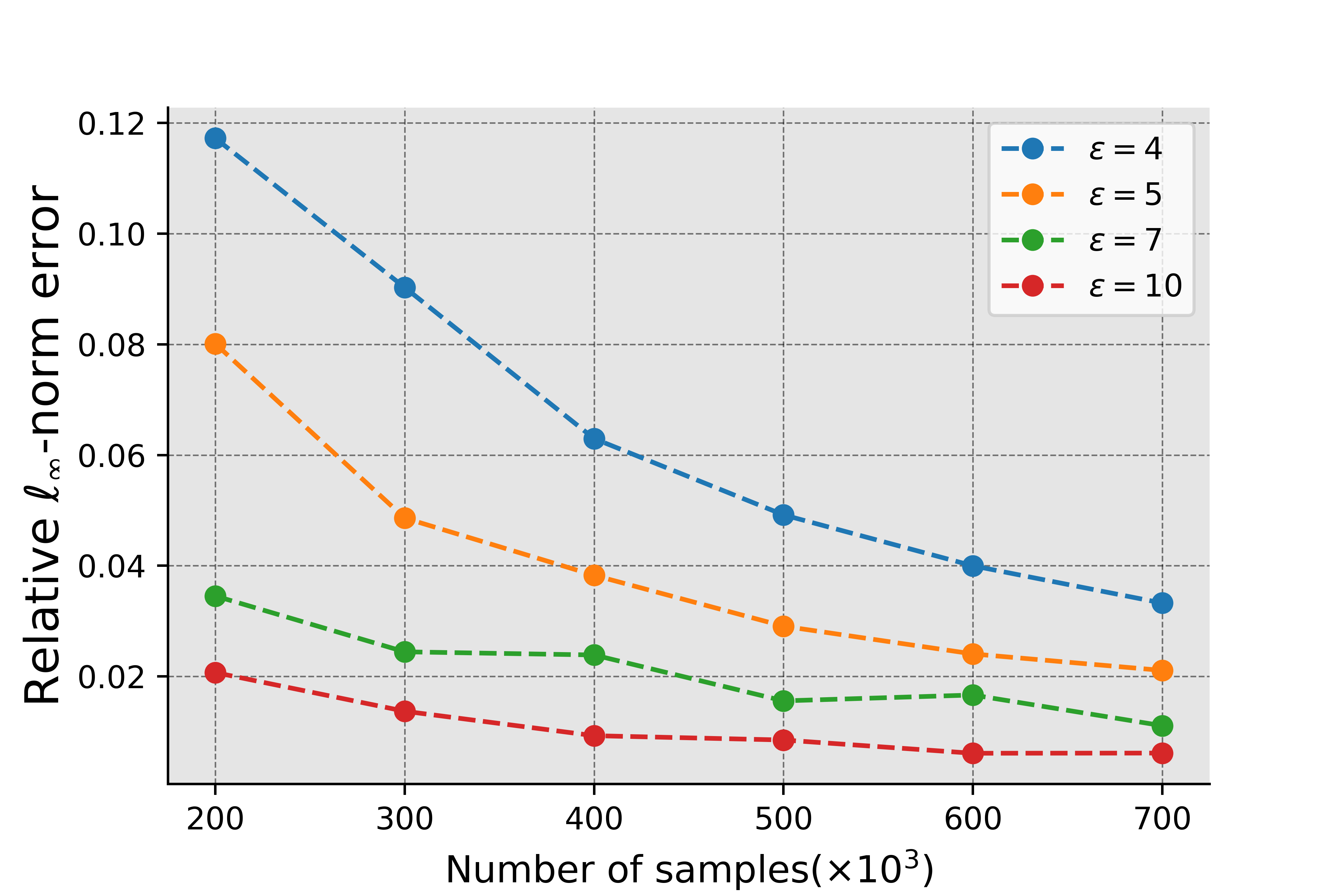}}
\subfigure[ $p=50$  \label{fig:alg4_cubic_p50}]
{\includegraphics[width=0.32\textwidth, height=0.16\textheight]{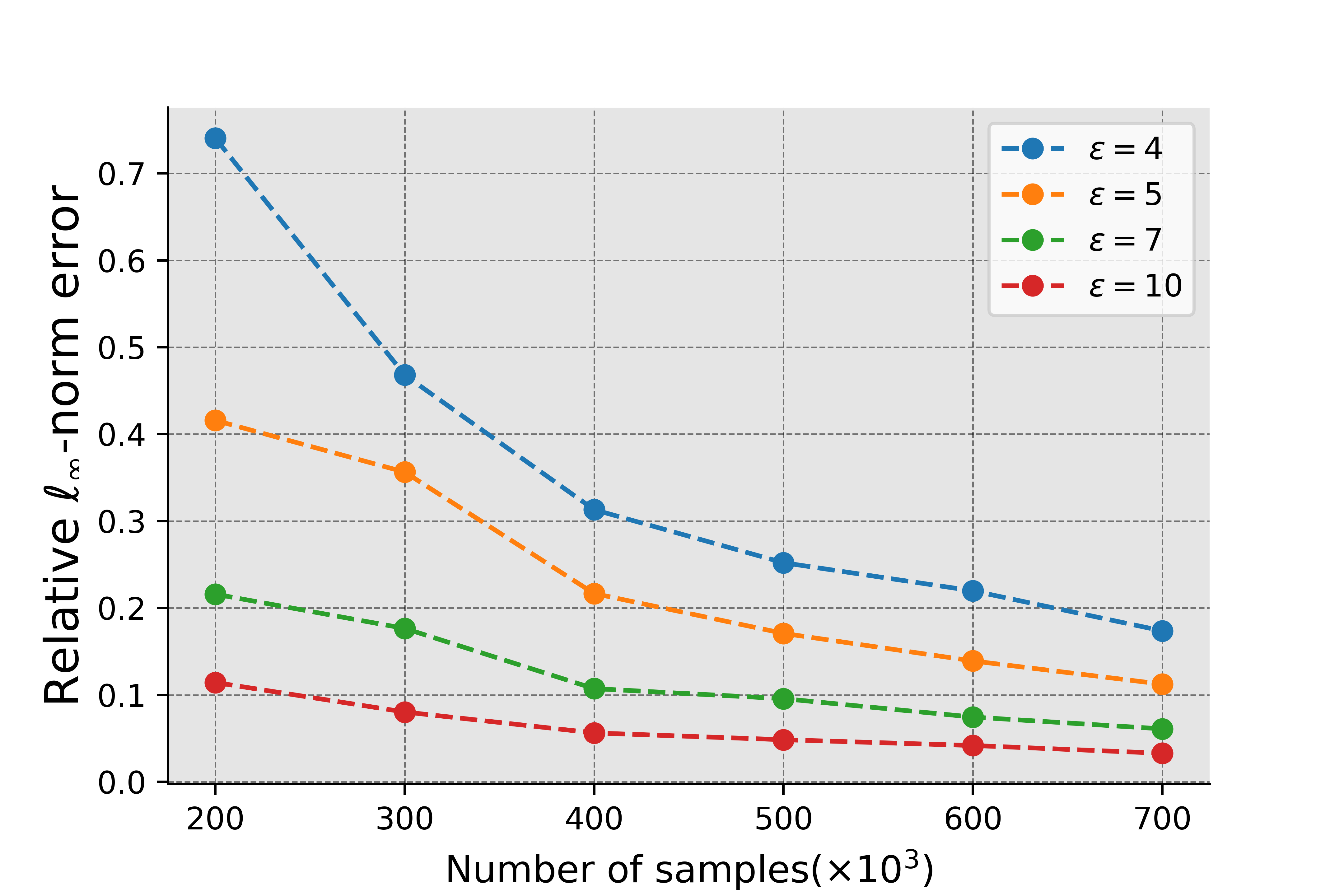}}
    \caption{  Algorithm \ref{alg:2} for cubic link function with Bernoulli data under different dimension $p$.  \label{fig:alg4_cubic_epsilon_whole}
    }
\end{figure*}

\begin{figure*}[!ht]
\centering
\subfigure[$\epsilon=5$ \label{fig:alg4_cubic_e5_2}]
{\includegraphics[width=0.32\textwidth,  height=0.16\textheight]{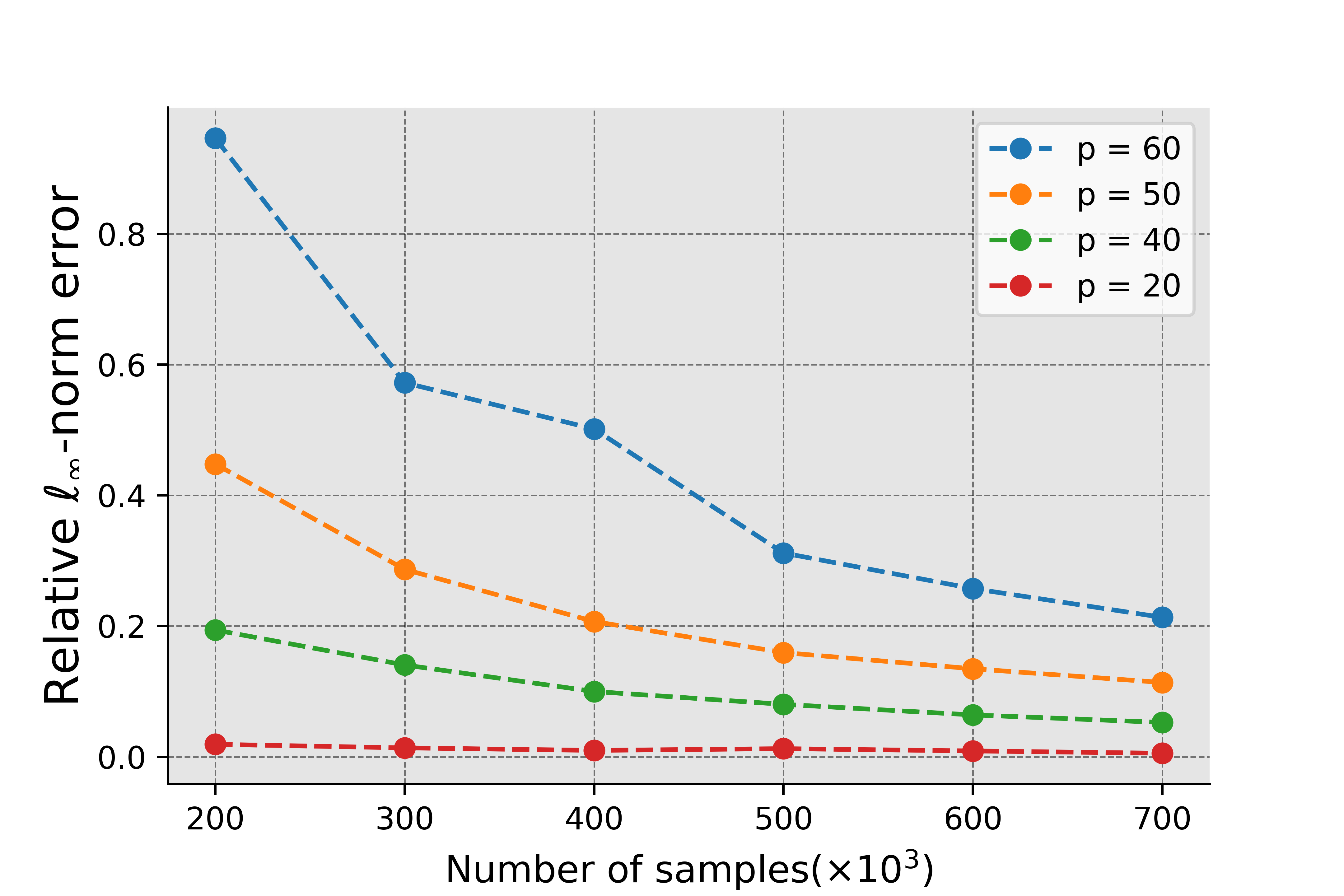}}
\subfigure[ $n=5\times 10^4$ and $p=30$  \label{fig:alg4_cubic_n5w_p30}]
{\includegraphics[width=0.32\textwidth, height=0.16\textheight]{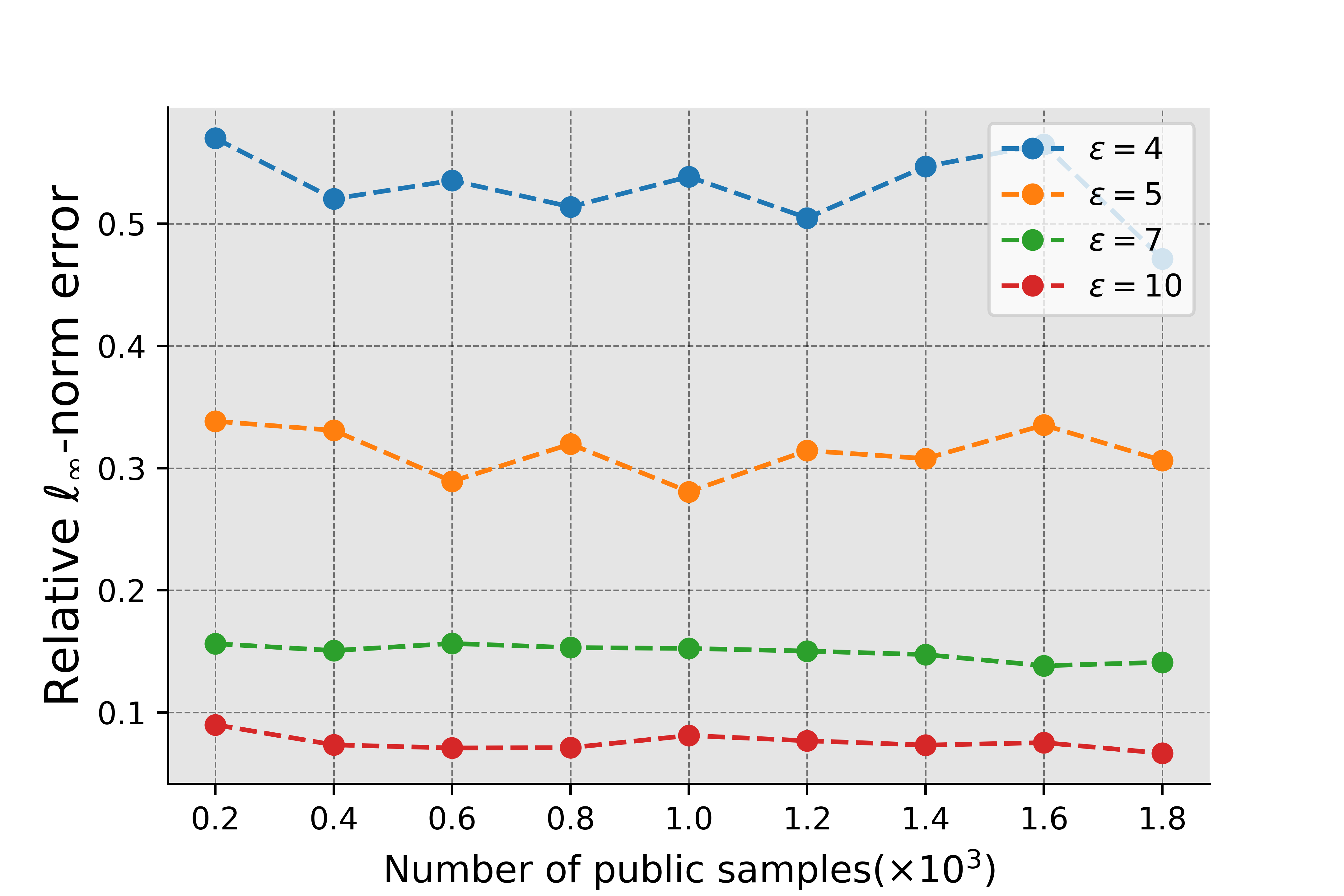}}
\subfigure[  $n= 10^5$ and $p=40$ \label{fig:alg4_cubic_n10w_p40}]
{\includegraphics[width=0.32\textwidth, height=0.16\textheight]{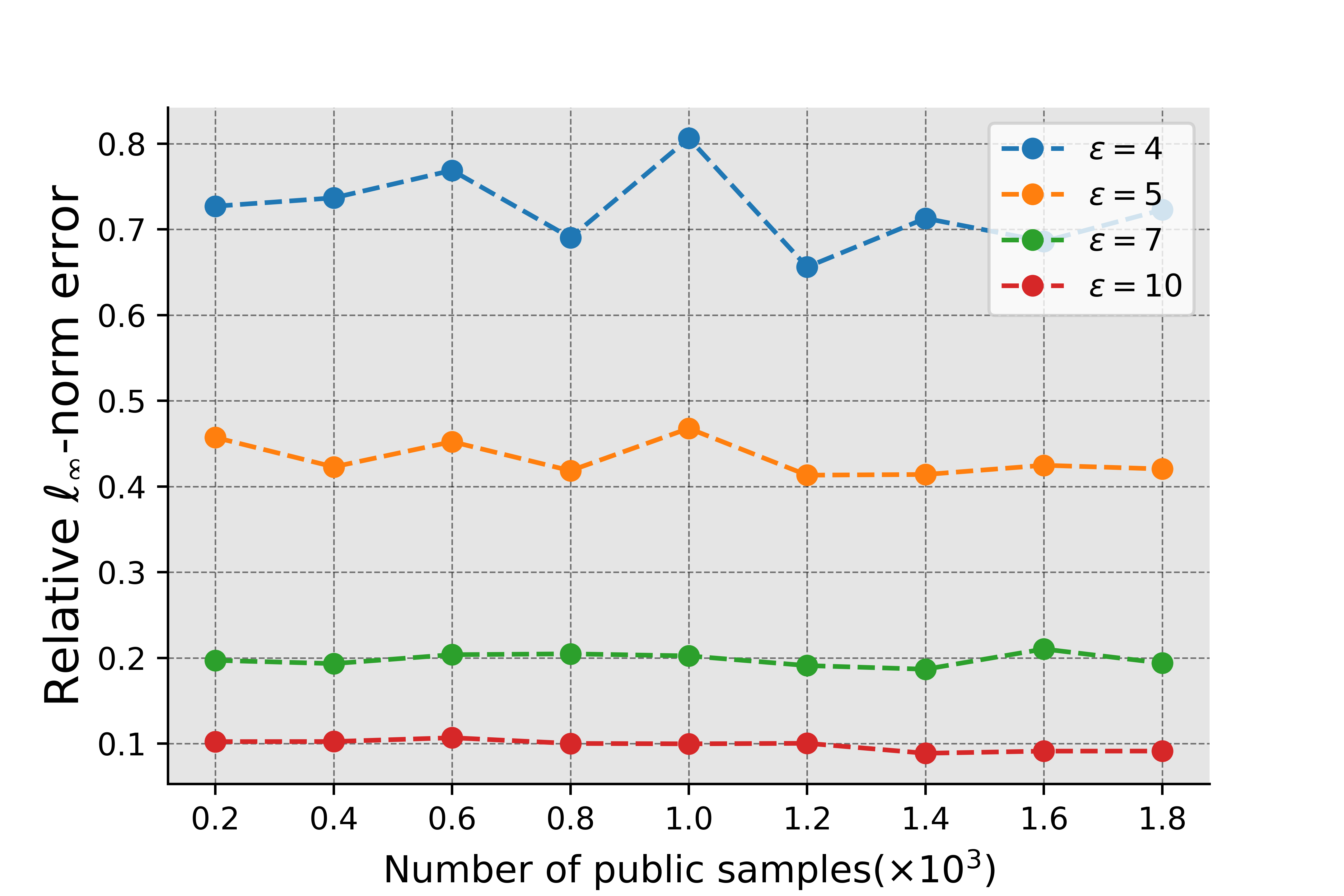}}
    \caption{  Algorithm \ref{alg:2} for cubic link function with Bernoulli data. The left plot shows the relative error with different dimension $p$. The middle and the right plots show the relative error with different size of public data $m$  when $n$ and $p$ are fixed. \label{fig:alg4_cubic_public_whole}
    }
\end{figure*}

\begin{figure*}[!ht]
\centering
\subfigure[ $p=10$ \label{fig:alg4_logistic_p10}]
{\includegraphics[width=0.32\textwidth, height=0.18\textheight]{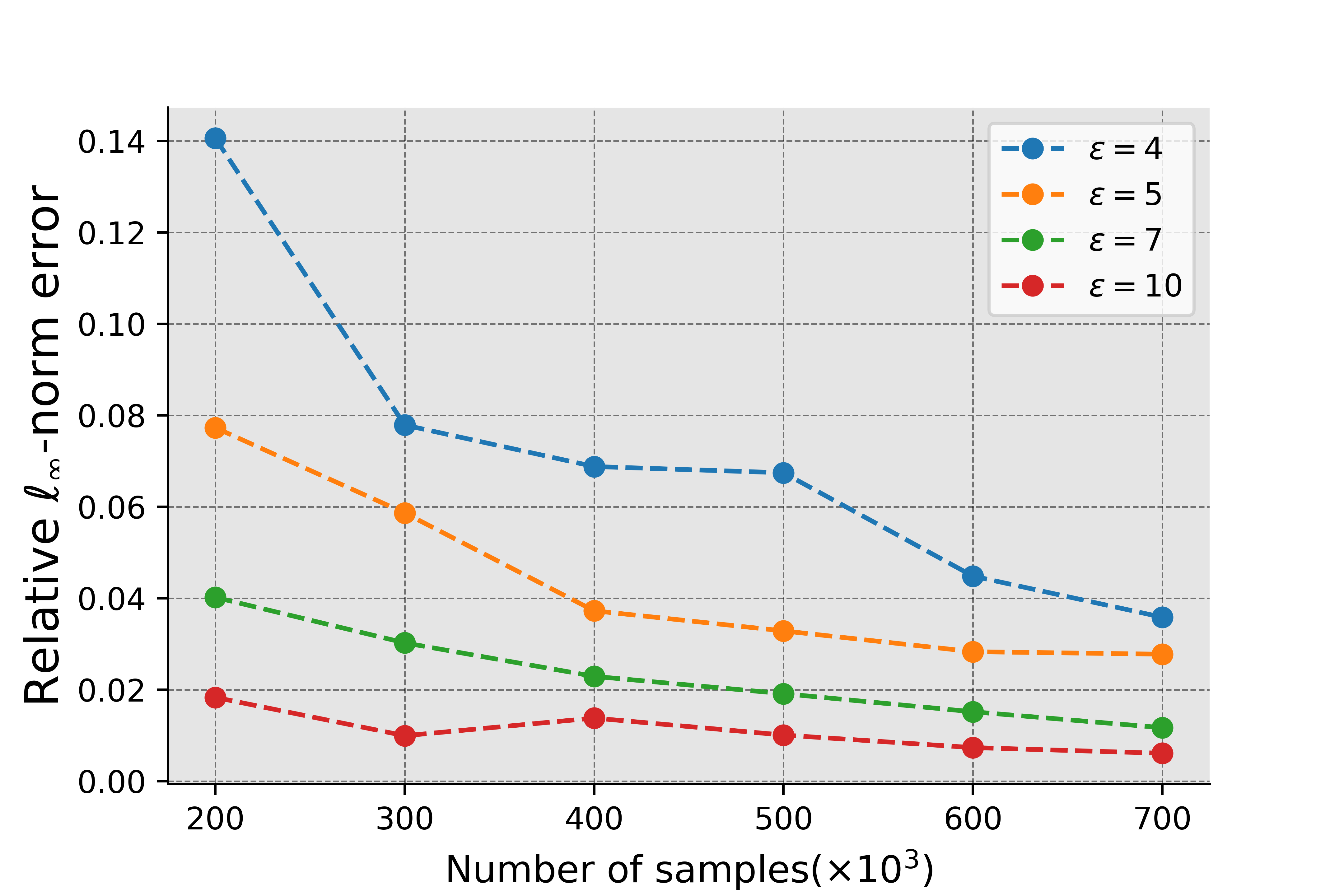}}
\subfigure[ $p=20$  \label{alg4_logistic_p20}]
{\includegraphics[width=0.32\textwidth, height=0.18\textheight]{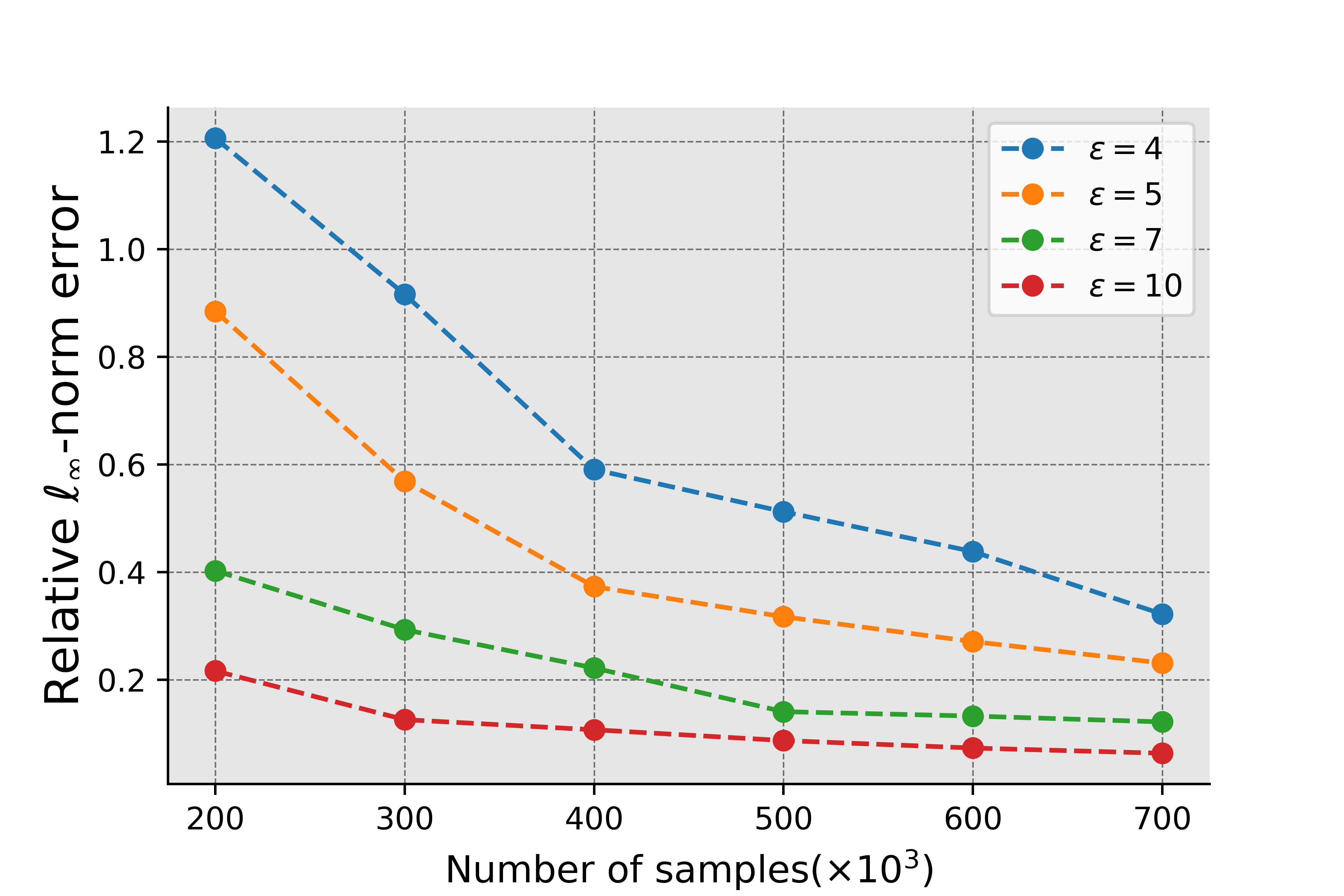}}
\subfigure[ $p=30$  \label{fig:alg4_logistic_p30}]
{\includegraphics[width=0.32\textwidth, height=0.18\textheight]{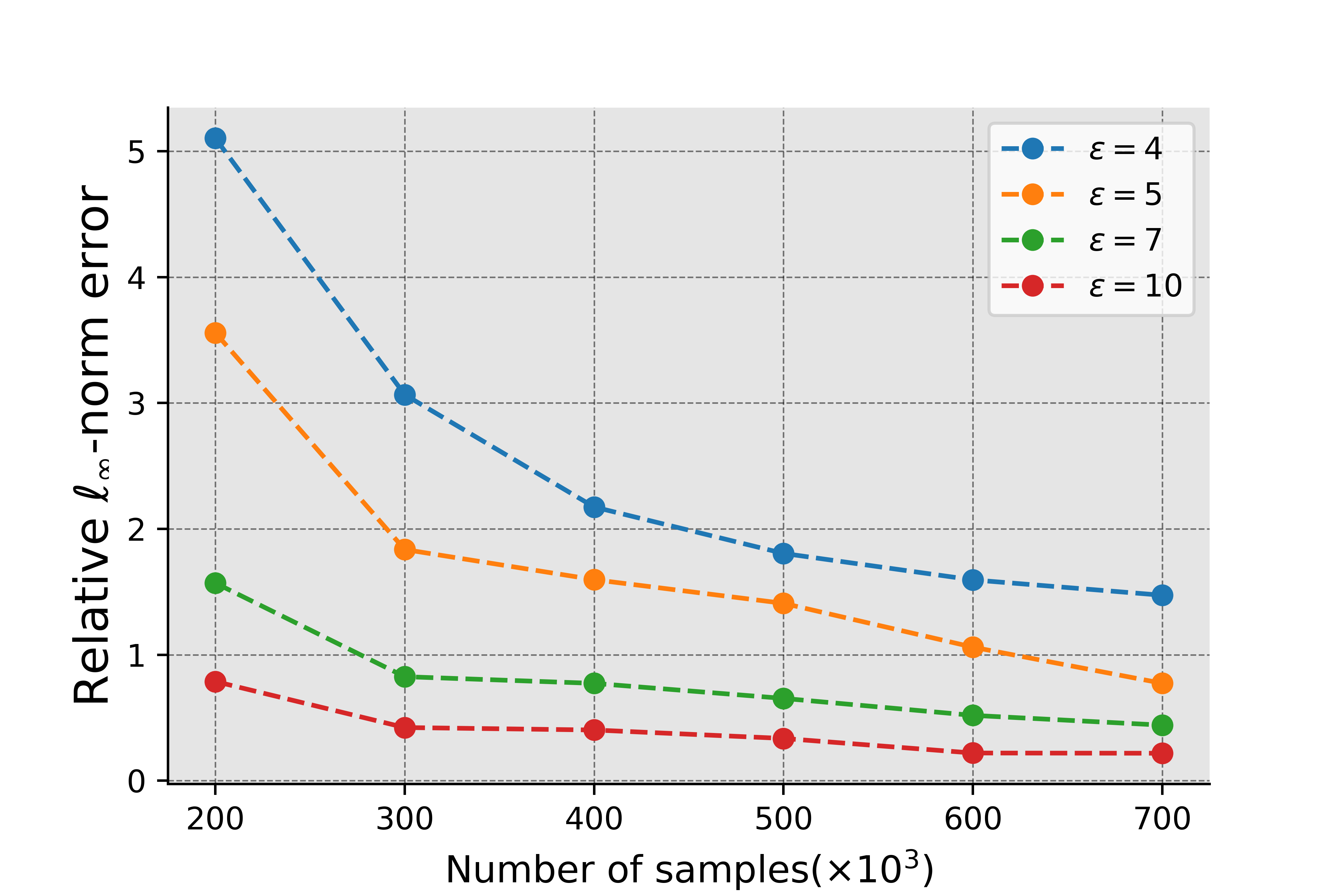}}
    \caption{  Algorithm \ref{alg:2} for logistic link function with Bernoulli data under different dimension $p$.  \label{fig:alg4_logistic_epsilon_whole}
    }
\end{figure*}

\begin{figure*}[!ht]
\centering
\subfigure[$\epsilon=7$ \label{fig:alg4_logistic_e7}]
{\includegraphics[width=0.32\textwidth,  height=0.18\textheight]{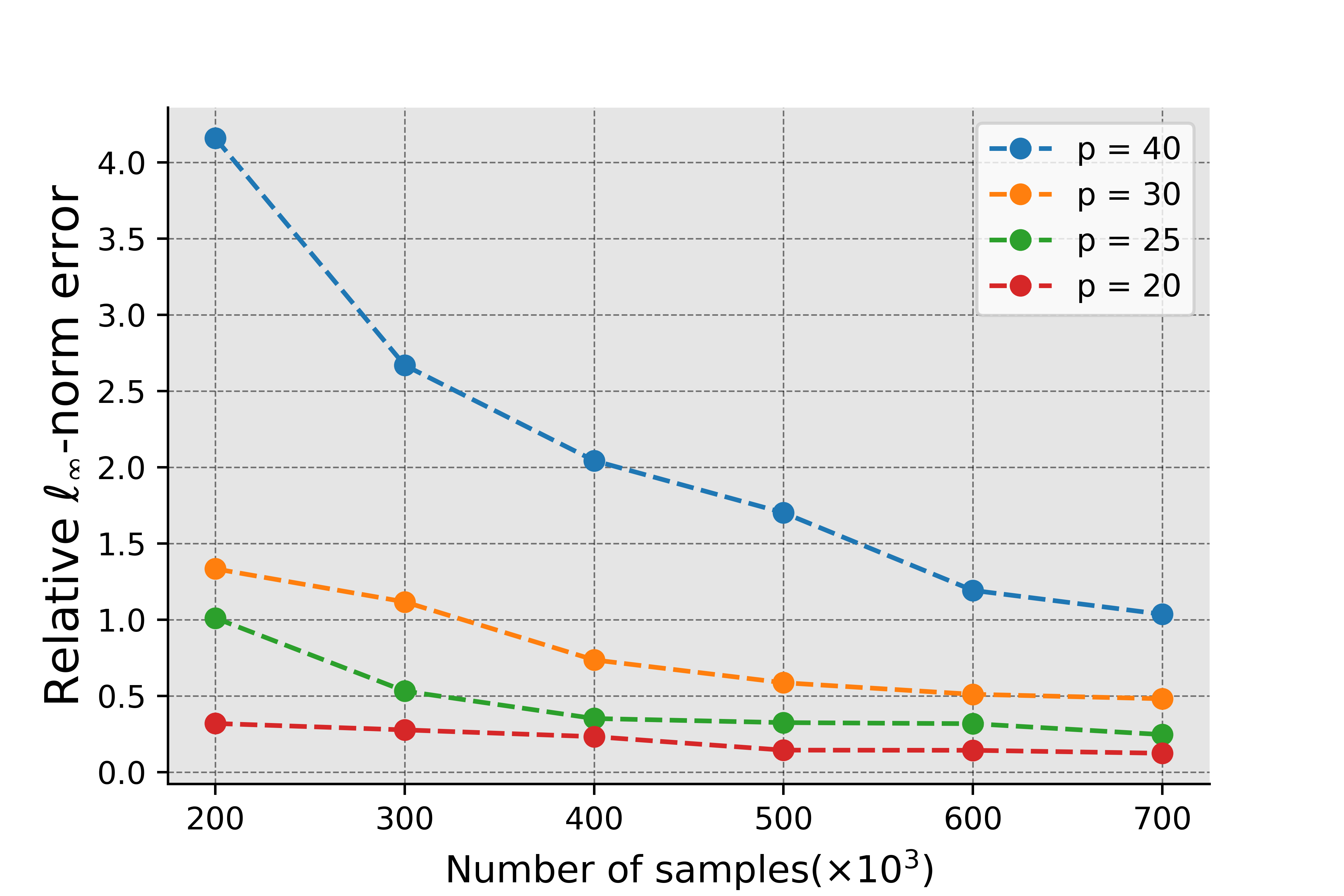}}
\subfigure[ $n=5\times 10^4$ and $p=20$  \label{fig:alg4_logistic_n5w_p20}]
{\includegraphics[width=0.32\textwidth, height=0.18\textheight]{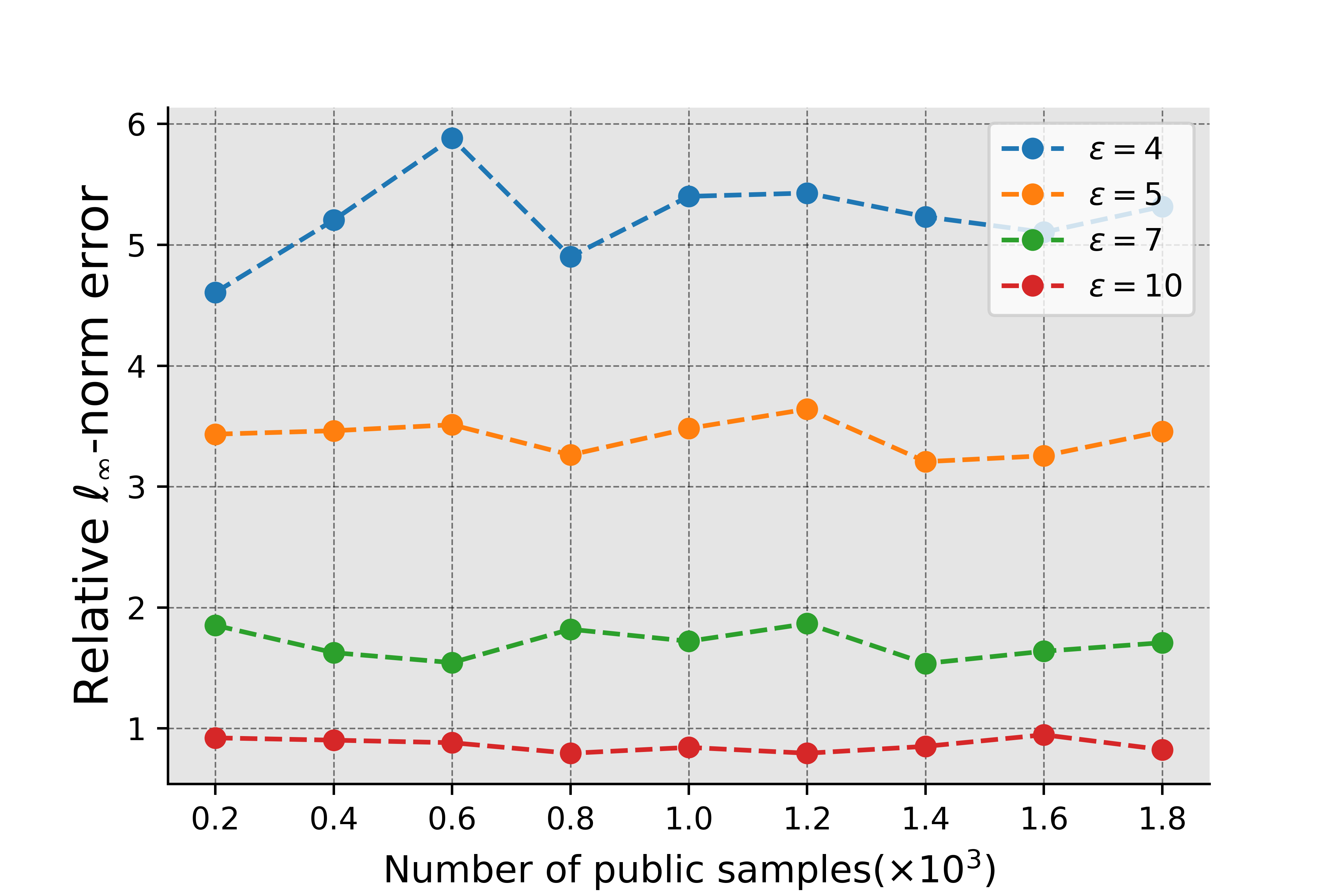}}
\subfigure[  $n=4\times 10^5$ and $p=40$ \label{fig:alg4_logistic_n40w_p20_2}]
{\includegraphics[width=0.32\textwidth, height=0.18\textheight]{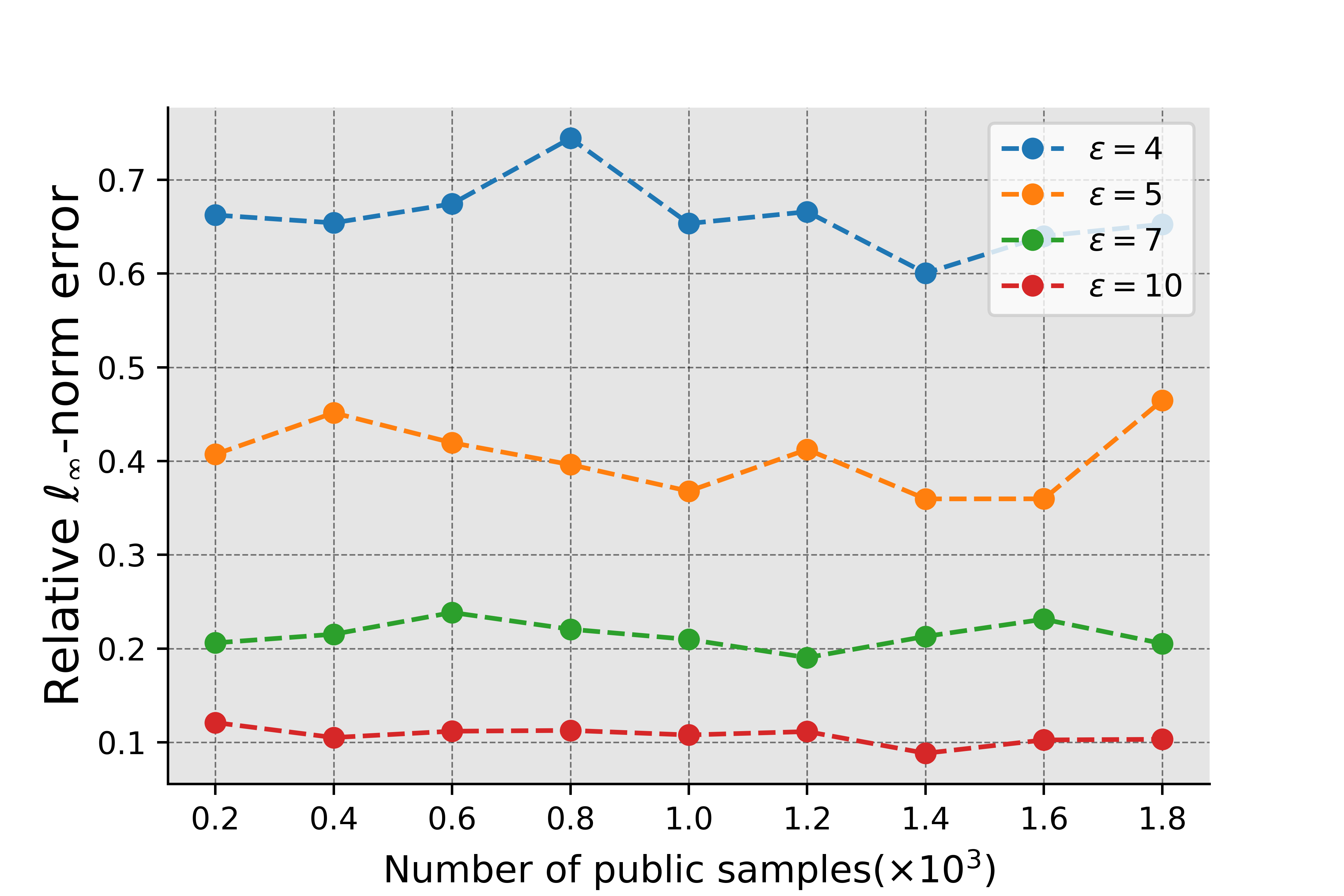}}
    \caption{  Algorithm \ref{alg:2} for logistic link function with Bernoulli data. The left plot shows the relative error with different dimension $p$. The middle and the right plots show the relative error with different size of public data $m$ when $n$ and $p$ are fixed. \label{fig:alg4_logistic_public_whole}
    }
\end{figure*}

\begin{figure*}[!ht]
\centering
\subfigure[Covertype \label{fig:Covtype_public_1w}]
{\includegraphics[width=0.32\textwidth,  height=0.18\textheight]{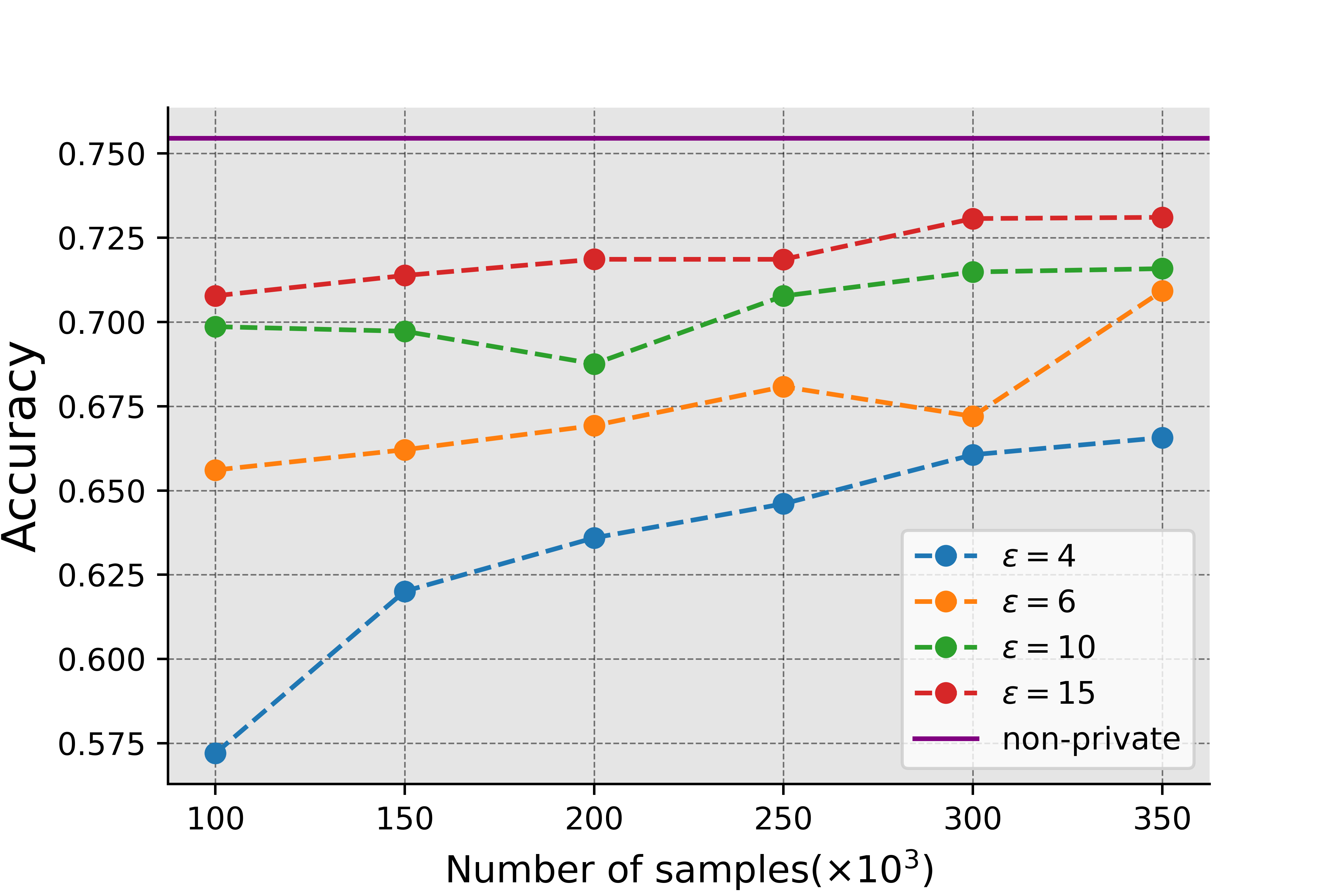}}
\subfigure[ SUSY  \label{fig:SUSY_public_1w}]
{\includegraphics[width=0.32\textwidth, height=0.18\textheight]{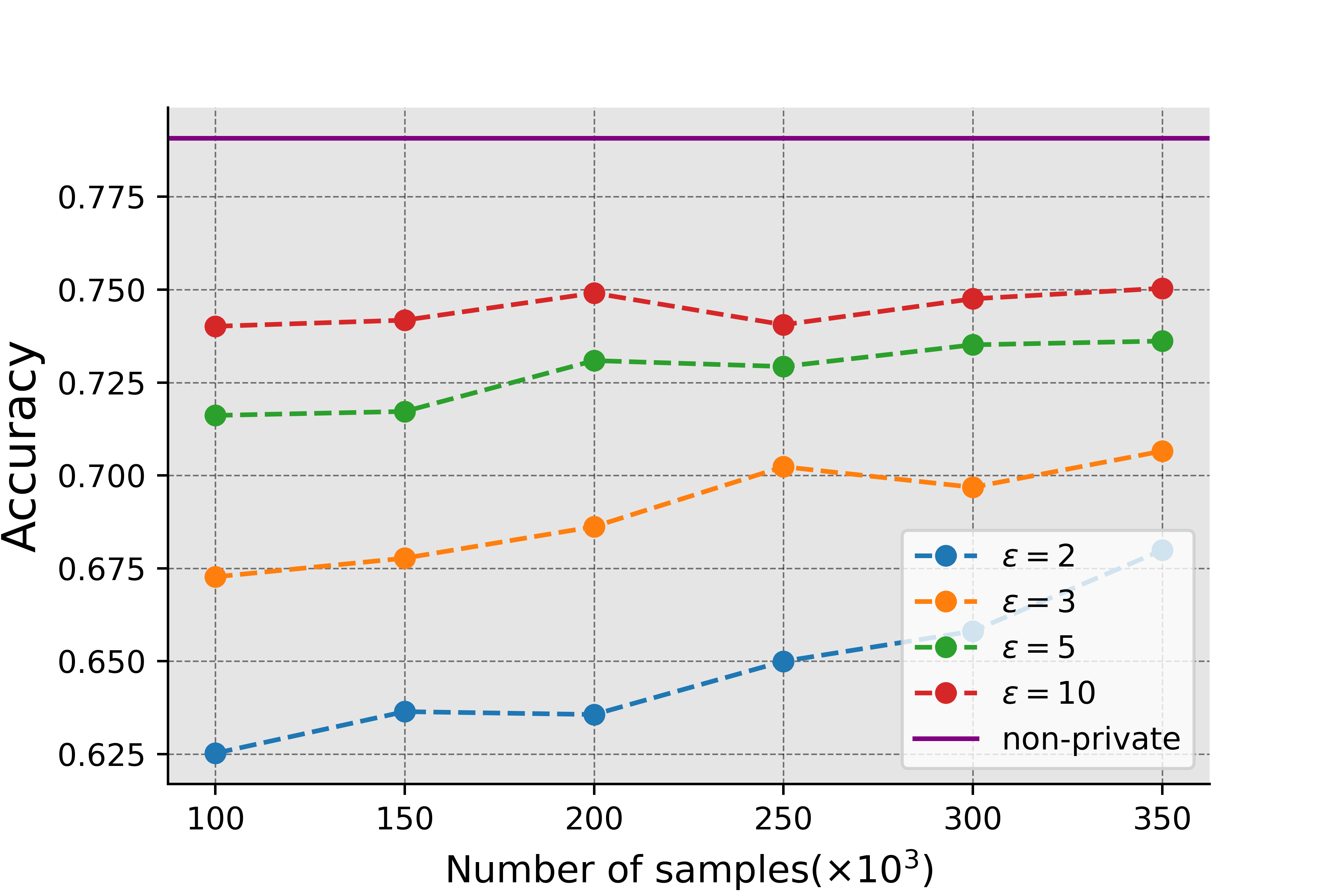}}
\subfigure[  Skin Segmentation \label{fig:Skin_public_5k}]
{\includegraphics[width=0.32\textwidth, height=0.18\textheight]{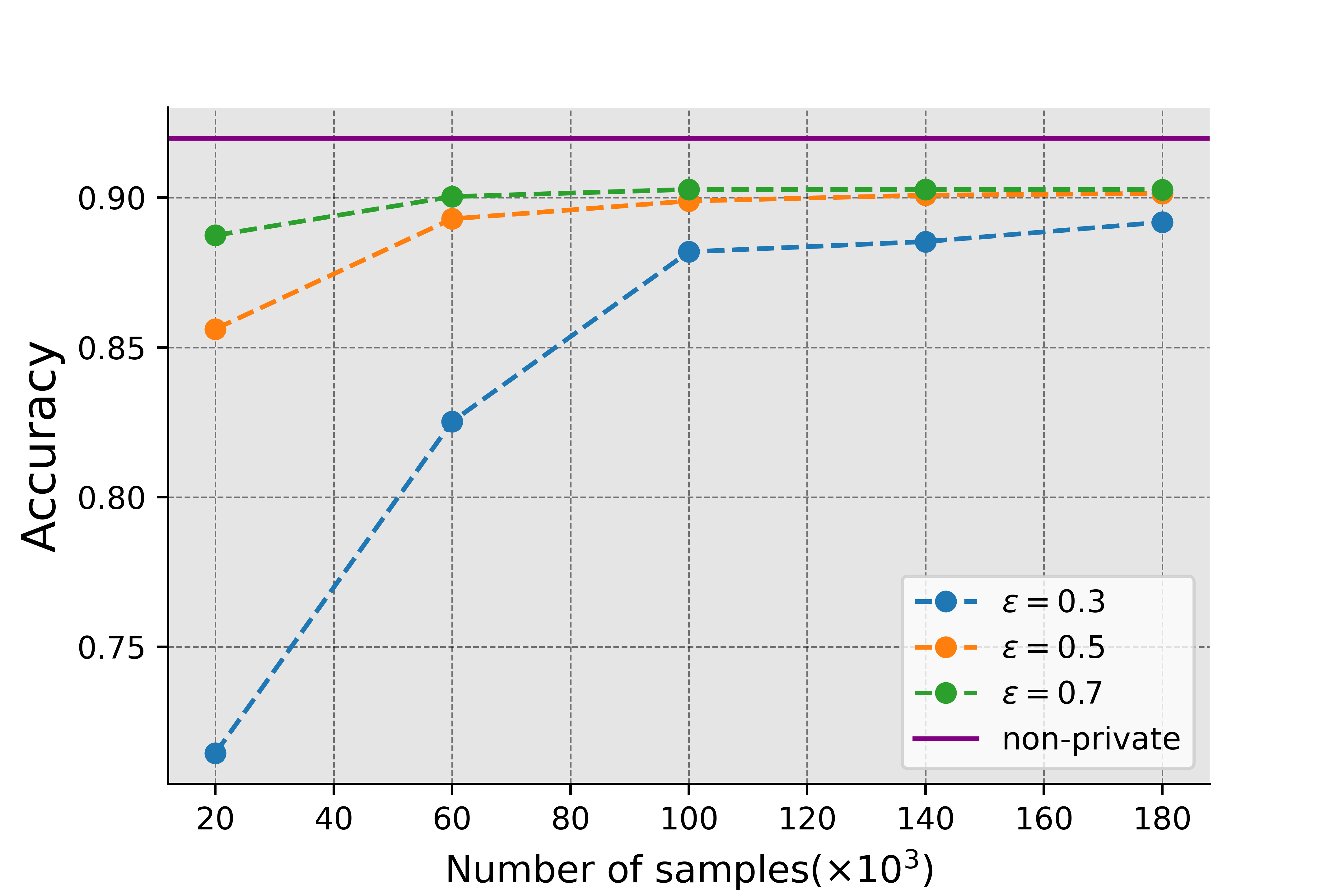}}
    \caption{  Algorithm \ref{alg:1} for logistic  regression on different real data.  \label{fig:alg2_log_publicdata_whole}
    }
\end{figure*}
\begin{figure*}[!ht]
\centering
\subfigure[ Covertype \label{fig:Covtype_real}]
{\includegraphics[width=0.32\textwidth, height=0.18\textheight]{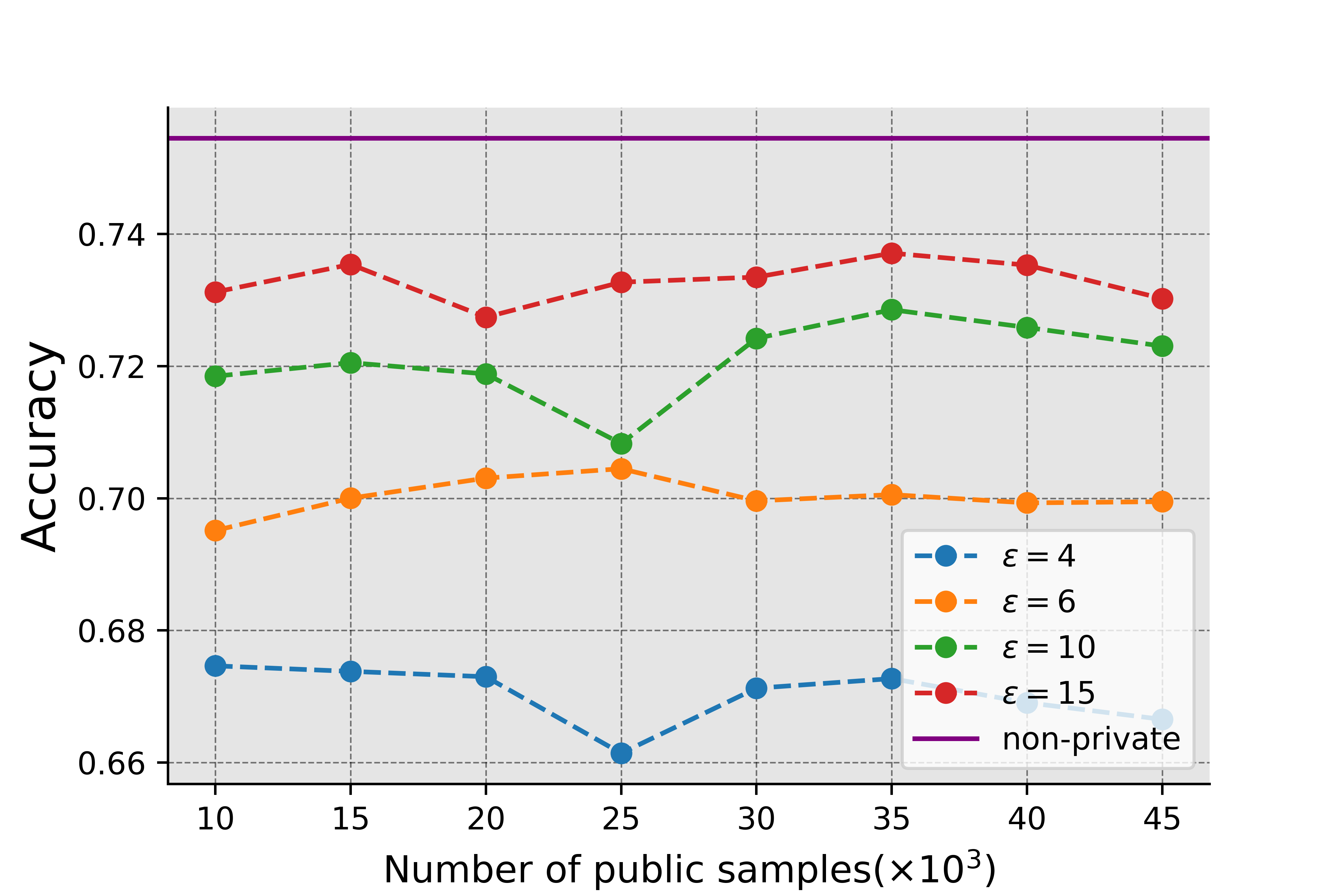}}
\subfigure[ SUSY  \label{fig:SUSY_real_2}]
{\includegraphics[width=0.32\textwidth, height=0.18\textheight]{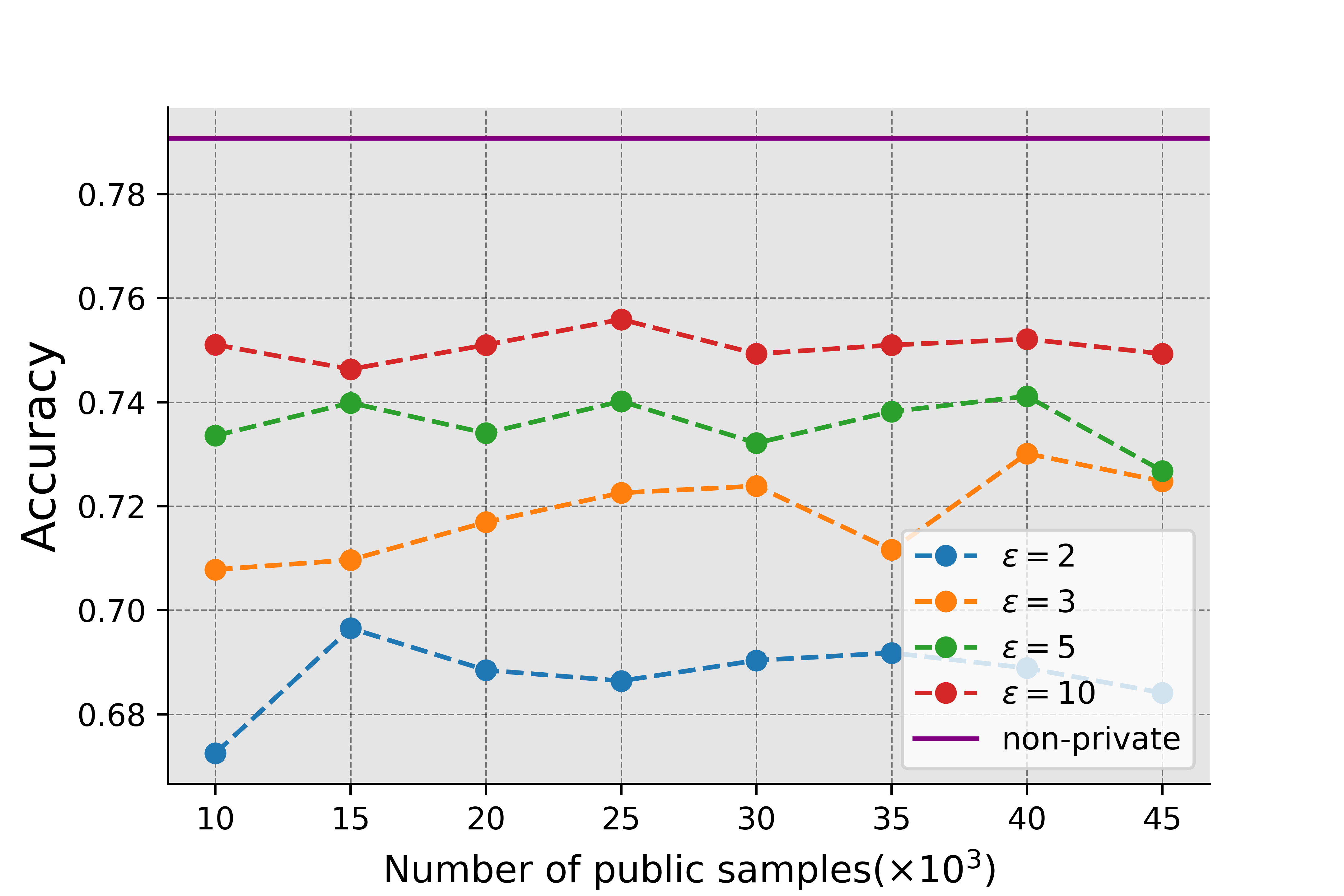}}
\subfigure[ Skin Segmentation   \label{fig:SUSY_real}]
{\includegraphics[width=0.32\textwidth, height=0.18\textheight]{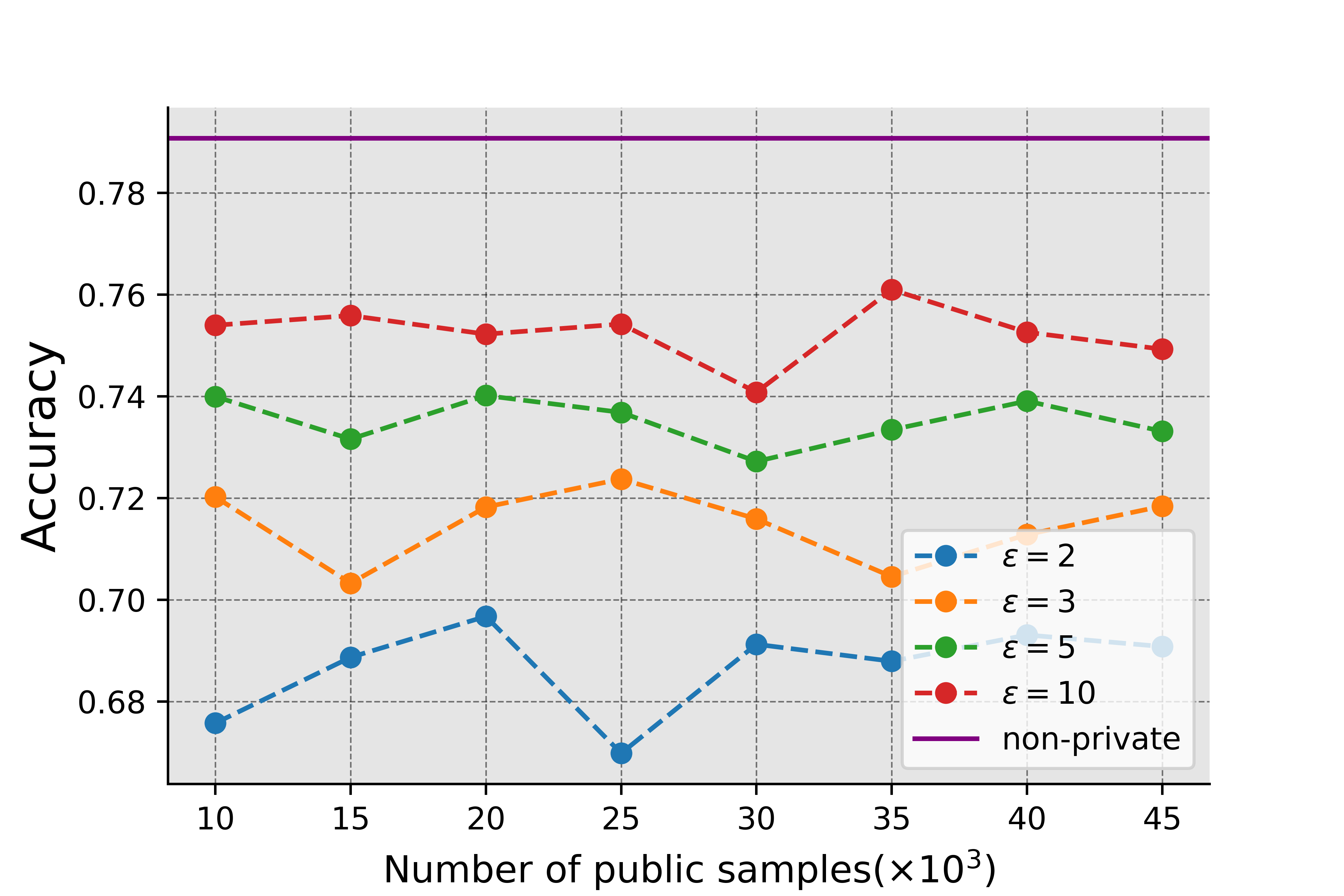}}
    \caption{  Algorithm \ref{alg:1} for logistic regression on different real data with different size of public data $m$.  \label{fig:alg2_log_public_pulic_whole}
    }
\end{figure*}

\begin{figure*}[!ht]
\centering
\subfigure[ Algorithm \ref{alg:0} for logistic regression where  $p=30$ and the Gaussian distribution has non-diagonal covariance matrix. \label{fig:alg1 vs alg5_logistic_org_p30_2}]
{\includegraphics[width=0.45\textwidth, height=0.20\textheight]{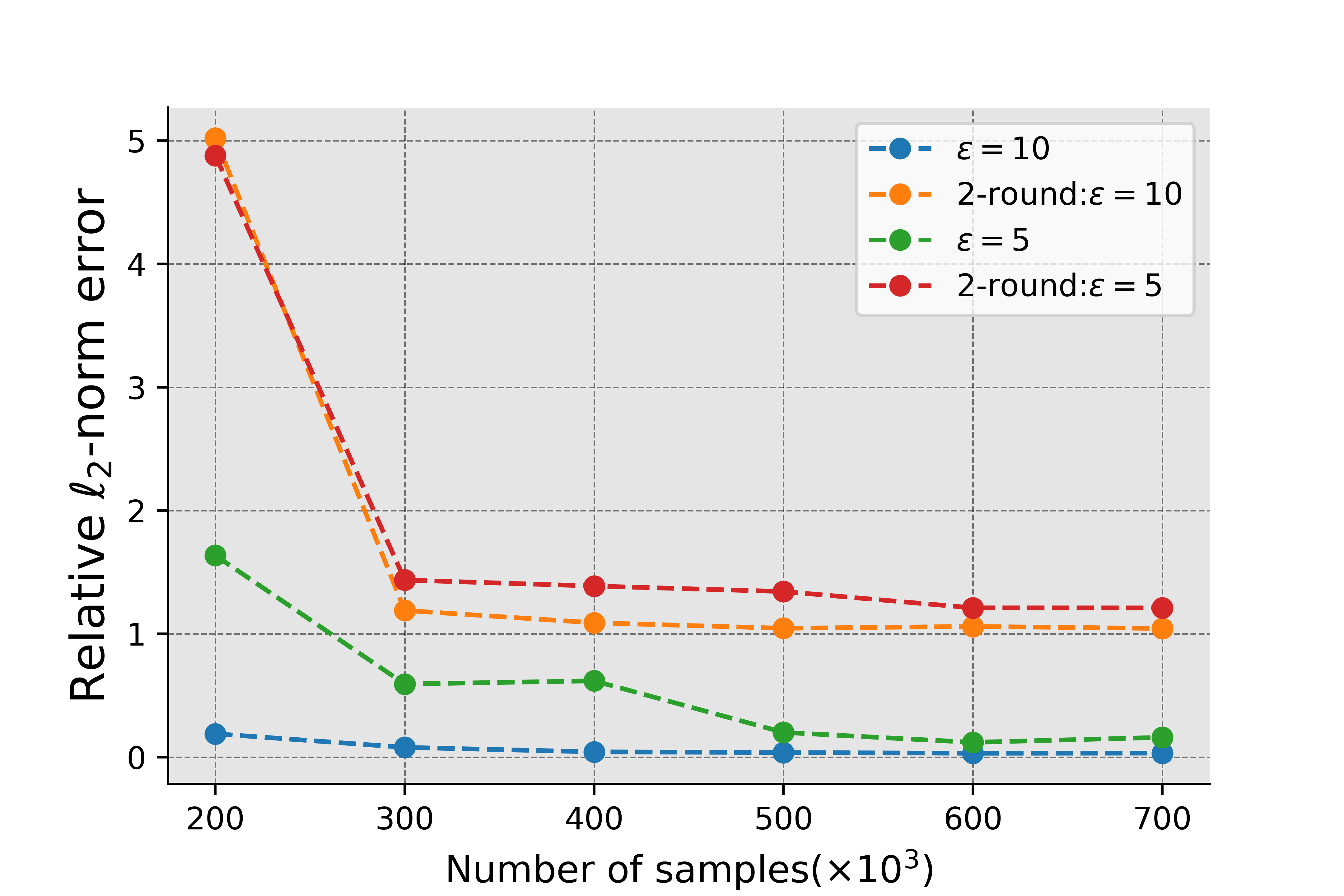}}
\subfigure[ Algorithm \ref{alg:1} for  logistic regression with $p=30$ for Bernoulli data.  \label{fig:alg2 vs alg5_logistic_p30}]
{\includegraphics[width=0.45\textwidth, height=0.20\textheight]{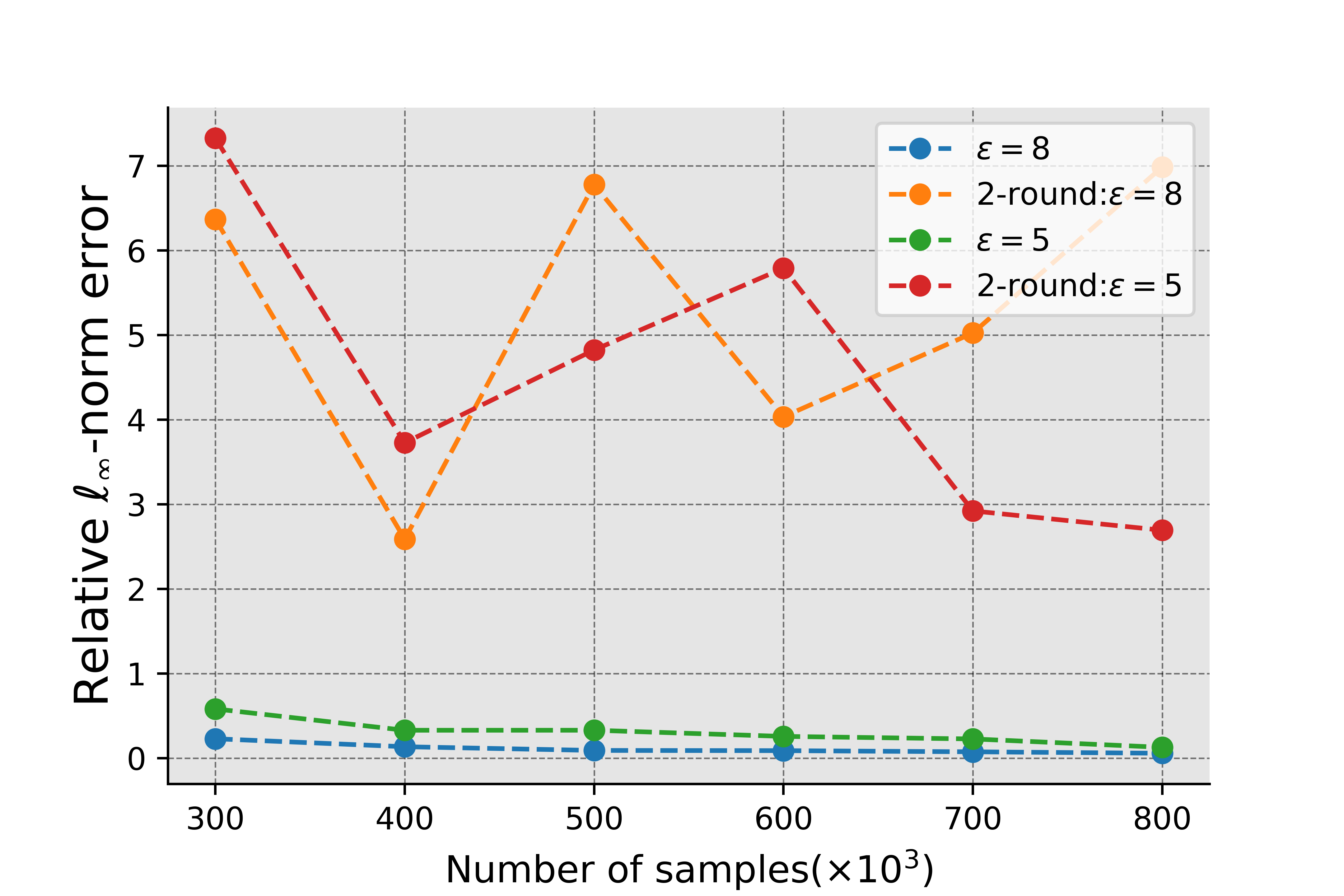}}
\subfigure[ Algorithm \ref{alg:1.5} for sigmoid link function where  $p=20$ and the Gaussian distribution has non-diagonal covariance matrix.   \label{fig:alg3 vs alg5_sigmoid_org_p20}]
{\includegraphics[width=0.45\textwidth, height=0.20\textheight]{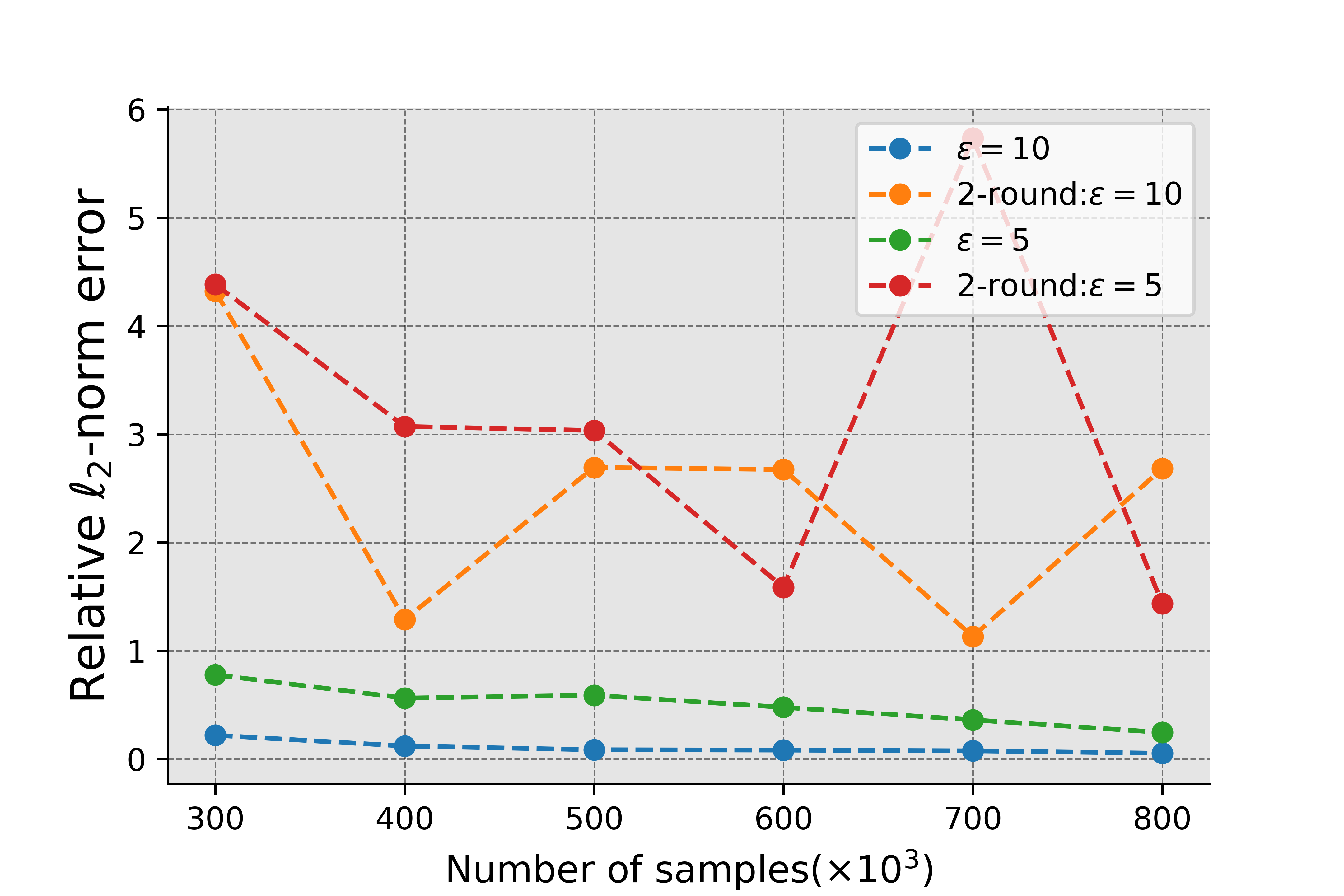}}
\subfigure[ Algorithm \ref{alg:2} for  cubic  link function with $p=30$ for Bernoulli data.     \label{fig:alg4 vs alg5_cubic_p20}]
{\includegraphics[width=0.45\textwidth, height=0.20\textheight]{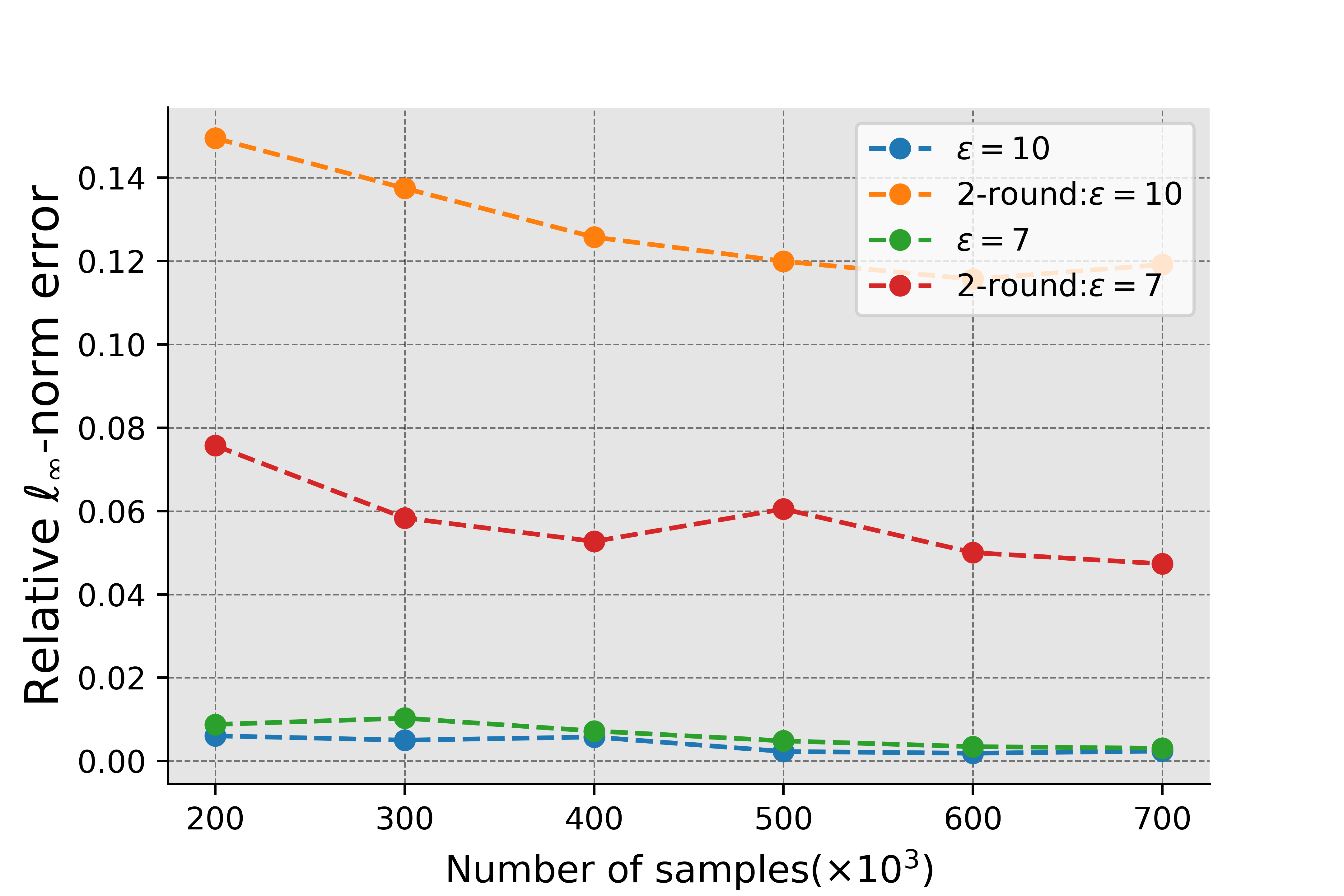}}

    \caption{Comparison of our methods with their corresponding 2-round LDP algorithms.  \label{fig:alg5_compare}
    }
\end{figure*}

\bibliography{nips}
\appendix
\section{Background and Auxiliary Lemmas}
\paragraph{Notations} For a positive semi-definite matrix $M\in \mathbb{R}^{p\times p}$, we define the $M$-norm for a vector $w$ as $\|w\|_M^2= w^TMw$. $\lambda_{\min}(A)$ is the minimal singular value of the matrix $A$.
    \begin{lemma}[Non-communicative Matrix Bernstein inequality \citep{vershynin2010introduction}]\label{bernstein}
   Consider a finite
sequence  $X_i$ of independent centered symmetric random $p\times p$ matrices.  Assume we have for some numbers $K$ and $\sigma$ that
\begin{equation*}
    \|X_i\|_2\leq K, \|\sum_i \mathbb{E}[X_i^2]\|_2\leq \sigma^2.
\end{equation*}
Then, for every $t\geq 0$ we have
\begin{equation*}
    \text{Pr}(\|\sum_{i}X_i\|_2\geq t) \leq 2p\exp(-\frac{t^2/2}{\sigma^2+K t/3}). 
\end{equation*}

    \end{lemma}
    \begin{lemma}[Hoeffding type inequality for norm-subGaussian \citep{jin2019short}]\label{norm-sub}
If the random vectors $X_i \in \mathbb{R}^p$ satisfy
\begin{equation*}
    \text{Pr}(\|X_i-\mathbb{E}X\|_2\geq t)\leq \exp(-\frac{t^2}{2\sigma^2})
\end{equation*}
for $i=1, 2, \cdots n$ with some $\sigma $ and any $t>0$. Then there exists an absolute constant $c$ 
such that with probability
at least $1-\delta$ for any $\delta>0$:
\begin{equation*}
    \|\sum_{i=1}^n X_i\|_2\leq c\sqrt{n\sigma^2 \log \frac{2d}{\delta}}.
\end{equation*}
\end{lemma}
\begin{lemma}[Weyl's Inequality \citep{Stewart90matrixperturbation}]\label{lemma:a5.2}
Let $X, Y\in \mathbb{R}^{p\times p}$ be two symmetric matrices, and  $E=X-Y$. Then, for all $i=1, \cdots, p$, we have 
\begin{equation*}
    |\sigma_i(X)-\sigma_i(Y)|\leq \|E\|_2, 
\end{equation*}
where $\sigma_i(M)$ is the $i$-th eigenvalue of the matrix $M$. 
\end{lemma}
\begin{lemma}\label{lemma:a5.3}
Let $w\in \mathbb{R}^p$ be a fixed vector and $E$ be a symmetric Gaussian random matrix where the upper triangle entries are i.i.d Gaussian distribution $\mathcal{N}(0, \sigma^2)$. Then, with probability at least $1-\xi$, the following holds for a fixed positive semi-definite matrix $M\in \mathbb{R}^{p\times p} $ 
\begin{equation*}
    \|Ew\|_M^2\leq \sigma^2 \text{Tr}(M)\|w\|^2\log \frac{2p^2}{\xi}.
\end{equation*}

\end{lemma}
\begin{proof}[Proof of Lemma \ref{lemma:a5.3}]
Let $M=U^T\Sigma U$ denote the eigenvalue decomposition of $M$. Then, we have 
\begin{equation*}
    \|Ew\|^2_{M}= w^TE^TU^T\Sigma UEw=\sum_{i=1}^p \sigma_i \sum_{j=1}^p [UE]_{ij}^2w_i^2.
\end{equation*}
Note that $[UE]_{i,j}=\sum_{k=1}^p U_{i,k}E_{j,k}$ where $E_{i,j}$ is Gaussian. Since $U$ is orthogonal, we know that $[UE]_{i,j}\sim \mathcal{N}(0, \sigma^2)$. Using the Gaussian tail bound for all $i,j\in [d]^2$, we have
\begin{equation*}
    \mathbb{P}(\max_{i,j\in [p]^2 }|[UE]_{i,j}|\geq \sqrt{\sigma^2\log\frac{2p^2}{\xi}})\leq \xi. 
\end{equation*}
\end{proof}
\begin{lemma}[Theorem 4.7.1 in \citep{vershynin2018high} ]\label{lemma:a5.4}
Let $x$ be a random vector in $\mathbb{R}^p$ that is sub-Gaussian  with covariance matrix $\Sigma$ and $\|\Sigma^{-\frac{1}{2}}x\|_{\psi_2}\leq \kappa_x$. Then, with probability at least $1-\exp(-p)$, the empirical covariance matrix $\frac{1}{n}X^TX=\frac{1}{n}\sum_{i=1}^n x_ix_i^T$ satisfies 
\begin{equation*}
    \| \frac{1}{n}X^TX- \Sigma \|_2 \leq C\kappa_x^2\sqrt{\frac{p}{n}}\|\Sigma\|_2.
\end{equation*}
\end{lemma}

\begin{lemma}[Corollary 2.3.6 in \citep{tao2011topics}]\label{lemma:a5.5}
Let $M\in\mathbb{R}^{p\times p}$ be a symmetric matrix whose entries 
$m_{ij}$ 
are independent for $j>i$, have mean zero, and are uniformly bounded in magnitude by 1. 
Then, there exists absolute constants $C_2, c_1>0$ such that with probability at least $1-\exp(-C_2c_1 p)$, the following inequality holds $\|M\|_2\leq C\sqrt{p}$. 
\end{lemma}

Below we introduce some concentration lemmas given in \citep{erdogdu2019scalable}.

\begin{lemma}\label{lemma:a5.6}
Let $\mathbb{B}^\delta(\tilde{w})$ denote the ball centered at $\tilde{w}$ and with radius $\delta$ ({\em i.e.,} $\mathbb{B}^\delta(\tilde{w})=\{w: \|w-\tilde{w}\|_2\leq \delta\}$). For $i=1, 2 \cdots, n$, let $x_i\in \mathbb{R}^p$ be i.i.d  isotropic sub-Gaussian random vectors with $\|x_i\|_{\psi_2}\leq k_x$, and  $\tilde{\mu}= \frac{\mathbb{E}[\|x\|_2]}{\sqrt{p}}$. For any given  function $g: \mathbb{R}\mapsto \mathbb{R}$ that is Lipschitz continuous with $G$ and satisfies $\sup_{w\in \mathbb{B}^\delta(\tilde{w})}\|g(\langle x, w\rangle)\|_{\psi_2}\leq \kappa_{g}$,  with probability at least $1-2\exp(-p)$, the following holds for $np>51\max\{\chi, \chi^2\}$
\begin{equation*}
    \sup _{w\in \mathbb{B}^\delta(\tilde{w})}|\frac{1}{m}\sum_{i=1}^m g(\langle x_i, w\rangle)- \mathbb{E}[g(\langle x, w \rangle)]|\leq c(\kappa_g+\frac{\kappa_x}{\tilde{u}})\sqrt{\frac{p\log m}{m}}, 
\end{equation*}
where $\chi= \frac{(\kappa_g+\frac{\kappa_x}{\tilde{\mu}})^2}{c\delta^2G^2\tilde{\mu}^2}$. $c$ is some absolute constant. 
\end{lemma}

\begin{lemma}\label{lemma:a5.7}
Let $\mathbb{B}^\delta(\tilde{w})$ be the ball centered at $\tilde{w}$ and with radius $\delta$ ({\em i.e.,} $\mathbb{B}^\delta(\tilde{w})=\{w: \|w-\tilde{w}\|_2\leq \delta\}$).
For $i=1, 2 \cdots, n$, let $x_i\in \mathbb{R}^p$ be i.i.d sub-Gaussian random vectors with covariance matrix $\Sigma$. For any given function $g:\mathbb{R}\mapsto \mathbb{R}$ that is uniformly bounded by $L$ and Lipschitz continuous with $G$, the following holds with probability at least $1-\exp(-p)$ 
\begin{equation*}
      \sup _{w\in \mathbb{B}^\delta(\tilde{w})}|\frac{1}{m}\sum_{i=1}^m g(\langle x_i, w\rangle)- \mathbb{E}[g(\langle x, w \rangle)]| \leq 2\{G(\|\tilde{w}\|_2+\delta)\|\Sigma\|_2+L\}\sqrt{\frac{p}{m}}. 
\end{equation*}
\end{lemma}

The following lemma shows that the private estimator $\hat{w}^{ols}$ is close to the unperturbed one. 

\begin{lemma}\label{lemma:a5.8}
Let $X=[x_1^T; x_2^T; \cdots; x_n^T] \in \mathbb{R}^{n\times  d}$ be a matrix such that 
$X^TX$ is invertible, and $x_1, \cdots, x_n$ are realizations of a sub-Gaussian random variable $x$ whose $\ell_2$ norm is bounded by $r$. Moreover if $x$ satisfies the condition of $\|\Sigma^{-\frac{1}{2}}x\|_{\psi_2}\leq \kappa_x=O(1)$ and $\Sigma=\mathbb{E}[xx^T]$ is the the population covariance matrix. 
Let  $\tilde{w}^{ols}= (X^TX)^{-1}X^Ty$ denote the empirical linear regression estimator.  Then, for sufficiently large $n\geq \Omega(\frac{\kappa_x^4\|\Sigma\|_2^2pr^4\log \frac{1}{\delta}}{\epsilon^2\lambda^2_{\min}(\Sigma)})$, the following holds  with probability at least $1-\exp(-\Omega(p))-\xi$, 
\begin{equation}\label{aeq:5.1}
    \|\hat{w}^{ols}- \tilde{w}^{ols}\|_2^2 = O\big( \frac{p r^2(1+r^2\|\tilde{w}^{ols}\|_2^2)\log \frac{1}{\delta}\log \frac{p^2}{\xi} }{\epsilon^2n\lambda^2_{\min}(\Sigma)}\big),  
\end{equation}
where $\|x_i\|_2\leq r$ is sampled from some bounded distribution. 
\end{lemma}
\begin{proof}[Proof of Lemma \ref{lemma:a5.8}]
 It is obvious that $\widehat{X^TX}= X^TX+E_1$, where $E_1$ is a symmetric Gaussian matrix with each entry sampled from $\mathcal{N}(0, \sigma_1^2)$ and $\sigma_1^2=O(\frac{nr^4\log \frac{1}{\delta}}{\epsilon^2})$. $\widehat{X^Ty}=X^Ty+E_2$, where $E_2$ is a Gaussian vector sampled from $\mathcal{N}(0, \sigma_2^2I_p)$ and $\sigma_2^2=O(\frac{nr^2\log\frac{1}{\delta}}{\epsilon^2})$.  
 
 We first show that $\widehat{X^TX}$ is invertible with high probability under our assumption.
 
 It is sufficient to show that $X^TX+E_1\succ \frac{X^TX}{2}$, {\em i.e.,}  $\|E_1\|_2\leq \frac{\lambda_{\min}(X^TX)}{2}$. By Lemma \ref{lemma:a5.5}, we can see that with probability $1-\exp(-\Omega(p))$,  $$\|E_1\|_2\leq O(\frac{r^2\sqrt{pn\log \frac{1}{\delta}}}{\epsilon}).$$ 
 Also, by Lemma \ref{lemma:a5.4} and Lemma  \ref{lemma:a5.2} we know that with probability at least $1-\exp(-\Omega(p))$, $$\lambda_{\min} (X^TX)\geq n\lambda_{\min}(\Sigma)-O(\kappa_x^2\|\Sigma\|_2\sqrt{pn}).$$ Thus, it is sufficient to show that $n\lambda_{\min}(\Sigma)\geq O(\frac{\kappa_x^2\|\Sigma\|_2 r^2\sqrt{pn\log \frac{1}{\delta}}}{\epsilon}) $, which is true under the assumption of $n\geq \Omega(\frac{\kappa_x^4\|\Sigma\|_2^2pr^4\log \frac{1}{\delta}}{\epsilon^2\lambda^2_{\min}(\Sigma)})$. Thus, with probability at least $1-\exp(-\Omega(p))$, it is invertible. In the following we will always assume that this event holds. 
 
 By direct calculation we have 
 \begin{equation*}
      \|\hat{w}^{ols}- \tilde{w}^{ols}\|_2= -(X^TX+E_1)^{-1} E_1 \tilde{w}^{ols}+ (X^TX+E_1)^{-1}E_2. 
 \end{equation*}
 Thus, by Cauchy-Schwartz inequality we get
 \begin{equation*}
     \|\hat{w}^{ols}- \tilde{w}^{ols}\|^2_2= O\big( \|E_1\tilde{w}^{ols}\|_{(X^TX+E_1)^{-2}}^2 + \|E_2\|_{(X^TX+E_1)^{-2}}^2   \big).
 \end{equation*}
 Since we already assume that $X^TX+E_1 \succ \frac{X^TX}{2} $, by Lemma \ref{lemma:a5.3} we can obtain the following with probability at least $1-\xi$
 \begin{align*}
     &\|E_1\tilde{w}^{ols}\|_{(X^TX+E_1)^{-2}}^2 \leq O\big( \frac{nr^4\log \frac{1}{\delta}}{\epsilon^2}\|\tilde{w}^{ols}\|_2^2\text{Tr}((X^TX)^{-2})\log \frac{4p^2}{\xi}\big) \\ 
     &\|E_2\|_{(X^TX+E_1)^{-2}}^2\leq O\big(\frac{nr^2\log\frac{1}{\delta}}{\epsilon^2} \text{Tr}((X^TX)^{-2})\frac{4p}{\xi}  \big). 
 \end{align*}
 Thus, we have
 \begin{equation*}
    \|\hat{w}^{ols}- \tilde{w}^{ols}\|_2^2 \leq C_1 n\cdot \frac{r^2(1+r^2\|\tilde{w}^{ols}\|_2^2)\log \frac{1}{\delta}\log \frac{p^2}{\xi} }{\epsilon^2}\text{Tr}( (X^TX)^{-2}) . 
\end{equation*}
For the term of $\text{Tr}( (X^TX)^{-2})$,  we get $$\text{Tr}( (X^TX)^{-2})\leq (\text{Tr}( (X^TX)^{-1}))^2\leq p \|(X^TX)^{-2}\|^2_2= \frac{p}{\lambda^2_{\min}(X^TX) }\leq O(\frac{p}{n^2\lambda^2_{\min}(\Sigma)}),$$ where the last inequality is due to the fact that $\lambda_{\min} (X^TX)\geq n\lambda_{\min}(\Sigma)-O(\kappa_x^2\|\Sigma\|_2\sqrt{pn})\geq \frac{1}{2}n\lambda_{\min}(\Sigma)$ (by the assumption on $n$). This completes the proof.
\end{proof}

Let $w^{ols}=(\mathbb{E}[xx^T])^{-1}\mathbb{E}[xy]$  denote the population linear regression estimator. The following lemma bounds the estimation error between $\tilde{w}^{ols}$ and $w^{ols}$. The proof could be found in \citep{erdogdu2019scalable} or \citep{dhillon2013new}. 
\begin{lemma}[Prop. 7 in \citep{erdogdu2019scalable}]\label{lemma:a5.9}
Assume that $\mathbb{E}[x_i]=0$, $\mathbb{E}[x_ix_i^T]=\Sigma$, and $\Sigma^{-\frac{1}{2}}x_i$ and $y_i$ are sub-Gaussian with norms $\kappa_x$ and $\gamma$, respectively. If $n\geq \Omega(\kappa_x\gamma p)$, the following holds
\begin{equation*}
    \|\tilde{w}^{ols}-w^{ols}\|_2\leq O\big(\gamma\kappa_x\sqrt{\frac{p}{n\lambda_{\min}(\Sigma)}}\big),
\end{equation*}
with probability at least $1-3\exp(-p)$. 
\end{lemma}

\section{Proofs of LDP}

The LDP proof of Algorithm \ref{alg:0} and \ref{alg:1} follows from the Gaussian mechanism (Lemma \ref{lemma:gaussian}) and the post-processing property of DP. 

For Algorithm \ref{alg:2}, it is $(\epsilon, \delta)$-LDP due to the $\ell_2$-norm bound on $\|x_i y_i\|_2= \|x_i\|_2\|f(\langle x, w^* \rangle)+\sigma_i\|_2\leq   \|x_i\|_2(L\|x\|_2+|f(0)|+C) $,
where the last inequality is due to the fact that  $f'$ is $L$-bounded  and $\|w^*\|_2\leq 1$. That is, $|f(\langle x, w^* \rangle)-f(0)|\leq L|\langle x, w^*\rangle -0|\leq L \|x\|_2\|w^*\|_2$. The proof is similar to Algorithm \ref{alg:1.5}.

\section{Proofs in Section 4}

Since Theorem \ref{thm:4} is the most complicated one, we will first prove it and then prove Theorem \ref{thm:3}. Finally we will proof Theorem \ref{thm:0}.  

\subsection{Proof of Theorem \ref{thm:4}} 
In the following proof we denote $\tilde{\mu}= \frac{\mathbb{E}[\|x\|_2]}{\sqrt{p}}$. 

Since $r=O(1)$ (by assumption), combining this with Lemmas \ref{lemma:a5.8} and \ref{lemma:a5.9}, we have that with probability at least $1-\exp(-\Omega(p))-\xi$ and under the assumption on $n$,  there is a constant $C_3>0$ such that 
\begin{equation}\label{aeq:5.2}
    \|\hat{w}^{ols}-w^{ols}\|_2\leq C_3 \frac{\kappa_x \sqrt{p}r^2\|w^{ols}\|_2\sqrt{\log\frac{1}{\delta}\log\frac{p^2}{\xi}}}{\epsilon\sqrt{n}\lambda^{1/2}_{\min} (\Sigma)\min\{\lambda^{1/2}_{\min} (\Sigma), 1\}}. 
\end{equation}

\begin{lemma}\label{lemma:a5.10}
Let $\Phi^{(2)}$ be a function that is Lipschitz continuous with constant $G$, and $f: \mathbb{R}\times \mathbb{R}^p \mapsto \mathbb{R}$ be another function such that $f(c, w)= c\mathbb{E}[\Phi^{(2)}(\langle x, w\rangle c)]$ and its empirical one is  
\begin{equation*}
    \hat{f}(c, w)= \frac{c}{m}\sum_{j=1}^m \Phi^{(2)}(\langle x, w\rangle c). 
\end{equation*}
Let $\mathbb{B}^\delta(\bar{w}^{ols})=\{w: \|w-\bar{w}^{ols}\|_2\leq \delta\}$, where $\bar{w}^{ols}=\Sigma^{\frac{1}{2}}w^{ols}$. 
Under the assumptions in Lemma \ref{lemma:a5.8} and Eq.~(\ref{aeq:5.2}), if further assume that $\|\Sigma^{-\frac{1}{2}}x\|_{\psi_2}\leq \kappa_x$,  $\sup_{w\in\mathbb{B}^\delta(\bar{w}^{ols})} \|\Phi^{(2)}(\langle x, w \rangle)\|_{\psi_2}\leq \kappa_g$, and there exist $\bar{c}>0$ and $\tau>0$ such that $f(\bar{c}, w^{ols})\geq 1+\tau$, then there is $\bar{c}_{\Phi}\in (0, \bar{c})$ such that $1=f(\bar{c}_\Phi, w^{ols})$. 
Also, for sufficiently large $n$ and $m$ such that  
  \begin{align}
      &m\geq \Omega \big( (\kappa_g+\frac{\kappa_x}{\tilde{\mu}})^2 \max\{ p\log m \tau^{-2},\frac{1}{G^2\tilde{\mu}^2} \frac{\epsilon^2n}{p r^4\|w^{ols}\|^2_2\log\frac{1}{\delta}\log\frac{p^2}{\xi}\|\Sigma\|_2}\}\big), \label{aeq:5.3}\\
      &n\geq \Omega (\kappa_x^4G^2\bar{c}^4\|\Sigma\|_2\frac{p r^4\|w^{ols}\|_2^2\log\frac{1}{\delta}\log\frac{p^2}{\xi}}{\tau^2\epsilon^2\lambda_{\min}(\Sigma)\min\{\lambda_{\min}(\Sigma),1 \}}\big), \label{aeq:5.4}
  \end{align}
  with probability at least $1-2\exp(-p)$, there exists a $\hat{c}_{\Phi}\in [0, \bar{c}]$ such that $\hat{f}(\hat{c}_{\Phi}, \hat{w}^{ols})=1$. 
 Furthermore,
 if the derivative of $c\mapsto f(c, w^{ols})$ is bounded below in the absolute value ({\em i.e.,} does not change sign) by $M>0$ in the interval $c\in [0, \bar{c}]$, then the following holds  
 \begin{equation}\label{aeq:5.5}
     |\hat{c}_\Phi-\bar{c}_\Phi|\leq O\big(M^{-1}\bar{c}( \kappa_g+\frac{\kappa_x}{\tilde{\mu}})\sqrt{\frac{p\log m}{m}}+M^{-1}G\kappa_x^2\bar{c}^2\|\Sigma\|^{\frac{1}{2}}_2\frac{\sqrt{p}r^2\|w^{ols}\|_2\sqrt{\log\frac{1}{\delta}\log\frac{p^2}{\xi}}}{\epsilon\sqrt{n}\lambda^{1/2}_{\min} (\Sigma)\min\{\lambda^{1/2}_{\min} (\Sigma), 1\}} \big ).
 \end{equation}
\end{lemma}

\begin{proof}[{\bf Proof of Lemma \ref{lemma:a5.10}}]
 We divide the proof into three parts.
 \paragraph{Part 1: Existence of $\bar{c}_\Phi$:} From the definition,  we know that $f(0, w^{ols})=0$ and $f(\bar{c}, w^{ols})>1$. Since $f$ is continuous, we known that there exists a constant $\bar{c}_\Phi\in (0, \bar{c})$ which satisfies $f(\bar{c}_{\Phi}, w^{ols})=1$. 
 
  \paragraph{Part 2: Existence of $\hat{c}_\Phi$:} For simplicity, we use the following notations. 
  \begin{equation}\label{aeq:5.6}
      \delta = C_3\frac{\kappa_x \sqrt{p}r^2\|w^{ols}\|_2\sqrt{\log\frac{1}{\delta}\log\frac{p^2}{\xi}}}{\epsilon\sqrt{n}\min\{\lambda^{1/2}_{\min} (\Sigma), 1\}}, \delta'= \frac{\|\Sigma\|_2^{\frac{1}{2}}\delta}{\lambda_{\min}^{\frac{1}{2}}(\Sigma)},
  \end{equation}
  where $C_3$ is the one in (\ref{aeq:5.2}). Thus, $\|\Sigma^{\frac{1}{2}}\hat{w}^{ols}-\Sigma^{\frac{1}{2}}w^{ols}\|_2\leq \delta'$.
  
  Now consider the term of $|\hat{f}(c, \hat{w}^{ols})-f(c, \hat{w}^{ols})|$ for $c\in [0, \bar{c}]$. We have 
  \begin{align}\label{aeq:5.7}
     \sup_{c\in [0, \bar{c}]} |\hat{f}(c, \hat{w}^{ols})-f(c, \hat{w}^{ols})|\leq \sup_{c\in [0, \bar{c}]}\sup_{w\in \mathbb{B}_{\Sigma}^{\delta'}(w^{ols})}|\hat{f}(c, w)-f(c, w)|, 
  \end{align}
  where $\mathbb{B}_{\Sigma}^{\delta'}(w^{ols})=\{w: \|\Sigma^{\frac{1}{2}}w-\Sigma^{\frac{1}{2}}w^{ols}\|_2\leq \delta'\} $. 
  
  Note that for any $x$, we have $\langle x, w\rangle= \langle v, \Sigma^{\frac{1}{2}}w\rangle$, where $v=\Sigma^{-\frac{1}{2}}x$ follows an isotropic sub-Gaussian distribution. Also, by definition we know that $w\in \mathbb{B}_{\Sigma}^{\delta'}(w^{ols})$ is equivalent to $\Sigma^{\frac{1}{2}}w \in \mathbb{B}^{\delta'}(\bar{w}^{ols})$. Thus, we have 
  \begin{align}
       &\sup_{c\in [0, \bar{c}]}\sup_{w\in \mathbb{B}_{\Sigma}^{\delta'}(w^{ols})}|\hat{f}(c, \hat{w}^{ols})-f(c, \hat{w}^{ols})| \nonumber \\
       &\leq \bar{c} \sup_{c\in [0, \bar{c}]}\sup_{w\in \mathbb{B}_{\Sigma}^{\delta'}(w^{ols})}|\frac{1}{m}\sum_{j=1}^{m}\Phi^{(2)}(\langle v_i, \Sigma^{\frac{1}{2}}w\rangle c)- \mathbb{E}\Phi^{(2)}(\langle v, \Sigma^{\frac{1}{2}}w\rangle c)| \nonumber \\
       &=\bar{c} \sup_{c\in [0, \bar{c}]}\sup_{\Sigma^{\frac{1}{2}} w\in \mathbb{B}^{\delta'}(\bar{w}^{ols})}|\frac{1}{m}\sum_{j=1}^{m}\Phi^{(2)}(\langle v_i, \Sigma^{\frac{1}{2}}w\rangle c)- \mathbb{E}\Phi^{(2)}(\langle v, \Sigma^{\frac{1}{2}}w\rangle c)|\nonumber \\
       &=\bar{c}\sup_{w'\in  \mathbb{B}^{\bar{c}\delta'}(\bar{w}^{ols})}|\frac{1}{m}\sum_{j=1}^{m}\Phi^{(2)}(\langle v_i, w'\rangle )- \mathbb{E}\Phi^{(2)}(\langle v, w'\rangle)| \label{aeq:5.8}.
  \end{align}
  By Lemma \ref{lemma:a5.6}, we know that when $mp\geq 51\max\{\chi, \chi^{-1}\}$, where $$\chi=\frac{(\kappa_g+\frac{\kappa_x}{\tilde{\mu}})^2}{c\delta'^2G^2\tilde{\mu}^2}= \Theta\big(\frac{(\kappa_g+\frac{\kappa_x}{\tilde{\mu}})^2}{G^2\tilde{\mu}^2} \frac{\epsilon^2n \lambda_{\min} (\Sigma) \min\{\lambda_{\min}(\Sigma), 1\}}{p r^4\|w^{ols}\|^2_2\log\frac{1}{\delta}\log\frac{p^2}{\xi}\|\Sigma\|_2} \big),$$
  the following holds with probability at least $1-2\exp(-p)$
  \begin{equation}\label{aeq:5.9}
      \sup_{w'\in  \mathbb{B}^{\bar{c}\delta}(\bar{w}^{ols})}|\frac{1}{m}\sum_{j=1}^{m}\Phi^{(2)}(\langle v_i, w'\rangle )- \mathbb{E}\Phi^{(2)}(\langle v, w'\rangle)|\leq O(( \kappa_g+\frac{\kappa_x}{\tilde{\mu}})\sqrt{\frac{p\log m}{m}}).
  \end{equation}

  By the Lipschitz property of $\Phi^{(2)}$, we have that for any $w_1$ and $w_2$,
  \begin{align}
      \sup_{c\in [0, \bar{c}]}|f(c, w_1)-f(c, w_2)|&\leq G\bar{c}^2\mathbb{E}[\langle v, \Sigma^{\frac{1}{2}}(w_1-w_2)\rangle ] \nonumber \\
      & \leq \kappa_x G\bar{c}^2\| \Sigma^{\frac{1}{2}}(w_1-w_2)\|_2. \label{aeq:5.10}
  \end{align}
  Taking $w_1= \hat{w}^{ols}$ and $w_2 = w^{ols}$, we have 
  $$\sup_{c\in [0, \bar{c}]}|f(c, \hat{w}^{ols})-f(c, w^{ols})|\leq O\big( \kappa_x G\bar{c}^2\|\Sigma\|_2^{\frac{1}{2}} \frac{\delta}{\lambda^{\frac{1}{2}}_{\min}(\Sigma) } \big). $$
  Combining this with (\ref{aeq:5.8}), (\ref{aeq:5.9}),  (\ref{aeq:5.10}), and taking $\delta$ as in (\ref{aeq:5.6}), we  get
  \begin{equation}\label{aeq:5.11}
      \sup_{c\in [0, \bar{c}] }|\hat{f}(c, \hat{w}^{ols})-f(c, w^{ols})|\leq  O\big( \bar{c}( \kappa_g+\frac{\kappa_x}{\tilde{\mu}})\sqrt{\frac{p\log m}{m}}+ G\bar{c}^2\|\Sigma\|_2^{\frac{1}{2}} \frac{\kappa_x^2 \sqrt{p}r^2\|w^{ols}\|_2\sqrt{\log\frac{1}{\delta}\log\frac{p^2}{\xi}}}{\epsilon\sqrt{n}\lambda^{1/2}_{\min}\min\{\lambda^{1/2}_{\min} (\Sigma), 1\}}\big). 
  \end{equation}
  Let $B$ denote
  the RHS of (\ref{aeq:5.11}). If $c=\bar{c}$, we have $\hat{f}(c, \hat{w}^{ols})\geq 1+\tau- B$. Thus, if $B\leq \tau$, 
  there must exist a $\hat{c}_\Phi\in [0, \bar{c}]$ such that $\hat{f}(\hat{c}_\Phi, \hat{w}^{ols})=1$. 
  
  To ensure that $B\leq \tau$ holds, it is sufficient to have
  $$O(\bar{c}(\kappa_g+\frac{\kappa_x}{\tilde{\mu}})\sqrt{\frac{p\log m}{m}})\leq \frac{\tau}{2}$$ and $$O( G\bar{c}^2\|\Sigma\|_2^{\frac{1}{2}} \frac{\kappa_x^2\sqrt{p}r^2\|w^{ols}\|_2\sqrt{\log\frac{1}{\delta}\log\frac{p^2}{\xi}}}{\epsilon\sqrt{n}\lambda^{1/2}_{\min}(\Sigma)\min\{\lambda^{1/2}_{\min} (\Sigma), 1\}})\leq \frac{\tau}{2}.$$ 
  This means that  
  \begin{align*}
      &m\geq \Omega \big( \bar{c}^2(\kappa_g+\frac{\kappa_x}{\tilde{\mu}})^2 p\log m \tau^{-2}\big), \\
      &n\geq \Omega(\kappa_x^4G^2\bar{c}^4\|\Sigma\|_2\frac{p r^4\|w^{ols}\|_2^2\log\frac{1}{\delta}\log\frac{p^2}{\xi}}{\tau^2\epsilon^2\lambda_{\min}(\Sigma)\min\{\lambda_{\min}(\Sigma),1 \}}\big),
  \end{align*}
  which are assumed in the lemma.
  
  \paragraph{Part 3: Estimation Error:} So far, we know that $\hat{f}(\hat{c}_\Phi, \hat{w}^{ols})= f(\bar{c}_\Phi, w^{ols})=1$ with high probability. By (\ref{aeq:5.7}), (\ref{aeq:5.8}) and (\ref{aeq:5.9}), we have 
  \begin{equation*}
      |1- f(\hat{c}_\Phi, \hat{w}^{ols})|= |\hat{f}(\hat{c}_\Phi, \hat{w}^{ols})- f(\hat{c}_\Phi, \hat{w}^{ols})|\leq O(\bar{c}( \kappa_g+\frac{\kappa_x}{\tilde{\mu}})\sqrt{\frac{p\log m}{m}}).
  \end{equation*}
  By the same argument for (\ref{aeq:5.11}), we have 
  \begin{equation*}
      |f(\hat{c}_\Phi, \hat{w}^{ols})- f(\hat{c}_\Phi, w^{ols})|\leq G\kappa_x\bar{c}^2\|\Sigma\|^{\frac{1}{2}}_2\frac{\delta}{\lambda^{\frac{1}{2}}_{\min}(\Sigma)}. 
  \end{equation*}
  Thus, using Taylor expansion on $f(c, w^{ols})$ around $c_\Phi$ and by the assumption of the bounded derivative of $f$, we have 
  \begin{align*}
      M|\hat{c}_\Phi-\bar{c}_\Phi|&\leq |f(\hat{c}_\Phi, w^{ols})- f(\bar{c}_\Phi, w^{ols})|\\
      &\leq |f(\hat{c}_\Phi, w^{ols})-f(\hat{c}_\Phi, \hat{w}^{ols})|+ |f(\hat{c}_\Phi, \hat{w}^{ols})-1|\\
      &\leq O\big(\bar{c}( \kappa_g+\frac{\kappa_x}{\tilde{\mu}})\sqrt{\frac{p\log m}{m}}+G\kappa_x^2\bar{c}^2\|\Sigma\|^{\frac{1}{2}}_2\frac{\sqrt{p}r^2\|w^{ols}\|_2\sqrt{\log\frac{1}{\delta}\log\frac{p^2}{\xi}}}{\epsilon\sqrt{n}\lambda^{1/2}_{\min} (\Sigma)\min\{\lambda^{1/2}_{\min} (\Sigma), 1\}} \big ). 
  \end{align*}
\end{proof}

Next, we  prove our main theorem. 

\begin{proof}[{\bf Proof of Theorem \ref{thm:4}}]
By definition, we have 
\begin{align}
    \|\hat{w}^{glm}-w^*\|_\infty &\leq \|\hat{c}_\Phi \hat{w}^{ols}- \bar{c}_{\Phi}w^{ols}\|_\infty+  \|\bar{c}_{\Phi}w^{ols}- w^*\|_\infty \nonumber \\
    &\leq  \|\hat{c}_\Phi \hat{w}^{ols}- \bar{c}_{\Phi}w^{ols}\|_\infty+  \|\bar{c}_{\Phi}w^{ols}- c_{\Phi}w^{ols}\|_\infty+\|c_{\Phi}w^{ols}- w^*\|_\infty. \label{aeq:5.12}
\end{align}
We first bound the term of $|\bar{c}_\Phi- c_\Phi|$. Since 
$\bar{c}_\Phi \mathbb{E}[\Phi^{(2)}(\langle x, w^{ols}\rangle \bar{c}_\Phi)]=1$ and $c_\Phi \mathbb{E}[\Phi^{(2)}(\langle x, w^* \rangle )]=1$ (by definition), we get
\begin{align*}
    |f(\bar{c}_\Phi, w^{ols})- f(c_\Phi, w^{ols})|&=|c_\Phi \mathbb{E}[\Phi^{(2)}(\langle x, w^* \rangle )]- f(c_\Phi, w^{ols})|\\
    & \leq  c_{\Phi}|\mathbb{E}[\Phi^{(2)}(\langle x, w^*\rangle ) -\Phi^{(2)}(\langle x, w^{ols}\rangle c_\Phi)]\\
    &\leq c_{\Phi}G|\mathbb{E}[\langle x, (w^*-c_\Phi w^{ols}) \rangle ]  \\
    &\leq c_{\Phi}G\|(w^*-c_\Phi w^{ols})\|_\infty \mathbb{E}\|x\|_1\\
    &\leq  c_{\Phi}Gr \|c_\Phi w^{ols}- w^*\|_\infty,
\end{align*}
where the last inequality is due to the assumption that $\|x\|_1\leq r$.

Thus, by the assumption of the bounded deviation of $f(c, w^{ols})$ on $[0, \max\{\bar{c}, c_\Phi\}]$, we have 
\begin{equation*}
    M|\bar{c}_\Phi- c_\Phi|\leq |f(\bar{c}_\Phi, w^{ols})- f(c_\Phi, w^{ols}) |\leq  c_{\Phi}Gr \|c_\Phi w^{ols}- w^*\|_\infty. 
\end{equation*}
By Lemma \ref{lemma:2}, we have 
\begin{equation}\label{aeq:5.13}
     |\bar{c}_\Phi- c_\Phi| \leq 16M^{-1} c_{\Phi}G^2 r^2\kappa_x^3\sqrt{\rho_2}\rho_{\infty}\frac{\|w^*\|^2_\infty}{\sqrt{p}}. 
\end{equation}
Thus, the second term of (\ref{aeq:5.12}) is bounded by 
\begin{align}
    &\|\bar{c}_{\Phi}w^{ols}- c_{\Phi}w^{ols}\|_\infty \leq 16M^{-1} c_{\Phi}G^2 r^2\kappa_x^3\sqrt{\rho_2}\rho_{\infty}\frac{\|w^*\|^2_\infty}{\sqrt{p}}\|w^{ols}\|_\infty \nonumber \\
    &\leq 16M^{-1} c_{\Phi}G^2 r^2\kappa_x^3\sqrt{\rho_2}\rho_{\infty}\frac{\|w^*\|^3_\infty}{\sqrt{p}}(\frac{1}{c_\Phi}+ 16Gr\kappa_x^3\sqrt{\rho_2}\rho_{\infty}\frac{\|w^*\|_\infty}{\sqrt{p}}) \nonumber \\
    &=O\big( M^{-1}r^3\kappa_x^6  G^3\rho_2 \rho_\infty^2 \frac{\|w^*\|^3_\infty\max\{1, \|w^*\|_\infty\}}{\sqrt{p}}\max\{1, c_\Phi\}  \big), \label{aeq:5.14}
\end{align}
where the last inequality is due to Lemma \ref{lemma:2}.

By Lemma \ref{lemma:2}, the third term of (\ref{aeq:5.12}) is bounded by $16c_\Phi Gr\kappa_x^3\sqrt{\rho_2}\rho_{\infty}\frac{\|w^*\|^2_\infty}{\sqrt{p}}$.

For the first term of (\ref{aeq:5.12}), by (\ref{aeq:5.2}) and Lemma \ref{lemma:a5.10} we have 
\begin{align}
    &\|\hat{c}_\Phi \hat{w}^{ols}- \bar{c}_{\Phi}w^{ols}\|_\infty \leq |\hat{c}_\Phi|\cdot \|\hat{w}^{ols}- w^{ols}\|_\infty + |\hat{c}_\Phi- \bar{c}_{\Phi}|\cdot \|w^{ols}\|_{\infty} \nonumber \\
    &\leq O\big( \bar{c}\frac{\kappa_x \sqrt{p}r^2\|w^{ols}\|_2\sqrt{\log\frac{1}{\delta}\log\frac{p^2}{\xi}}}{\epsilon\sqrt{n}\lambda^{1/2}_{\min} (\Sigma)\min\{\lambda^{1/2}_{\min} (\Sigma), 1\}} \nonumber \\
    & +\|w^{ols}\|_{\infty}(M^{-1}\bar{c}( \kappa_g+\frac{\kappa_x}{\tilde{\mu}})\sqrt{\frac{p\log m}{m}}+M^{-1}G\kappa_x^2\bar{c}^2\|\Sigma\|^{\frac{1}{2}}_2\frac{\sqrt{p}r^2\|w^{ols}\|_2\sqrt{\log\frac{1}{\delta}\log\frac{p^2}{\xi}}}{\epsilon\sqrt{n}\lambda^{1/2}_{\min} (\Sigma)\min\{\lambda^{1/2}_{\min} (\Sigma), 1\}}) \big). \label{aeq:5.15}
\end{align}
For the first term of (\ref{aeq:5.15}), we have 
\begin{align}
    &\bar{c}\frac{\kappa_x \sqrt{p}r^2\|w^{ols}\|_2\sqrt{\log\frac{1}{\delta}\log\frac{p^2}{\xi}}}{\epsilon\sqrt{n}\lambda^{1/2}_{\min} (\Sigma)\min\{\lambda^{1/2}_{\min} (\Sigma), 1\}}\leq \bar{c}\frac{\kappa_x p r^2\|w^{ols}\|_\infty\sqrt{\log\frac{1}{\delta}\log\frac{p^2}{\xi}}}{\epsilon\sqrt{n}\lambda^{1/2}_{\min} (\Sigma)\min\{\lambda^{1/2}_{\min} (\Sigma), 1\}}\nonumber \\
    &\leq \bar{c}\frac{\kappa_x p r^2\|w^*\|_\infty \sqrt{\log\frac{1}{\delta}\log\frac{p^2}{\xi}}}{\epsilon\sqrt{n}\lambda^{1/2}_{\min} (\Sigma)\min\{\lambda^{1/2}_{\min} (\Sigma), 1\}} (\frac{1}{c_\Phi}+ 16Gr\kappa_x^3\sqrt{\rho_2}\rho_{\infty}\frac{\|w^*\|_\infty}{\sqrt{p}})\nonumber \\
    &=O\big(\bar{c}\frac{p\kappa_x^4\sqrt{\rho_2}\rho_{\infty}Gr^3\|w^*\|_\infty \max\{1, \|w^*\|_\infty\} \sqrt{\log\frac{1}{\delta}\log\frac{p^2}{\xi}}}{\epsilon\sqrt{n}\lambda^{1/2}_{\min} (\Sigma)\min\{\lambda^{1/2}_{\min} (\Sigma), 1\}}\max\{1, \frac{1}{c_\Phi}\} \big). \label{aeq:5.16}
\end{align}
For the second term of (\ref{aeq:5.15}), we have 
\begin{align}
    &\|w^{ols}\|_{\infty}M^{-1}\bar{c} ( \kappa_g+\frac{\kappa_x}{\tilde{\mu}})\sqrt{\frac{p\log m}{m}} \nonumber \\
    &\leq \bar{c}\|w^*\|_\infty ( \kappa_g+\frac{\kappa_x}{\tilde{\mu}})\sqrt{\frac{p\log m}{m}}(\frac{1}{c_\Phi}+ 16Gr\kappa_x^3\sqrt{\rho_2}\rho_{\infty}\frac{\|w^*\|_\infty}{\sqrt{p}})\nonumber \\
    &\leq O\big( Gr\kappa_x^3\sqrt{\rho_2}\rho_{\infty} \bar{c}\|w^*\|_\infty \max\{1, \|w^*\|_\infty \}( \kappa_g+\frac{\kappa_x}{\tilde{\mu}})\sqrt{\frac{p\log m}{m}}\max\{1, \frac{1}{c_\Phi}\} \big). \label{aeq:5.17}
\end{align}
For the third term of (\ref{aeq:5.15}), we have 
\begin{align}
    &\|w^{ols}\|_\infty M^{-1}G\kappa_x^2\bar{c}^2\|\Sigma\|^{\frac{1}{2}}_2\frac{\sqrt{p}r^2\|w^{ols}\|_2\sqrt{\log\frac{1}{\delta}\log\frac{p^2}{\xi}}}{\epsilon\sqrt{n}\lambda^{1/2}_{\min} (\Sigma)\min\{\lambda^{1/2}_{\min} (\Sigma), 1\}}) \nonumber \\
    &\leq 
    M^{-1}G\kappa_x^2\bar{c}^2\|\Sigma\|^{\frac{1}{2}}_2\frac{pr^2\|w^*\|^2_\infty\sqrt{\log\frac{1}{\delta}\log\frac{p^2}{\xi}}}{\epsilon\sqrt{n}\lambda^{1/2}_{\min} (\Sigma)\min\{\lambda^{1/2}_{\min} (\Sigma), 1\}}(\frac{1}{c_\Phi}+ 16Gr\kappa_x^3\sqrt{\rho_2}\rho_{\infty}\frac{\|w^*\|_\infty}{\sqrt{p}})^2\nonumber \\
    &\leq O\big(M^{-1}G^3\kappa_x^8\bar{c}^2\rho_2\rho_\infty^2\|\Sigma^{\frac{1}{2}}\|_2\frac{pr^4\|w^*\|^2_\infty\max\{1, \|w^*\|^2_\infty\}\sqrt{\log\frac{1}{\delta}\log\frac{p^2}{\xi}}}{\epsilon\sqrt{n}\lambda^{1/2}_{\min} (\Sigma)\min\{\lambda^{1/2}_{\min} (\Sigma), 1\}}\max \{1, \frac{1}{c_\Phi}\}^2 \big). \label{aeq:5.18}
\end{align}
Thus, the first term of (\ref{aeq:5.12}) is bounded by (since $m\geq \Omega(n)$)
\begin{align*}
     &\|\hat{c}_\Phi \hat{w}^{ols}- \bar{c}_{\Phi}w^{ols}\|_\infty\leq O\big(\bar{c}\frac{p\kappa_x^4\sqrt{\rho_2}\rho_{\infty}Gr^3\|w^*\|^2_\infty\max\{1, \|w^*\|_\infty\} \sqrt{\log\frac{1}{\delta}\log\frac{p^2}{\xi}}}{\epsilon\sqrt{n}\lambda^{1/2}_{\min} (\Sigma)\min\{\lambda^{1/2}_{\min} (\Sigma), 1\}}\max\{1, \frac{1}{c_\Phi}\}   \\
     &  + Gr\kappa_x^3\sqrt{\rho_2}\rho_{\infty} \bar{c}\|w^*\|_\infty\max\{1, \|w^*\|_\infty\} ( \kappa_g+\frac{\kappa_x}{\tilde{\mu}})\sqrt{\frac{p\log m}{m}}\max\{1, \frac{1}{c_\Phi}\}+ \\
     &M^{-1}G^3\kappa_x^8\bar{c}^2\rho_2\rho_\infty^2\|\Sigma^{\frac{1}{2}}\|_2\frac{pr^4\|w^*\|^2_\infty\max\{1, \|w^*\|^2_\infty\}
     \sqrt{\log\frac{1}{\delta}\log\frac{p^2}{\xi}}}{\epsilon\sqrt{n}\lambda^{1/2}_{\min} (\Sigma)\min\{\lambda^{1/2}_{\min} (\Sigma), 1\}}\max \{1, \frac{1}{c_\Phi}\}^2\\
     &= O\big( M^{-1}( \kappa_g+\frac{\kappa_x}{\tilde{\mu}})G^3\kappa_x^8\bar{c}^2\rho_2\rho_\infty^2\|\Sigma^{\frac{1}{2}}\|_2\\
     &\times \frac{pr^4\|w^*\|_\infty\max\{1, \|w^*\|^3_\infty\} \sqrt{\log m\log\frac{1}{\delta}\log\frac{p^2}{\xi}}}{\epsilon\sqrt{n}\lambda^{1/2}_{\min} (\Sigma)\min\{\lambda^{1/2}_{\min} (\Sigma), 1\}}\max\{1, \frac{1}{c_\Phi}\}^2 \big).
\end{align*}
 Putting all the bounds together, we have 
\begin{align}
     &\|\hat{w}^{glm}-w^*\|_\infty \leq \tilde{O}\big(  M^{-1}G^3\kappa_x^8\bar{c}^2\rho_2\rho_\infty^2\|\Sigma^{\frac{1}{2}}\|_2 \nonumber\\
     &\times\frac{pr^4\|w^*\|_\infty\max\{1, \|w^*\|^3_\infty\}\sqrt{\log\frac{1}{\delta}\log\frac{p^2}{\xi}}}{\epsilon\sqrt{n}\lambda^{1/2}_{\min} (\Sigma)\min\{\lambda^{1/2}_{\min} (\Sigma), 1\}}\max\{1, \frac{1}{c_\Phi}\}^2 \nonumber \\
     &+ M^{-1}r^3\kappa_x^6 c_\Phi G^3\rho_2 \rho_\infty^2 \frac{\|w^*\|^2_\infty\max\{1,\|w^*\|^2_\infty\} }{\sqrt{p}}\max\{1, \frac{1}{c_\Phi}\}+ \nonumber \\
     &Gr\kappa_x^3\sqrt{\rho_2}\rho_{\infty} \bar{c}\|w^*\|_\infty\max\{1, \|w^*\|_\infty\} ( \kappa_g+\frac{\kappa_x}{\tilde{\mu}})\sqrt{\frac{p\log m}{m}}\max\{1, \frac{1}{c_\Phi}\}\big). \label{aeq:5.19}
\end{align}
Next, we bound the probability.  We assume that Lemma \ref{lemma:a5.8}, \ref{lemma:a5.9} and \ref{lemma:a5.10} hold with probability at least $1-\exp(-\Omega(p))-\rho$. They hold when 
  \begin{align}
      &m\geq \Omega \big( (\kappa_g+\frac{\kappa_x}{\tilde{\mu}})^2 \max\{ p\log m \tau^{-2},\frac{1}{G^2\tilde{\mu}^2} \frac{\epsilon^2n}{p r^4\|w^{ols}\|^2_2\log\frac{1}{\delta}\log\frac{p^2}{\xi}}\}\big),\\
      &n\geq \Omega(\max\{ \kappa_x^4G^2\bar{c}^4\|\Sigma\|_2\frac{p r^4\|w^{ols}\|_2^2\log\frac{1}{\delta}\log\frac{p^2}{\xi}}{\tau^2\epsilon^2\lambda_{\min}(\Sigma)\min\{\lambda_{\min}(\Sigma),1 \}}, \frac{\kappa_x^4\|\Sigma\|_2^2pr^4\log \frac{1}{\delta}}{\epsilon^2\lambda^2_{\min}(\Sigma)}\}\big).
  \end{align}
  Since $\|w^{ols}\|_2\leq \sqrt{p}\|w^*\|_\infty(\frac{1}{c_\Phi}+ 16Gr\kappa_x^3\sqrt{\rho_2}\rho_{\infty}\frac{\|w^*\|_\infty}{\sqrt{p}}),$ it suffices for $n$
\begin{equation}
    n\geq \Omega \big( G^4\bar{c}^4\|\Sigma\|^2_2\frac{p^2r^6\kappa_x^{10}\rho_2\rho^2_\infty\|w^*\|_\infty^2\max\{1,\|w^*\|_\infty^2\} \log\frac{1}{\delta}\log\frac{p^2}{\xi}}{\tau^2\epsilon^2\lambda_{\min}(\Sigma)\min\{\lambda_{\min}(\Sigma),1 \}}\max\{1, \frac{1}{c_\Phi}\}^2\big). 
\end{equation}
\end{proof}

\subsection{Proof of Theorem \ref{thm:3}}

\begin{lemma}\label{lemma:a5.13}
Let $\bar{c}_\Phi, \bar{c}, \tau, f, \hat{f}$ be defined the same as in Lemma \ref{lemma:a5.10}. If further assume that $|\Phi^{(2)}(\cdot)|\leq L$ for some constant $L>0$ and is Lipschitz continuous with constant $G$, then, under the assumptions in Lemma \ref{lemma:a5.8} and (\ref{aeq:5.2}), with probability at least $1-4\exp(-p)$ there exists a constant $\hat{c}_{\Phi}\in [0, \bar{c}]$ such that $\hat{f}(\hat{c}_\Phi, \hat{w}^{ols})=1$. Furthermore, if the derivative of $c\mapsto f(c, w^{ols})$ is bounded below in absolute value ({\em i.e.,} does not change the sign) by $M>0$ in the interval $c\in [0, \bar{c}]$, then with probability at least $1-4\exp(-p)$, the following holds  
\begin{equation}
    |\hat{c}_\Phi-\bar{c}_\Phi|\leq  O\big(\frac{ M^{-1}GL\bar{c}^2\kappa_x^2r^2 \|\Sigma\|_2^{\frac{1}{2}}\sqrt{p}\|w^{ols}\|_2 \sqrt{\log \frac{1}{\delta}\log \frac{p}{\xi^2}} 
        }{\epsilon\lambda^{\frac{1}{2}}_{\min}(\Sigma)\min\{ \lambda^{\frac{1}{2}}_{\min}(\Sigma), 1\}\sqrt{n} }+M^{-1}LG\|\Sigma\|_2^\frac{1}{2}\|w^{ols}\|_2 \sqrt{\frac{p}{m}}\big)
\end{equation}
for sufficiently large $m, n$ such that 
\begin{align}
   & n \geq \Omega \big( \frac{LG^2\tau^{-2} \bar{c}^4\|\Sigma\|_2 \kappa_x^4 p r^4\|w^{ols}\|_2^2\log\frac{1}{\delta}\log \frac{p^2}{\xi}}{\epsilon^2\lambda_{\min}(\Sigma)\min\{ \lambda_{\min}(\Sigma), 1\}}\big)\\
   &  m\geq \Omega \big(G^2L^2\|\Sigma\|_2 \|w^{ols}\|_2^2 p\tau^{-2}).
\end{align}
\end{lemma}

\begin{proof}[Proof of  Lemma \ref{lemma:a5.13} ]
The main idea of this proof is almost the same as the one for Lemma \ref{lemma:a5.10}. The only difference is that instead of using Lemma \ref{lemma:a5.6} to get (\ref{aeq:5.9}), we use here Lemma \ref{lemma:a5.7} to obtain the following with probability at least $1-\exp(-p)$ 
\begin{align}
        &\sup_{w'\in  \mathbb{B}^{\bar{c}\delta'}(\bar{w}^{ols})}|\frac{1}{m}\sum_{j=1}^{m}\Phi^{(2)}(\langle v_i, w'\rangle )- \mathbb{E}\Phi^{(2)}(\langle v, w'\rangle)|\nonumber\\
        &\leq O\big( (G(\|\bar{w}^{ols}\|_2+\bar{c}\delta')\|I\|_2+L)\sqrt{\frac{p}{m}}\nonumber \\
        &\leq O\big( ( G\|\Sigma\|_2^{\frac{1}{2}}(\|w^{ols}\|_2+\bar{c}\frac{\delta}{\lambda_{\min}^{\frac{1}{2}}(\Sigma)})+L)\sqrt{\frac{p}{m}}\big).\label{aeq:5.23}
\end{align}
Thus, by (\ref{aeq:5.8}), (\ref{aeq:5.10}) and (\ref{aeq:5.23}), we have 
\begin{multline}\label{aeq:5.27}
        \sup_{c\in [0, \bar{c}] }|\hat{f}(c, \hat{w}^{ols})-f(c, w^{ols})|\leq  O\big( G\|\Sigma\|_2^\frac{1}{2}\|w^{ols}\|_2\sqrt{\frac{p}{m}} +\\
       \frac{ G\kappa_x\bar{c}\|\Sigma\|_2^{\frac{1}{2}}\|w^{ols}\|_2 \sqrt{p} r^2\sqrt{\log\frac{1}{\delta}\log\frac{p^2}{\xi}}}{\epsilon\lambda^{1/2}_{\min}(\Sigma)\min\{\lambda^{1/2}_{\min} (\Sigma), 1\}}\sqrt{\frac{p}{mn}}+L\sqrt{\frac{p}{m}}\big).
\end{multline}
Let D denote the RHS of (\ref{aeq:5.27}), we have 
\begin{equation*}
    \hat{f}(\bar{c}, \hat{w}^{ols})\geq 1+\tau-D.
\end{equation*}
It is sufficient to show that $\tau>D$, which holds when 
$$O(G\bar{c}^2\|\Sigma\|_2^{\frac{1}{2}} \frac{\kappa_x^2 \sqrt{p}r^2\|w^{ols}\|_2\sqrt{\log\frac{1}{\delta}\log\frac{p^2}{\xi}}}{\epsilon\sqrt{n}\lambda^{1/2}_{\min}(\Sigma)\min\{\lambda^{1/2}_{\min} (\Sigma), 1\}})\leq \frac{\tau}{2}$$ and $$O( \frac{ G\kappa_x\bar{c}\|\Sigma\|_2^{\frac{1}{2}}L\|w^{ols}\|_2 \sqrt{p} r^2\sqrt{\log\frac{1}{\delta}\log\frac{p^2}{\xi}}}{\epsilon\lambda^{1/2}_{\min}(\Sigma)\min\{\lambda^{1/2}_{\min} (\Sigma), 1\}}\sqrt{\frac{p}{mn}})\leq\frac{\tau}{2}.$$ 
That is, 
\begin{align}
   & n \geq \Omega \big( \frac{G^2\tau^{-2} \bar{c}^4\|\Sigma\|_2 \kappa_x^4 p  r^4\|w^{ols}\|_2^2\log\frac{1}{\delta}\log \frac{p^2}{\xi}}{\epsilon^2\lambda_{\min}(\Sigma)\min\{ \lambda_{\min}(\Sigma), 1\}}\big)\\
   & m\geq \Omega \big(G^2L^2\|\Sigma\|_2 \|w^{ols}\|_2^2 p\tau^{-2}).
\end{align}
Then, there exists $\hat{c}_{\Phi}\in [0, \bar{c}]$ such that $ \hat{f}(\hat{c}_{\Phi}, \hat{w}^{ols})=1$. We can easily get 
\begin{align}\label{aeq:28}
       & M|\hat{c}_\Phi-\bar{c}_\Phi|\leq |f(\hat{c}_\Phi, w^{ols})- f(\bar{c}_\Phi, w^{ols})|\nonumber \\
        &\leq O\big( \frac{ G\bar{c}^2\kappa_x^2r^2 \|\Sigma\|_2^{\frac{1}{2}}\sqrt{p}\|w^{ols}\|_2\sqrt{ \log \frac{1}{\delta}\log \frac{p}{\xi^2} }
        }{\epsilon\lambda^{\frac{1}{2}}_{\min}(\Sigma)\min\{ \lambda^{\frac{1}{2}}_{\min}(\Sigma), 1\}\sqrt{n} } \nonumber\\
        &+ \frac{ G\kappa_x\bar{c}\|\Sigma\|_2^{\frac{1}{2}}\|w^{ols}\|_2 \sqrt{p} r^2\sqrt{\log\frac{1}{\delta}\log\frac{p^2}{\xi}}}{\epsilon\lambda^{1/2}_{\min}(\Sigma)\min\{\lambda^{1/2}_{\min} (\Sigma), 1\}}\sqrt{\frac{p}{mn}}+LG\|\Sigma\|_2^\frac{1}{2}\|w^{ols}\|_2 \sqrt{\frac{p}{m}}\big)\\
        &\leq O\big(\frac{ GL\bar{c}^2\kappa_x^2r^2 \|\Sigma\|_2^{\frac{1}{2}}\sqrt{p}\|w^{ols}\|_2 \sqrt{\log \frac{1}{\delta}\log \frac{p}{\xi^2}} 
        }{\epsilon\lambda^{\frac{1}{2}}_{\min}(\Sigma)\min\{ \lambda^{\frac{1}{2}}_{\min}(\Sigma), 1\}\sqrt{n} }+LG\|\Sigma\|_2^\frac{1}{2}\|w^{ols}\|_2 \sqrt{\frac{p}{m}}\big).
\end{align}
\end{proof}

\begin{proof}[{\bf Proof of Theorem \ref{thm:3} }]
The proof is almost the same as the one for Theorem \ref{thm:4}. By  definition,  we have 
\begin{align}
    \|\hat{w}^{glm}-w^*\|_\infty &\leq \|\hat{c}_\Phi \hat{w}^{ols}- \bar{c}_{\Phi}w^{ols}\|_\infty+  \|\bar{c}_{\Phi}w^{ols}- w^*\|_\infty \nonumber \\
    &\leq  \|\hat{c}_\Phi \hat{w}^{ols}- \bar{c}_{\Phi}w^{ols}\|_\infty+  \|\bar{c}_{\Phi}w^{ols}- c_{\Phi}w^{ols}\|_\infty+\|c_{\Phi}w^{ols}- w^*\|_\infty. \label{aeq:5.32}
\end{align}
The second term of (\ref{aeq:5.32}) is bounded by 
\begin{align}
      &\|\bar{c}_{\Phi}w^{ols}- c_{\Phi}w^{ols}\|_\infty \leq O\big( M^{-1}r^2\kappa_x^7 c_\Phi G^3\rho_2 \rho_\infty^2 \frac{\|w^*\|^3_\infty\max\{1, \|w^*\|_\infty\}}{\sqrt{p}}\max\{1, \frac{1}{c_\Phi}\}  \big). \label{aeq:5.33}
\end{align}
By Lemma \ref{lemma:2}, the third term of (\ref{aeq:5.32}) is bounded by $16c_\Phi Gr\kappa_x^3\sqrt{\rho_2}\rho_{\infty}\frac{\|w^*\|_\infty}{\sqrt{p}}$.
The first term is bounded by 
\begin{multline}
    \|\hat{c}_\Phi \hat{w}^{ols}- \bar{c}_{\Phi}w^{ols}\|_\infty\leq \\
    O\big( \frac{M^{-1} G^3L\bar{c}^2\kappa_x^8 r^4\rho_2\rho^2_\infty \|w^*\|^2_\infty\max\{1, \|w^*\|^2_\infty\} \|\Sigma\|_2^{\frac{1}{2}}p\sqrt{\log \frac{1}{\delta}\log \frac{p}{\xi^2}} 
        }{\epsilon\lambda^{\frac{1}{2}}_{\min}(\Sigma)\min\{ \lambda^{\frac{1}{2}}_{\min}(\Sigma), 1\}\sqrt{n} }
        \times \max\{\frac{1}{c_\Phi}, 1\}^2 \\
    + \frac{M^{-1} G^3L\bar{c}^2\kappa_x^6 r^2\rho_2\rho^2_\infty \|w^*\|^2_\infty\max\{1, \|w^*\|^2_\infty\} \|\Sigma\|_2^{\frac{1}{2}}p
        }{\sqrt{m} }
        \times \max\{\frac{1}{c_\Phi}, 1\}^2   \big).
\end{multline}
Thus, in total we have 
\begin{multline}
     \|\hat{w}^{glm}-w^*\|_\infty \leq 
     O\big( \frac{M^{-1} G^3L\bar{c}^2\kappa_x^6 r^2\rho_2\rho^2_\infty \|w^*\|^2_\infty\max\{1, \|w^*\|^2_\infty\} \|\Sigma\|_2^{\frac{1}{2}}p
        }{\sqrt{m} }
        \times \max\{\frac{1}{c_\Phi}, 1\}^2 \\+
        \frac{ G^3L\bar{c}^2\kappa_x^6r^4\rho_2\rho^2_\infty\|w^*\|^2_\infty\max\{1, \|w^*\|^2_\infty\} \|\Sigma\|_2^{\frac{1}{2}}p\sqrt{\log \frac{1}{\delta}\log \frac{p}{\xi^2}} 
        }{\epsilon\lambda^{\frac{1}{2}}_{\min}(\Sigma)\min\{ \lambda^{\frac{1}{2}}_{\min}(\Sigma), 1\}\sqrt{n} }\max\{\frac{1}{c_\Phi}, 1\}^2\\ + M^{-1}r^2\kappa_x^7 c_\Phi G^3\rho_2 \rho_\infty^2 \|\Sigma^{\frac{1}{2}}\|_\infty\frac{\|w^*\|^3_\infty\max\{1, \|w^*\|_\infty\}}{\sqrt{p}}\max\{1, \frac{1}{c_\Phi}\}  \big). 
\end{multline}
The probability of success is at least $1-\exp(-\Omega(p))-\xi$. The sample complexity should satisfy 
\begin{align}
    & m\geq \Omega \big(G^2L^2\|\Sigma\|_2 \|w^*\|^2_\infty\max\{1, \|w^*\|^2_\infty\}  G^2r^2\kappa_x^6\rho_2\rho^2_{\infty} p^2\tau^{-2} \max\{1, \frac{1}{c_\Phi}\}^2 \big)\\
    & n \geq \Omega \big( \frac{\rho_2\rho^2_\infty G^4\tau^{-2} \bar{c}^4\|\Sigma\|^2_2 \kappa_x^{10} p^2\|w^*\|_\infty^2 r^6\max\{1,\|w^*\|_\infty^2\}\log\frac{1}{\delta}\log \frac{p^3}{\xi}}{\epsilon^2\lambda_{\min}(\Sigma)\min\{ \lambda_{\min}(\Sigma), 1\}}\max\{1, \frac{1}{c_\Phi}\}^2\big) . 
\end{align}
\end{proof}

\subsection{Proof of Theorem \ref{thm:0}}
\begin{proof} 
To prove the result, we first focus on the term of $   \|\hat{w}^{ols}-w^{ols}\|_2$ where $w^{ols}=\Sigma^{-1}\mathbb{E}(xy)$. First, note that by Lemma \ref{lemma:a5.4} we have with probability at least $1-\exp(-p)$, 
\begin{equation*}
    \|\Sigma_m\|_2\geq \|\Sigma\|_2-O(k_x^2 \sqrt{\frac{p}{m}}\|\Sigma\|_2). 
\end{equation*}
Thus, when $m\geq \Omega(k_x^4 p)$ we have $\frac{3\|\Sigma\|_2}{2}\geq  \|\Sigma_m\|_2\geq \frac{\|\Sigma\|_2}{2}$. In the following we will always assume the inequality holds.  We denote $\hat{\Sigma}=\mathbb{E}(\bar{x}\bar{x}^T)$ where $x\sim \mathcal{N}(0, \Sigma)$ and $\bar{x}=x\min\{1, \frac{r}{\|x\|_2}\}$.  Next we show the lemma of bounding the term $\|\hat{\Sigma}-\Sigma\|_2$ and $\|\mathbb{E}(\bar{x}y)-\mathbb{E}(xy)\|_2$: 
\begin{lemma}\label{lemma:revised1}
We have $  \|\hat{\Sigma}-\Sigma\|_2\leq O(\frac{\|\Sigma\|_2^2}{n})$ and $\|\mathbb{E}(\bar{x}y)-\mathbb{E}(xy)\|_2\leq O(\frac{\sqrt{p\|\Sigma\log n}}{n})$. 
\end{lemma}
\begin{proof}[{ Proof of Lemma \ref{lemma:revised1}}]
    By the definitions we have 
    \begin{align*}
        \|\hat{\Sigma}-\Sigma\|_2\leq \|\mathbb{E}[(\bar{x}\bar{x}^T-xx^T)\mathbb{I}_{\|x\|_2\geq r} ]\|_2.
    \end{align*}
    For any unit vector $v\in \mathbb{R}^p$ we have 
    \begin{align*}
        &v^T \mathbb{E}[(xx^T-\bar{x}\bar{x}^T)\mathbb{I}_{\|x\|_2\geq r} ]v=\mathbb{E}[((v^Tx)^2- (v^T\bar{x})^2)\mathbb{I}_{\|x\|_2\geq r} ]\\
        &\leq \mathbb{E}[(v^Tx)^2\mathbb{I}_{\|x\|_2\geq r}]\leq \sqrt{\mathbb{E}[(v^Tx)^4]\text{Pr}[\|x\|_2\geq r] }\leq O(\frac{\|\Sigma\|_2^2}{n}), 
    \end{align*}
    where the last inequality is due to that $\text{Pr}[\|x\|_2\geq r]\leq \text{Pr}[\|x\|_2\geq \sqrt{10p\|\Sigma\|_2\log n }]\leq \frac{1}{n^2}$.

    For $\|\mathbb{E}(\bar{x}y)-\mathbb{E}(xy)\|_2$ we have 
    \begin{align*}
        &\|\mathbb{E}(\bar{x}y)-\mathbb{E}(xy)\|_2=       \|\mathbb{E}(\bar{x}-x)y\mathbb{I}_{\|x\|_2\geq r}\|_2\\
        &\leq \sqrt{\mathbb{E}\|(\bar{x}-x)y\|_2^2 \text{Pr}(\|x\|_2\geq r)  }\leq O(\frac{r+\sqrt{p}\|\Sigma\|_2}{n}).
    \end{align*}
\end{proof}
By  Lemma \ref{lemma:a5.2} and \ref{lemma:revised1} we have $\lambda_{\min}(\hat{\Sigma})\geq \frac{\lambda_{\min}({\Sigma}) }{2}$ when $n\geq \frac{\|\Sigma\|_2^2}{ \lambda_{\min}({\Sigma}) }$.

In the following we will bound the term $ \|\hat{w}^{ols}-w^{ols}\|_2$. For simplicity we denote $\overline{XX^T}=\sum_{i=1}^n \bar{x}_i \bar{x}_i^T$ and $\overline{X^Ty}=\sum_{i=1}^n \bar{x}_i y_i$. Then we have 
\begin{align}
     &\|\hat{w}^{ols}-w^{ols}\|_2\leq 
     \|\hat{w}^{ols}-(\overline{XX^T})^{-1}\overline{X^Ty}\|_2+\|(\overline{XX^T})^{-1}\overline{X^Ty}-\Sigma^{-1}\mathbb{E}(xy)\|_2 \notag
     \\ &\leq \|\hat{w}^{ols}-(\overline{XX^T})^{-1}\overline{X^Ty}\|_2 +\|(\overline{XX^T})^{-1}\overline{X^Ty}-\hat{\Sigma}^{-1}\mathbb{E}(\bar{x}y)\|_2+ \|\hat{\Sigma}^{-1}\mathbb{E}(\bar{x}y)-\Sigma^{-1}\mathbb{E}(xy)\|_2. \label{eq:revised1}
\end{align}
We then bound each term in (\ref{eq:revised1}). We first bound the second term: 
\begin{align*}
   & \|(\frac{1}{n}\overline{XX^T})^{-1}(\frac{1}{n}\overline{X^Ty})-\hat{\Sigma}^{-1}\mathbb{E}(\bar{x}y)\|_2 \\&\leq \| (\frac{1}{n}\overline{XX^T})^{-1}-\hat{\Sigma}^{-1}\|_2 \|\frac{1}{n}\overline{X^Ty}\|_2+\| \hat{\Sigma}^{-1}\|_2\|\frac{1}{n}\overline{X^Ty}-\mathbb{E}(\bar{x}y)\|_2 \\
   &\leq \|\hat{\Sigma}^{-1}\|_2\|(\frac{1}{n}\overline{XX^T})^{-1}\|_2  \| \frac{1}{n}\overline{XX^T}-\hat{\Sigma}\|_2\|\frac{1}{n}\overline{X^Ty}\|_2 + \| \hat{\Sigma}^{-1}\|_2\|\frac{1}{n}\overline{X^Ty}-\mathbb{E}(\bar{x}y)\|_2
\end{align*}
Below we consider two lemmas: 
\begin{lemma}\label{lemma:revised2}
If $n\geq \tilde{\Omega}(p\|\Sigma\|_2)$, with probability at least $1-\zeta$ 
    \begin{equation*}
        \|\frac{1}{n}\overline{XX^T}-\hat{\Sigma}\|_2\leq O(\frac{\sqrt{p\|\Sigma\|_2\log n\log \frac{p}{\zeta} }}{\sqrt{n}}).
        \end{equation*}
\end{lemma}
\begin{proof}
    Note that $\|\bar{x}\bar{x}^T-\hat{\Sigma}\|_2\leq \|\bar{x}\bar{x}^T\|_2+\|\hat{\Sigma}\|_2\leq 2r^2$. And for any unit vector $v\in \mathbb{R}^p$ we have the following if we denote $\bar{X}=\bar{x}\bar{x}^T$
    \begin{align*}
        \mathbb{E}(v^T \bar{X}^T\bar{X}v)=\mathbb{E}[\|\bar{x}\|_2^2 (v^T\bar{x})^2]\leq O(r^4).
    \end{align*}

    Thus we  have $\|\mathbb{E}[\bar{X}^T\bar{X}]\|_2\leq O(r^2)$. Since $\|\mathbb{E}(\bar{X})^T\mathbb{E}(\bar{X})\|_2\leq \|\mathbb{E}(\bar{X})\|_2^2\leq r^2$, we have $\|\mathbb{E}[\bar{X}-\mathbb{E}\bar{X}]^T\mathbb{E}[\bar{X}-\mathbb{E}\bar{X}]\|_2\leq O(r^2)$. Thus, by the Non-communicative Bernstein inequality (Lemma \ref{bernstein}) we have for some constant $c>0$:
    \begin{equation*}
        \text{Pr}(\|\frac{1}{n}\overline{XX^T}-\hat{\Sigma}\|_2>t)\leq 2p\exp(-c\min(\frac{nt^2}{r^2}, \frac{nt}{r^2})).
    \end{equation*}
    Thus we have with probability at least $1-\zeta$ and the definition of $r$ we have, 
    \begin{equation*}
        \|\frac{1}{n}\overline{XX^T}-\hat{\Sigma}\|_2\leq O(\frac{\sqrt{p\|\Sigma\|_2\log n\log \frac{p}{\zeta} }}{\sqrt{n}}). 
    \end{equation*}
    
\end{proof}

Since each $\|\bar{x}_iy_i-\mathbb{E}[\bar{x}_iy_i]\|\leq 2r $, by Lemma \ref{norm-sub} we have 
\begin{lemma}\label{lemma:revised3}
With probability at least $1-\zeta$, $ \|\frac{1}{n}\overline{X^Ty}-\mathbb{E}(\bar{x}y)\|_2\leq O(\frac{r\sqrt{\log\frac{p}{\zeta}}}{\sqrt{n}})$.
\end{lemma}
Next we bound the term of $ \|\hat{\Sigma}^{-1}\|_2, \|(\frac{1}{n}\overline{XX^T})^{-1}\|_2 $. By Lemma \ref{lemma:revised1} we can see we have $\|\hat{\Sigma}^{-1}\|_2=\frac{1}{\lambda_{\min}(\hat{\Sigma})}\leq \frac{2}{\lambda_{\min}(\Sigma)}$. By Lemma \ref{lemma:revised2} we have if $n\geq \tilde{\Omega}(\frac{p\|\Sigma\|_2}{\lambda_{\min}(\Sigma)})$ then we have $\lambda_{\min}(\frac{1}{n}\overline{XX^T})\geq \frac{\lambda_{\min}(\hat{\Sigma})}{2}\geq \frac{\lambda_{\min}({\Sigma})}{4}$. Thus, in total we have 
\begin{align*}
    & \|\hat{\Sigma}^{-1}\|_2\|(\frac{1}{n}\overline{XX^T})^{-1}\|_2  \| \frac{1}{n}\overline{XX^T}-\hat{\Sigma}\|_2 \|\frac{1}{n}\overline{X^Ty}\|_2 + \| \hat{\Sigma}^{-1}\|_2\|\frac{1}{n}\overline{X^Ty}-\mathbb{E}(\bar{x}y)\|_2\\
    &\leq O(\frac{{p\|\Sigma\|_2\log n\log \frac{p}{\zeta} }}{\lambda^2_{\min}(\Sigma)\sqrt{n}}+\frac{\sqrt{p\|\Sigma\|_2\log n\log\frac{p}{\zeta}}}{\lambda_{\min}(\Sigma)\sqrt{n}}).
\end{align*}
Next we consider the third term of (\ref{eq:revised1})
\begin{align*}
    &\|\hat{\Sigma}^{-1}\mathbb{E}(\bar{x}y)-\Sigma^{-1}\mathbb{E}(xy)\|_2\\
    &\leq \|\hat{\Sigma}-\Sigma\|_2\|\hat{\Sigma}^{-1}\|_2 \|\Sigma^{-1}\|_2 \|\mathbb{E}(\bar{x}y)\|_2+\|\Sigma^{-1}\|_2\|\mathbb{E}(\bar{x}y)-\mathbb{E}(xy)\|_2\\
    &\leq O(\frac{\sqrt{p\|\Sigma\|_2\log n}\|\Sigma\|_2^2}{\lambda^2_{\min}(\Sigma) n}+ \frac{\sqrt{p\|\Sigma\|_2\log n}}{\lambda_{\min}(\Sigma) n}).
\end{align*}
For the first term of \eqref{eq:revised1}, by using a similar proof as in Lemma \ref{lemma:a5.8}  we have  if $n\geq \Omega(\frac{\|\Sigma\|_2^2pr^4\log \frac{1}{\delta}}{\epsilon^2\lambda^2_{\min}(\Sigma)})$ then  with probability at least $1-\exp(-\Omega(p))-\xi$  (if we denote $\bar{w}^{ols}=(\overline{XX^T})^{-1}\overline{X^Ty}$) 
\begin{equation}
    \|\hat{w}^{ols}- \bar{w}^{ols}\|_2^2 = O\big( \frac{p r^2(1+r^2\|\bar{w}^{ols}\|_2^2)\log \frac{1}{\delta}\log \frac{p^2}{\xi} }{\epsilon^2n\lambda^2_{\min}(\Sigma)}\big).  
\end{equation}
By the previous proof we can see that 
\begin{align*}
    \|\bar{w}^{ols}-w^{ols}\|_2
    &\leq  O\left(\frac{{p\|\Sigma\|_2\log n\log \frac{p}{\zeta} }}{\lambda^2_{\min}(\Sigma)\sqrt{n}}+\frac{\sqrt{p\|\Sigma\|_2\log n\log\frac{p}{\zeta}}}{\lambda_{\min}(\Sigma)\sqrt{n}}+\frac{\sqrt{p\|\Sigma\|_2\log n}\|\Sigma\|_2^2}{\lambda^2_{\min}(\Sigma) n}+ \frac{\sqrt{p\|\Sigma\|_2\log n}}{\lambda_{\min}(\Sigma) n}\right)\\
    &=O(\frac{{p\|\Sigma\|_2\log n \log \frac{p}{\zeta}}}{\sqrt{n} \lambda_{\min}(\Sigma)\min\{\lambda_{\min}(\Sigma), 1\}}).
\end{align*}
In total when $n\geq \tilde{\Omega}(p^2\|\Sigma\|_2/\lambda^4_{\min}(\Sigma))$ we have 
\begin{equation}
    \|\hat{w}^{ols}- \bar{w}^{ols}\|_2^2 = O\big( \frac{p r^2(1+r^2\|{w}^{ols}\|_2^2)\log \frac{1}{\delta}\log \frac{p^2}{\xi} }{\epsilon^2n\lambda^2_{\min}(\Sigma)}\big).  
\end{equation}
Thus, combine all the previous results we have with probability at least $1-\exp(-\Omega(p))-\xi$ we have 
\begin{align*}
      \|\hat{w}^{ols}-{w}^{ols}\|_2\leq O\big( \frac{\sqrt{p}r^2\|{w}^{ols}\|_2\log \sqrt{\frac{1}{\delta}\log \frac{p}{\xi} }}{\epsilon\sqrt{n}\lambda_{\min}(\Sigma)}+\frac{{p\|\Sigma\|_2\log n \log \frac{p}{\zeta}}}{\sqrt{n} \lambda_{\min}(\Sigma)\min\{\lambda_{\min}(\Sigma), 1\}}\big)
\end{align*}
Thus, there is a constant $C_3>0$ such that 
\begin{equation}\label{aeq:new5.2}
    \|\hat{w}^{ols}-w^{ols}\|_2\leq C_3 \frac{\sqrt{p^3}\|\Sigma\|_2 \|w^{ols}\|_2\log {n}\sqrt{\log\frac{1}{\delta}\log\frac{p}{\xi}}}{\epsilon\sqrt{n}\lambda_{\min} (\Sigma)\min\{\lambda_{\min} (\Sigma), 1\}}. 
\end{equation}

The same Lemma \ref{lemma:a5.13}, we have the following lemma. 
	\begin{lemma}\label{lemma:new1}
Let $\bar{c}_\Phi, \bar{c}, \tau, f, \hat{f}$ be defined the same as in Lemma \ref{lemma:a5.10}. If further assume that $|\Phi^{(2)}(\cdot)|\leq L$ for some constant $L>0$ and is Lipschitz continuous with constant $G$, then, under the assumptions in Lemma \ref{lemma:a5.8} and (\ref{aeq:5.2}), with probability at least $1-4\exp(-p)$ there exists a constant $\hat{c}_{\Phi}\in [0, \bar{c}]$ such that $\hat{f}(\hat{c}_\Phi, \hat{w}^{ols})=1$. Furthermore, if the derivative of $c\mapsto f(c, w^{ols})$ is bounded below in absolute value ({\em i.e.,} does not change the sign) by $M>0$ in the interval $c\in [0, \bar{c}]$, then with probability at least $1-4\exp(-p)$, the following holds  (note that for the Gaussian case $c_\Phi=\bar{c}_\Phi$)
\begin{equation}
    |\hat{c}_\Phi-c_\Phi|\leq  O\big(\frac{ M^{-1}GL\bar{c}^2 \|\Sigma\|_2^{\frac{3}{2}}p^{\frac{3}{2}}\|w^{ols}\|_2\log {n}\sqrt{\log \frac{1}{\delta}\log \frac{p}{\xi}} 
        }{\epsilon\lambda_{\min}(\Sigma)\min\{ \lambda_{\min}(\Sigma), 1\}\sqrt{n} }+M^{-1}LG\|\Sigma\|_2^\frac{1}{2}\|w^{ols}\|_2 \sqrt{\frac{p}{m}}\big)
\end{equation}
for sufficiently large $m, n$ such that 
\begin{align}
   & n \geq \Omega \big( \frac{LG^2\tau^{-2} \bar{c}^4\|\Sigma\|^3_2  p^3 \|w^{ols}\|_2^2 \log^2{n}\log\frac{1}{\delta}\log \frac{p}{\xi}}{\epsilon^2\lambda^2_{\min}(\Sigma)\min\{ \lambda^2_{\min}(\Sigma), 1\}}\big)\\
   &  m\geq \Omega \big(G^2L^2\|\Sigma\|_2 \|w^{ols}\|_2^2 p\tau^{-2}).
\end{align}
\end{lemma}
Next we bound $\|\hat{w}^{glm}-w^*\|_2=\|\hat{c}_\Phi \hat{w}^{ols}-c_\Phi w^{ols}\|_2$. We have 
\begin{align}\label{aeq:new1}
	\|\hat{c}_\Phi \hat{w}^{ols}-c_\Phi w^{ols}\|_2\leq |\hat{c}_\Phi-c_\Phi|\|\hat{w}^{ols}\|_2+ c_\Phi \|\hat{w}^{ols}-w^{ols}\|_2. 
\end{align}
For the second term of (\ref{aeq:new1}),  by (\ref{aeq:new5.2}) we have 
\begin{equation*}
	c_\Phi \|\hat{w}^{ols}-w^{ols}\|_2\leq O(\frac{\bar{c}p^\frac{3}{2}\|\Sigma\|_2 \|w^{ols}\|_2\log {n}\sqrt{\log\frac{1}{\delta}\log\frac{p^2}{\xi}}}{\epsilon\sqrt{n}\lambda_{\min} (\Sigma)\min\{\lambda_{\min} (\Sigma), 1\}}). 
\end{equation*}
For the first term of (\ref{aeq:new1}), by Lemma \ref{lemma:new1} and (\ref{aeq:new5.2}) we have 
\begin{equation*}
	|\hat{c}_\Phi-c_\Phi|\|\hat{w}^{ols}\|_2\leq O\big(\frac{ M^{-1}GL\bar{c}^2 \|\Sigma\|_2^{\frac{3}{2}}p^{\frac{3}{2}}\|w^{ols}\|^2_2\log{n} \sqrt{\log \frac{1}{\delta}\log \frac{p}{\xi}} 
        }{\epsilon\lambda_{\min}(\Sigma)\min\{ \lambda_{\min}(\Sigma), 1\}\sqrt{n} }+M^{-1}LG\|\Sigma\|_2^\frac{1}{2}\|w^{ols}\|^2_2 \sqrt{\frac{p}{m}}\big)
\end{equation*}
Take $w^{ols}=\frac{w^*}{c_\Phi}$ we can get the proof. 
\end{proof}
\subsection{Proof of Theorem \ref{theorem:11}}
\begin{proof}
	We can see that 
	\begin{equation*}
		\Phi^{(2)}(z)=\frac{e^z}{(1+e^z)^2},  \Phi^{(3)}(z)=\frac{e^z-e^{2z}}{(1+e^z)^3},  \Phi^{(4)}(z)=\frac{e^z(1-4e^z+e^{2z})}{(1+e^z)^4}
	\end{equation*}
	We can see $|\Phi^{(2)}(\cdot)|\leq 1$ and $\Phi^{(2)}(\cdot)$ is $1$-Lipschtitz, and $\Phi^{(2)}$ and $\Phi^{(4)}$ are even functions. Using the local convexity for $z\geq 0$ around $z=2.5$ we have 
	\begin{equation*}
		\Phi^{(2)}(z)\geq a-bz, 
	\end{equation*}
	where $a= \Phi^{(2)}(2.5)- 2.5\Phi^{(3)}(2.5)\approx 0.22$ and $b= -\Phi^{(3)}(2.5)\approx 0.06$. Denote $W\sim \mathcal{N}(0,1)$, $\phi$ as the density function of $W$ and $\zeta$ as the cumulative distribution function of $W$, we have 
	\begin{align*}
		f(z)&=z\mathbb{E}[\Phi^{(2)}(\langle x, w^{ols}\rangle z)] = z\mathbb{E}[\Phi^{(2)}(\frac{Wz}{20})]\\
		&= 2z\int_{0}^{\infty} \Phi^{(2)}(\frac{wz}{20})\phi(w) dw\geq 2z\int_{0}^{\frac{20a}{bz}} (a-b\frac{wz}{20})\phi(w)dw \\
		&=2z(a\zeta(\frac{20a}{bz})-\frac{a}{2}-\frac{bz}{20\sqrt{2\pi}}(1-e^\frac{-200a^2}{b^2z^2})). 
	\end{align*}
	Thus take $\bar{c}=6$ we have $f(\bar{c})>1+0.22$. 

	Next we will show $  c_\Phi \leq \bar{c}$. Recall that $c_{\Phi}=\frac{1}{\mathbb{E}[\Phi^{(2)}(\langle x_i, w^*\rangle)] }$, thus we need to proof 
	\begin{equation*}
		 \mathbb{E}[\Phi^{(2)}(\langle x_i, w^*\rangle)]>  \frac{1}{6}. 
	\end{equation*}
	This is because 
	\begin{align*}
		\mathbb{E}[\Phi^{(2)}(\langle x, w^*\rangle )] &= \mathbb{E}[\Phi^{(2)}(\frac{W}{4})]\\
		&= 2\int_{0}^{\infty} \Phi^{(2)}(\frac{w}{4})\phi(w) dw\geq 2\int_{0}^{\frac{4a}{b}} (a-b\frac{w}{4})\phi(w)dw \\
		&=2(a\zeta(\frac{4a}{b})-\frac{a}{2}-\frac{b}{4\sqrt{2\pi}}(1-e^\frac{-8a^2}{b^2}))> \frac{1}{6}. 	
	\end{align*}
Finally, we will show that $f'(z)$ is bounded by  constant $M=0.19$ on $[0, \bar{c}]$ from below. Since $x$ follows the Gaussian distribution, by Stein's lemma (Definition \ref{def:13}) we have 
\begin{equation*}
	f'(z)=\mathbb{E}[\Phi^{(2)}(\frac{Wz}{20})]+\frac{z^2}{20^2}\mathbb{E}[\Phi^{(4)}(\frac{Wz}{20})].
\end{equation*}
Thus 
\begin{align*}
	f'(z) &\geq \mathbb{E}[\Phi^{(2)}(\frac{Wz}{20})]- \frac{9}{100}|\Phi^{(4)}| \\
	&\geq 2(a\zeta(\frac{20a}{bz})-\frac{a}{2}-\frac{bz}{20\sqrt{2\pi}}(1-e^\frac{-200a^2}{b^2z^2}))- \frac{9}{800}>0.1
\end{align*}

\end{proof}

\subsection{Proof of Theorem \ref{theorem:12}}
\begin{proof}
	By simple calculation we can see that 
	\begin{equation*}
		\Phi^{(2)}(z)= \frac{1}{4}(1+\frac{z^2}{4})^{-\frac{3}{2}},  
		\Phi^{(3)}(z)= -\frac{3}{16}z(1+\frac{z^2}{4})^{-\frac{5}{2}},  
		\Phi^{(4)}=\frac{3}{64}\frac{5z^2(1+\frac{z^2}{4})^{-2}-4}{(1+\frac{z^2}{4})^{\frac{5}{4}}}, 
	\end{equation*}
	we can see that $|\Phi^{(2)}(\cdot)|\leq \frac{1}{4}$, $|\Phi^{(2)}(\cdot)|$ is $\frac{3}{16}$-Lipschitz and these two functions are even. Using the local convexity for $z\geq 0$ around $z=2$ we have 
	\begin{equation*}
		\Phi^{(2)}(z)\geq a-bz, 
	\end{equation*}
	where $a= \Phi^{(2)}(2)- 2\Phi^{(3)}(2)\approx 0.22$ and $b= -\Phi^{(3)}(2)\approx 0.066$. Denote $W\sim \mathcal{N}(0,1)$, $\phi$ as the density function of $W$ and $\zeta$ as the cumulative distribution function of $W$, we have 
	\begin{align*}
		f(z)&=z\mathbb{E}[\Phi^{(2)}(\langle x, w^{ols}\rangle z)] = z\mathbb{E}[\Phi^{(2)}(\frac{Wz}{20})]\\
		&= 2z\int_{0}^{\infty} \Phi^{(2)}(\frac{wz}{20})\phi(w) dw\geq 2z\int_{0}^{\frac{20a}{bz}} (a-b\frac{wz}{20})\phi(w)dw \\
		&=2z(a\zeta(\frac{20a}{bz})-\frac{a}{2}-\frac{bz}{20\sqrt{2\pi}}(1-e^\frac{-200a^2}{b^2z^2})). 
	\end{align*}
	Thus take $\bar{c}=6$ we have $f(\bar{c})>1+0.22$. 
	
	Next we will show $ c_\Phi \leq  \bar{c}$. Recall that $c_{\Phi}=\frac{1}{\mathbb{E}[\Phi^{(2)}(\langle x_i, w^*\rangle)] }$, thus we need to proof 
	\begin{equation*}
		\mathbb{E}[\Phi^{(2)}(\langle x_i, w^*\rangle)]> \frac{1}{6}. 
	\end{equation*}
	This is because 
	\begin{align*}
		\mathbb{E}[\Phi^{(2)}(\langle x, w^*\rangle )] &= \mathbb{E}[\Phi^{(2)}(\frac{W}{4})]\\
		&= 2\int_{0}^{\infty} \Phi^{(2)}(\frac{w}{4})\phi(w) dw\geq 2\int_{0}^{\frac{4a}{b}} (a-b\frac{w}{4})\phi(w)dw \\
		&=2(a\zeta(\frac{4a}{b})-\frac{a}{2}-\frac{b}{4\sqrt{2\pi}}(1-e^\frac{-8a^2}{b^2}))> \frac{1}{6}. 	
	\end{align*}
Finally, we will show that $f'(z)$ is bounded by  constant $M=0.1$ on $[0, \bar{c}]$ from below. Since $x$ follows the Gaussian distribution, by Stein's lemma we have 
\begin{equation*}
	f'(z)=\mathbb{E}[\Phi^{(2)}(\frac{Wz}{20})]+\frac{z^2}{20^2}\mathbb{E}[\Phi^{(4)}(\frac{Wz}{20})].
\end{equation*}
Thus 
\begin{align*}
	f'(z) &\geq \mathbb{E}[\Phi^{(2)}(\frac{Wz}{20})]- \frac{9}{100}|\Phi^{(4)}| \\
	&\geq 2(a\zeta(\frac{20a}{bz})-\frac{a}{2}-\frac{bz}{20\sqrt{2\pi}}(1-e^\frac{-200a^2}{b^2z^2}))- \frac{27}{1600}>0.1
\end{align*}
\end{proof}
\section{Proofs in Section 5}
\subsection{Proof of Theorem \ref{thm:nlrNew1}}
\begin{proof}[{\bf Proof of Theorem \ref{thm:nlrNew1}}]
  Denote $\phi(\cdot, \Sigma)$ as the multivariate normal density with mean 0 and covariance matrix $\Sigma$,  by simple calculation we have $\frac{d \phi(x, \Sigma)}{dx}= -\Sigma^{-1}x\phi(x, \Sigma)$. By the setting of (\ref{eq:9}) we have. 
    \begin{align*}
        \mathbb{E}[xy] &= \mathbb{E}[xf(\langle x, w^*\rangle) ] = \int xf(\langle x, w^*\rangle )\phi(x, \Sigma) dx \\
        &= -\Sigma \int f(\langle x, w^*\rangle )\frac{d \phi(x, \Sigma)}{dx} dx\\
        &= \Sigma w^* \mathbb{E}[f'(\langle x, w^*\rangle)], 
    \end{align*}
    where the last equation is deduced by integration by part. Thus 
    \begin{equation*}
        w^*= \frac{1}{\mathbb{E}[f'(\langle x, w^*\rangle)]}w^{ols}. 
    \end{equation*}
\end{proof}
\subsection{Proof of Theorem \ref{thm:5} }

The idea of the proof follows the one in \citep{erdogdu2019scalable}. 

\begin{proof}[{\bf Proof of Theorem \ref{thm:5}}]
By assumption, we have 
\begin{equation*}
    \mathbb{E}[xy]= \mathbb{E}[xf(\langle x, w^* \rangle )]=\Sigma^{\frac{1}{2}}\mathbb{E}[vf(\langle v, \hat{w}^*\rangle )],
\end{equation*}
where $\hat{w}^*=\Sigma^\frac{1}{2} w^*$. Now, consider each coordinate $j\in [p]$ for the term $\mathbb{E}[vf(\langle v, \hat{w}^*\rangle )]$.  Let  $v_j^*$ denote the zero-bias transformation of $v_j$ 
conditioned on $V_j=\langle v, \hat{w}^*\rangle-v_j \hat{w}_j^*$. Then, we have 
\begin{align*}
  \mathbb{E}[v_jf(\langle v, \hat{w}^*\rangle )]&= \mathbb{E}\mathbb{E}[v_jf(v_j\hat{w}^*_j+V_j)|V_j]\\
  &=  \hat{w}_j^*\mathbb{E} \mathbb{E}[f'(v_j^* \hat{w}_j^*+V_j)|V_j] \\ 
  &= \hat{w}_j^* \mathbb{E} \mathbb{E}[f'((v_j^*-v_j) \hat{w}_j^*+\langle v, \hat{w}^* \rangle )|V_j]  \\ 
  &= \hat{w}_j^* \mathbb{E}[f'((v_j^*-v_j) \hat{w}_j^*+\langle v, \hat{w}^* \rangle ) ]. 
\end{align*}
Thus, we have $w^{ols}=\Sigma^{-\frac{1}{2}}D\Sigma^\frac{1}{2}  w^* $, where $D$ is a diagonal matrix whose $i$-th entry is $\mathbb{E}[f'((v_j^*-v_j) \hat{w}_j^*+\langle v, \hat{w}^* \rangle ) ].$ 

By the Lipschitz condition, we have 
\begin{equation*}
    |\mathbb{E}[f'((v_j^*-v_j) \hat{w}_j^*+\langle v, \hat{w}^* \rangle ) ]- \mathbb{E}[f'(\langle v,\hat{w}^*\rangle)]|\leq G|\hat{w}_j^*|\mathbb{E}|(v_j^*-v_j) |. 
\end{equation*}
By the same argument  given in \citep{erdogdu2019scalable}, we have 
\begin{equation*}
    \mathbb{E}|(v_j^*-v_j) |\leq 1.5\mathbb{E}[|v_j|^3].
\end{equation*}
Using the bound of the third moment induced by the sub-Gaussian norm, we have 
\begin{equation*}
    L|\hat{w}_j^*|\mathbb{E}|(v_j^*-v_j) |\leq 8G\kappa_x^3\max_{j\in [p]}|\hat{w}_j^*| \leq 8G\kappa_x^3\|\Sigma^{\frac{1}{2}}w^*\|_{\infty}. 
\end{equation*}
Thus, we get 
\begin{equation*}
    \max_{j\in [d]}|D_{jj}-\frac{1}{c_f}|\leq 8G\kappa_x^3\|\Sigma^{\frac{1}{2}}w^*\|_{\infty}.
\end{equation*}
This means that 
\begin{align*}
    \|w^{ols}-\frac{1}{c_f}w^*\|_\infty &=\|\Sigma^{-\frac{1}{2}}(D-\frac{1}{c_f}I)\Sigma^\frac{1}{2}  w^*\|_\infty \\
    &\leq \max_{j\in[p]}|D_{jj}-\frac{1}{c_f}| \|\Sigma^{-\frac{1}{2}}\|_\infty \|\Sigma^\frac{1}{2} \|_\infty \|w^*\|_\infty \\
    &\leq 8L\kappa_x^3\rho_\infty L\|\Sigma^{\frac{1}{2}}\|_\infty \|w^*\|^2_\infty. 
\end{align*}
Due to the diagonal dominance property we have 
\begin{equation*}
    \|\Sigma^{\frac{1}{2}}\|_{\infty}=\max_{i}\sum_{j=1}^p|\Sigma_{ij}^{\frac{1}{2}}|\leq 2\max \Sigma_{ii}^{\frac{1}{2}}\leq 2\|\Sigma\|_2^{\frac{1}{2}}.
\end{equation*}
Since we have $\|x\|_2\leq r$, we write 
\begin{equation*}
    r^2 \geq \mathbb{E}[\|x\|_2^2]=\text{Trace}(\Sigma)\geq p\|\Sigma\geq \frac{p\|\Sigma\|_2}{\rho_2}.
\end{equation*}
Thus we have 
   $ \|\Sigma^{\frac{1}{2}}\|_\infty\leq 2r\sqrt{\frac{\rho_2}{p}}$.
\end{proof}

\subsection{Proof of Theorem \ref{thm:7}} 
By the same argument in the proof of Theorem \ref{thm:0}, we can show that if $n\geq \Omega(\frac{\|\Sigma\|_2^2pr^4\log \frac{1}{\delta}}{\epsilon^2\lambda^2_{\min}(\Sigma)\min\{\lambda^2_{\min}(\Sigma),1\}})$ then  with probability at least $1-\exp(-\Omega(p))-\xi$ 
\begin{equation}\label{aeq:6.1}
   \|\hat{w}^{ols}-\tilde{w}^{ols}\|^2_2 = O\big( \frac{p C^2r^2(L^2r^2+C^2+ r^2\|\tilde{w}^{ols}\|_2^2)\log \frac{1}{\delta}\log \frac{p^2}{\xi} }{\epsilon^2n\lambda^2_{\min}(\Sigma)}\big).
\end{equation}

Thus, by Lemma \ref{lemma:a5.9} we have 
\begin{equation}\label{aeq:6.2} 
       \|\hat{w}^{ols}-w^{ols}\|_2\leq O\big( \frac{CL\kappa_x \sqrt{p}r^2\|w^{ols}\|_2\sqrt{\log\frac{1}{\delta}\log\frac{p^2}{\xi}}}{\epsilon\sqrt{n}\lambda^{1/2}_{\min} (\Sigma)\min\{\lambda^{1/2}_{\min} (\Sigma), 1\}} \big).  
\end{equation}
In the following, we will always assume that (\ref{aeq:6.2}) holds. By the same argument given in Lemma \ref{lemma:a5.13}, we have the following Lemma, which can be proved in the same way as  
Lemma \ref{lemma:a5.13}. 

\begin{lemma}
Let $f'$ be a function that is Lipschitz continuous with constant $G$ and $|f'(\cdot)|\leq L$, and  $g: \mathbb{R}\times \mathbb{R}^p\mapsto \mathbb{R}$ be another function such that 
$g(c, w)= c\mathbb{E}[f'(\langle x, w\rangle c)]$ and its empirical one is 
\begin{equation*}
    \hat{g}(c, w)= \frac{c}{m}\sum_{j=1}^m f'(\langle x, w\rangle c). 
\end{equation*}
Let $\mathbb{B}^\delta(\bar{w}^{ols})=\{w: \|w-\bar{w}^{ols}\|_2\leq \delta\}$, where $\bar{w}^{ols}=\Sigma^{\frac{1}{2}}w^{ols}$. Then, under the assumptions in Lemma \ref{lemma:a5.8} and Eq.~(\ref{aeq:6.2}), 
 with probability at least $1-4\exp(-p)$, there exists a constant $\hat{c}_{f}\in [0, \bar{c}]$ such that $\hat{g}(\hat{c}_f, \hat{w}^{ols})=1$. 
Furthermore, if the derivative of $c\mapsto g(c, w^{ols})$ is bounded below in absolute value ({\em i.e.,} does not change the sign) by $M>0$ in the interval of $c\in [0, \bar{c}]$, then with probability at least $1-4\exp(-p)$, the following holds 
\begin{equation}
    |\hat{c}_f-\bar{c}_f|\leq O\big(\frac{ M^{-1}CGL\bar{c}^2r^2 \|\Sigma\|_2^{\frac{1}{2}}\sqrt{p}\|w^{ols}\|_2 \log \frac{1}{\delta}\log \frac{p}{\xi^2} 
        }{\epsilon\lambda^{\frac{1}{2}}_{\min}(\Sigma)\min\{ \lambda^{\frac{1}{2}}_{\min}(\Sigma), 1\}\sqrt{n} }+M^{-1}LG\|\Sigma\|_2^\frac{1}{2}\|w^{ols}\|_2 \sqrt{\frac{p}{m}}\big)
\end{equation}
for sufficiently large $m, n$ such that 
\begin{align}
   & n \geq \Omega \big( \frac{LG^2\tau^{-2} \bar{c}^4\|\Sigma\|_2 \kappa_x^4 p r^4\|w^{ols}\|_2^2\log\frac{1}{\delta}\log \frac{p^2}{\xi}}{\epsilon^2\lambda_{\min}(\Sigma)\min\{ \lambda_{\min}(\Sigma), 1\}}\big)\\
   &  m\geq \Omega \big(G^2L^2\|\Sigma\|_2 \|w^{ols}\|_2^2 p\tau^{-2}).
\end{align}
where $r=\max_{i\in [n]}\|x_i\|_2$. 
\end{lemma}
\subsection{Proof of Theorem \ref{thm:nlrnew2}}
The proof is almost the same as the proof of Theorem \ref{thm:0}. We know that  when $n\geq \Omega(\frac{\|\Sigma\|_2^2pr^4\log \frac{1}{\delta}}{\epsilon^2\lambda^2_{\min}(\Sigma)\min\{\lambda^2_{\min}(\Sigma), 1\}})$,  with probability at least $1-\exp(-\Omega(p))-\xi$,
 there is a constant $C_3>0$ such that 
\begin{equation}\label{eq:64}
    \|\hat{w}^{ols}-w^{ols}\|_2\leq C_3 \frac{\sqrt{p^3}\|\Sigma\|_2 \|w^{ols}\|_2\log n\sqrt{\log\frac{1}{\delta}\log\frac{p}{\xi}}}{\epsilon\sqrt{n}\lambda_{\min} (\Sigma)\min\{\lambda_{\min} (\Sigma), 1\}}. 
\end{equation}
Similar to Lemma \ref{lemma:new1}, we have the following lemma. 
	\begin{lemma}\label{lemma:newnlr1}
Let $f'$ be a function that is Lipschitz continuous with constant $G$ and $|f'(\cdot)|\leq L$, and  $g: \mathbb{R}\times \mathbb{R}^p\mapsto \mathbb{R}$ be another function such that 
$g(c, w)= c\mathbb{E}[f'(\langle x, w\rangle c)]$ and its empirical one is 
\begin{equation*}
    \hat{g}(c, w)= \frac{c}{m}\sum_{j=1}^m f'(\langle x, w\rangle c). 
\end{equation*}
Let $\mathbb{B}^\delta(\bar{w}^{ols})=\{w: \|w-\bar{w}^{ols}\|_2\leq \delta\}$, where $\bar{w}^{ols}=\Sigma^{\frac{1}{2}}w^{ols}$. Then, under Eq.~(\ref{eq:64}), 
 with probability at least $1-4\exp(-p)$, there exists a constant $\hat{c}_{f}\in [0, \bar{c}]$ such that $\hat{g}(\hat{c}_f, \hat{w}^{ols})=1$. 
Furthermore, if the derivative of $c\mapsto g(c, w^{ols})$ is bounded below in absolute value ({\em i.e.,} does not change the sign) by $M>0$ in the interval of $c\in [0, \bar{c}]$, then with probability at least $1-4\exp(-p)$, the following holds 
\begin{equation}
    |\hat{c}_f-c_f|\leq  O\big(\frac{ M^{-1}GL\bar{c}^2 \|\Sigma\|_2^{\frac{3}{2}}p^{\frac{3}{2}}\|w^{ols}\|_2\log n \sqrt{\log \frac{1}{\delta}\log \frac{p}{\xi}} 
        }{\epsilon\lambda_{\min}(\Sigma)\min\{ \lambda_{\min}(\Sigma), 1\}\sqrt{n} }+M^{-1}LG\|\Sigma\|_2^\frac{1}{2}\|w^{ols}\|_2 \sqrt{\frac{p}{m}}\big)
\end{equation}
for sufficiently large $m, n$ such that 
\begin{align}
   & n \geq \Omega \big( \frac{LG^2\tau^{-2} \bar{c}^4\|\Sigma\|^3_2  p^3 \|w^{ols}\|_2^2 \log n\log\frac{1}{\delta}\log \frac{p}{\xi}}{\epsilon^2\lambda^2_{\min}(\Sigma)\min\{ \lambda^2_{\min}(\Sigma), 1\}}\big)\\
   &  m\geq \Omega \big(G^2L^2\|\Sigma\|_2 \|w^{ols}\|_2^2 p\tau^{-2}).
\end{align}
\end{lemma}
Next we bound $\|\hat{w}^{nlr}-w^*\|_2=\|\hat{c}_f \hat{w}^{ols}-c_f w^{ols}\|_2$. We have 
\begin{align}\label{aeq:newnlr1}
	\|\hat{c}_f \hat{w}^{ols}-c_f w^{ols}\|_2\leq |\hat{c}_f-c_f|\|\hat{w}^{ols}\|_2+ c_f \|\hat{w}^{ols}-w^{ols}\|_2. 
\end{align}
For the second term of (\ref{aeq:newnlr1}),  by (\ref{eq:64}) we have 
\begin{equation*}
	c_f \|\hat{w}^{ols}-w^{ols}\|_2\leq O(\frac{\bar{c}p^\frac{3}{2}\|\Sigma\|_2 \|w^{ols}\|_2\log n\sqrt{\log\frac{1}{\delta}\log\frac{p}{\xi}}}{\epsilon\sqrt{n}\lambda_{\min} (\Sigma)\min\{\lambda_{\min} (\Sigma), 1\}}). 
\end{equation*}
For the first term of (\ref{aeq:newnlr1}), by Lemma \ref{lemma:newnlr1} and (\ref{eq:64}) we have 
\begin{equation*}
	|\hat{c}_f-c_f|\|\hat{w}^{ols}\|_2\leq O\big(\frac{ M^{-1}GL\bar{c}^2 \|\Sigma\|_2^{\frac{3}{2}}p^{\frac{3}{2}}\|w^{ols}\|^2_2\log n \sqrt{\log \frac{1}{\delta}\log \frac{p}{\xi}} 
        }{\epsilon\lambda_{\min}(\Sigma)\min\{ \lambda_{\min}(\Sigma), 1\}\sqrt{n} }+M^{-1}LG\|\Sigma\|_2^\frac{1}{2}\|w^{ols}\|^2_2 \sqrt{\frac{p}{m}}\big)
\end{equation*}
Take $w^{ols}=\frac{w^*}{c_f}$ we can get the proof. 
\subsection{Proof of Theorem \ref{theorem:13}}
\begin{proof}
	We can easily see that $f'(\cdot)$ is just the function $\Phi^{(2)}(\cdot)$ in Theorem \ref{theorem:11} for the logistic loss function. Thus the function $f'$ satisfies the assumptions in Theorem \ref{thm:7}, which was showed in the Theorem \ref{theorem:11}. 
\end{proof}

\section{A 2-Round LDP Algorithm for Algorithm \ref{alg:0}}\label{sec:2-round}
\begin{algorithm}[!ht]
\caption{2-round LDP for smooth GLMs with public data (Gaussian)	\label{alg:2-round}}
	\begin{algorithmic}[1]
		\State {\bfseries Input:} Private data $\{(x_i, y_i)\}_{i=1}^n \in  (\mathbb{R}^p\times [0, 1])^n$, where  $|y_i|\leq 1$, $\{x_i\}_{j=1}^{n}\sim \mathcal{N}(0, \Sigma)$ for some unknown $\Sigma$,  loss function $\Phi:\mathbb{R}\mapsto \mathbb{R}$, privacy parameters $\epsilon, \delta$, and initial value $c\in \mathbb{R}$.
	\State {\bf In the first round:}	
	\For {The server}
	\State Calculate $\Sigma_m=\frac{1}{m}\sum_{j=n+1}^{n+m}x_jx_j^T$ and send it to each user. 
	\EndFor 
     \For{Each user $i\in [n]$}
     \State Let $\bar{x}_i=x_i\min\{1, \frac{r}{\|x_i\|_2}\} $, where $r\equiv \sqrt{20 p\|\Sigma_m\|_2\log n}$. 
     
     \State Release $\widehat{{x}_i {x}_i^T}= \bar{x}_i\bar{x}_i^T + E_{1,i}$ and  $\widehat{x_iy_i}=\bar{x}_iy_i+ E_{2, i}$, where $E_{1,i} \in \mathbb{R}^{p\times p} $ is a symmetric matrix and each entry of  the upper triangle matrix is sampled from $\mathcal{N}(0, \frac{128r^4\log\frac{2.5}{\delta}}{\epsilon^2})$ and $E_{2,i}\in \mathbb{R}^{p}$ is sampled from $\mathcal{N}(0, \frac{128r^2\log \frac{2.5}{\delta}}{\epsilon^2}I_p)$.
     \EndFor 
     
     \For {The server}
     \State Let $\widehat{X^TX}=\sum_{i=1}^n \widehat{x_ix_i^T}$ and $\widehat{X^Ty}=\sum_{i=1}^n \widehat{x_iy_i}$. Calculate $\hat{w}^{ols}=(\widehat{X^TX})^{-1}\widehat{X^Ty}$. 
     \State Send $\hat{w}^{ols}$ to all users. 
     \EndFor 
     \State {\bf In the second round:}
       \For{Each user $i\in [n]$}
       \State Calculate $\bar{y}_i=\langle x_i, \hat{w}^{ols}\rangle$.
       \State Project $\bar{y}_i$ onto the interval $[0, 1]$ and denote it as $\hat{y}_i$. 
       \State Send $\tilde{y}_i=\hat{y}_i+\mathcal{N}(0, \frac{8\log\frac{2.5}{\delta}}{\epsilon^2})$ to the server. 
     \EndFor 
       \For {The server}
     
     \State  Find the root $\hat{c}_{\Phi}$ such that $1= \frac{\hat{c}_{\Phi}}{n}\sum_{j=1}^{n}\Phi^{(2)}(\hat{c}_{\Phi}\tilde{y}_j)$ by using Newton's root-finding method (or other methods):
     \For{$t=1, 2, \cdots$ until convergence}
    \State $c= c- \frac{c\frac{1}{n}\sum_{j=1}^{n}\Phi^{(2)}(c\tilde{y}_j)-1}{
     \frac{1}{n}\sum_{j=1}^{n}\{\Phi^{(2)}(c\tilde{y}_j)+c\tilde{y}_j\Phi^{(3)}(c\tilde{y}_j)\}}$. 
     \EndFor 
     \EndFor \\
     \Return {$\hat{w}^{glm}= \hat{c}_{\Phi}\cdot \hat{w}^{ols}$.} 
	\end{algorithmic}
\end{algorithm}

\end{document}